\documentclass[11pt]{article}
\usepackage[left=2.54cm,top=2.54cm,right=2.54cm,bottom=2.54cm]{geometry}
\usepackage{longtable}

\input{macros/preamble.tex}

\setcitestyle{round}
\usepackage[colorlinks=false,]{hyperref}
\usepackage{rotating}
\usepackage{xcolor}
\usepackage{enumitem}
\usepackage{mathrsfs}

\newcommand{\spn}{\operatorname{sp}}

\newcommand{\Pn}{\mathbb P_n}

\newcommand{\TV}{\operatorname{TV}}
\allowdisplaybreaks

\begin{document}
\title{\huge Optimal Value Inference for Reinforcement Learning}
\author{
    Nan Lu\thanks{Department of Biostatistics, Harvard T. H. Chan School of Public Health}  \quad
    Ethan Lee\thanks{Department of Biostatistics, Harvard T. H. Chan School of Public Health} \quad
    Ruicheng Ao\thanks{Institute for Data, Systems, and Society, Massachusetts Institute of Technology} \quad
    David Simchi-Levi\thanks{Institute for Data, Systems, and Society, Department of Civil and Environmental Engineering, and Operations Research Center, Massachusetts Institute of Technology; Mitch Daniels School of Business, Purdue University} \quad
    Junwei Lu\thanks{Department of Biostatistics, Harvard T. H. Chan School of Public Health}
}
\date{}
\maketitle

\begin{abstract}
We study offline inference for the optimal value in reinforcement learning under finite state and action spaces. Two new nuisances are derived as fixed points of a self-induced Bellman equation, in which we approximate the maximum Bellman operator by its softmax correspondence. We propose a debiased estimator through the Neyman orthogonality and establish its asymptotic normality under diverging horizons even when the behavior policy changes with time, as long as the nuisances have the statistical rates that can be achieved by many machine learning methods. We provide a concrete estimating procedure for these nuisances and show they can lead to valid inference. Synthetic experiments validate the numerical performance of our inference method, and we implement it in real-life decision-making problems, including bike repositioning and AI agentic tool use.
\end{abstract}

\paragraph{Keywords:}
    Reinforcement learning, optimal value inference, model evaluation, self-induced Bellman equation, infinite horizon.

\section{Introduction}\label{sec:intro}
Modern decision-making systems make decisions over increasingly long time horizons, generating interacting samples that can potentially inform future policy improvement. A language model agent, for example, interacts with users, tools, and environments through multistep sequences of observations and actions \citep{chang2024agentboard}.
Similar sequential decision problems arise in recommendation, online experimentation, dynamic pricing, and inventory management. A recommendation platform may want to quantify the engagement attainable over the long run under an optimized policy \citep{ie2019slateq,chen2019generative}. A pricing platform wants to assess the revenue achievable under an optimized pricing rule \citep{zhao2026dynamicpricing}, while a retailer wants to understand the long-run operating cost attainable under an optimized replenishment policy \citep{bai2026semiopenloop}. Across these settings, decision-making systems generate the trajectories of observed states, actions, and rewards implementing a given policy, which may not be optimal. This raises the core inferential question we aim to address in this paper: can we infer the optimal policy value given observational trajectories of decision-making processes generated from a given policy that is not optimal? Before deploying a new policy, a decision-maker may want to know not only how one candidate policy would perform, but how much value could be achieved by optimizing future decisions.
Modern agentic systems also make the temporal scale of this question increasingly important. Language-model agents increasingly solve tasks through extended sequences of reasoning, tool use, and environment interaction, and controlled experiments show that increasing task horizon alone can induce severe training instability through exploration difficulties and credit-assignment challenges \citep{kim2026longhorizon}. Extended sequential search and adaptation also arise in demanding domains such as formal mathematical reasoning \citep{hubert2026olympiad}. The trajectory length therefore need not remain fixed as more complex tasks are considered, motivating an inferential framework that allows the observation horizon to diverge.

Statistical inference using observational reinforcement learning (RL) trajectories has been considered in two directions.
Off-policy evaluation (OPE) estimates the value of a prespecified target policy using data generated by a different behavior policy \citep{li2011unbiased,jiang2016doubly,kallus2020double}, while offline reinforcement learning uses previously collected data to learn a policy with small suboptimality \citep{jin2021pessimism}. Performance guarantees for a learned policy typically characterize its loss relative to the optimum. However, neither standard off-policy evaluation nor offline reinforcement learning directly provides valid uncertainty quantification for the unknown optimal value. Therefore, in this paper, we ask what observational trajectories can reveal about the best long-run performance attainable in the underlying environment while providing valid uncertainty quantification. Specifically, we study the optimal average reward $J^*$, the best long-run average reward attainable by a stationary Markov policy. A confidence interval for the optimal value quantifies uncertainty about this quantity and, together with an evaluation of the current policy, can help assess the potential for further improvement. Throughout, we focus on finite state and action spaces.

Inference for $J^*$ is more challenging than off-policy evaluation because optimality itself is part of the target. Exact and near ties between actions create nonsmoothness and can make the optimal decision rule unstable to small perturbations and optimal-value inference nonregular, analogous to nonregularity under nonunique optimal treatment decisions \citep{robins2004optimal,luedtke2016statistical}. In a Markov decision process (MDP), this difficulty is compounded by recursion because action maximization enters the Bellman equation that determines continuation values and future decisions. Existing smoothing-based approaches to optimal value inference first estimate the optimal action-value function under the Bellman optimality equation and only then introduce smoothing. \citet{whitehouse2025inference} apply softmax smoothing to the resulting optimal-value functional, whereas \citet{wei2025characterization} constructs a softmax policy from the estimated optimal action-value function and evaluates the induced policy using standard off-policy inference machinery \citep{kallus2020double}. Valid inference for the value of a data-dependent policy can be obtained through suitable sample splitting when the learned policy is sufficiently close to optimality \citep{shi2022statistical}.
However, requiring the learned policy to approach optimality at a sufficiently fast rate can be restrictive in practice, especially when further policy optimization is itself difficult. Deep networks can lose plasticity under continued learning, including in reinforcement learning \citep{dohare2024loss}, while recent work on LLM
mathematical-reasoning post-training documents that subsequent RL improvement can be limited by reduced RL plasticity and premature convergence \citep{ding2026rethinking}. The main purpose of conducting optimal value inference is to assess how far the current policy is from optimality. The statistical question of how much optimal value is achievable in the environment should therefore be separated from the algorithmic question of whether a particular training procedure has already learned a near-optimal policy. Our goal is to make this separation possible by constructing an optimal value inference method that is robust to errors in optimal policy estimation.

The key idea of our proposed method is that, instead of deriving the optimal policy estimator from softmax, we directly smooth the Bellman operator itself and estimate the optimal value from a self-induced Bellman equation. Let the Bellman optimality equation be $J^*+ Q^*(s,a) = R^*(s,a) + \EE[\max_{a' \in \mathcal{A}} Q^*(S',a') | S=s, A = a]$, where $(S,A,S')$ is an MDP transition tuple, $\mathcal{A}$ is the action space, $R^*$ is the true reward function, $Q^*$ is the optimal action-value function, and $J^*$ is the optimal value. Our self-induced Bellman equation approximates the maximum by the softmax: $J_{\beta}+ Q_{\beta}(s,a) = R^*(s,a) + \EE[\mathrm{Softmax}_{\beta}[Q_{\beta}(S',\cdot)] | S=s, A = a]$, where $\beta$ is a positive smoothing parameter of the softmax, and we solve $(Q_\beta, J_\beta)$ from the equation as the approximation of $(Q^*, J^*)$; see Section~\ref{sec:est} for details. Our insight is that $(Q_\beta, J_\beta)$ are the key nuisances for valid inference with double machine learning properties.
We develop a Neyman-orthogonal score statistic such that the remainder error of its linear representation depends on estimation errors for $(Q_\beta, J_\beta)$ and an adjoint weight function, without assuming a rate for the optimal policy estimator. Moreover, we show that the proposed statistic yields valid inference with potentially nonstationary trajectories. In particular, we allow the behavior policy to vary over time, which is common in modern decision-making processes such as large language model training. For $N$ independent trajectories, each observed over $H$ transitions, under regular mixing conditions, our inference method yields $\sqrt{NH}$-rate asymptotic normality when either $N$ or $H$ diverges by using the underlying martingale property of the decision-making process. To estimate the nuisances, we propose a regularized minimax estimator for the nonlinear self-induced Bellman solution and a plug-in empirical risk minimization estimator for the adjoint weight. We establish finite-sample rates for the nuisance estimators that are sufficient to ensure valid inference for the optimal value.
We summarize our contributions in the paper as follows.

\noindent{\bf Double machine learning for optimal value inference.}
We introduce a self-induced Bellman equation and characterize its solution as the key nuisances. We then construct the corresponding normalized adjoint weight despite the singular linearization. Based on these objects, we develop a new Neyman-orthogonal score statistic.  We establish asymptotic normality under second-order nuisance-rate conditions without additional requirements to consistently estimate the optimal policy. We further provide a semiparametric efficiency characterization.

\noindent{\bf Regularized minimax estimator for nuisances.}
The nonlinear self-induced Bellman equation creates a nonstandard nuisance estimation problem. We develop a regularized minimax estimator based on Bellman residuals and test functions, together with a plug-in empirical risk minimization estimator for the adjoint weight. We establish error rates using all transitions from dependent trajectories, which verify the second-order nuisance conditions required for valid inference.

\noindent{\bf Long-horizon inference under nonstationary behavior.}
Our inference theory accommodates time-varying behavior policies and diverging observation horizons within a unified framework. We formulate the adjoint equation under the pooled one-step behavior policy rather than relying on a stationary state distribution as in \citet{liu2018breaking}. At the same time, the self-induced Bellman equation makes the Neyman-orthogonal score statistic have a martingale-difference sequence as the leading term, so cross-time covariances vanish and the trajectory-level variance scales as $H^{-1}$. Combining these structures, we establish $\sqrt{NH}$-rate asymptotic normality under regular mixing conditions by applying a triangular-array martingale central limit theorem.

\subsection{Related Work}

\paragraph{Optimal decision rules and nonregular value inference.}

A broad literature studies optimal treatment rules, empirical welfare maximization, and policy learning in contextual and finite-stage settings
\citep{manski2004statistical,kitagawa2018who,athey2021policy,zhou2023offline}.
For optimal sequential treatment decisions, \citet{robins2004optimal} develops optimal structural nested models based on blip functions, which encode treatment contrasts when subsequent treatments follow an optimal regime. Further developments include methods for optimal treatment and testing strategies \citep{robins2008estimation}, dynamic treatment regimes \citep{murphy2001marginal,orellana2010dynamic}, and, more recently, pessimistic policy learning under limited action coverage \citep{zhou2023optimizing}.
These approaches primarily concern contextual or finite-stage longitudinal decisions, whereas we study the globally optimal long-run average reward in a Markov decision process.

Nonregularity of optimized values is closely tied to nonuniqueness of the optimal decision.
\citet{robins2004optimal} shows that exceptional laws, under which the optimal treatment is nonunique with positive probability, lead to nonregular inference for optimal treatment regimes. \citet{luedtke2016statistical} give necessary and sufficient conditions for pathwise differentiability of the optimal value, establish regular and asymptotically efficient estimation when the parameter is pathwise differentiable, and develop root-$n$ inference that remains valid in nonregular cases.
Related work studies targeted learning of optimal dynamic treatment values \citep{van2015targeted} and margin conditions for nearly indifferent decisions \citep{mammen1999smooth}.
Related asymptotic inference results have also been developed for online contextual bandits and preference-based decision problems
\citep{chen2021,chen2021statistical,shen2024doubly,lu2026contextual}.

More recently, smoothing has been used to regularize optimized-value inference.
\citet{whitehouse2025inference} apply softmax smoothing to maximum-type irregular functionals, primarily for static treatment rules with an extension to a two-stage dynamic regime.
For Markov decision processes, \citet{wei2025characterization} studies optimal value inference under policy nonuniqueness and considers a sample-split procedure that applies softmax to a $Q$-function estimated from the Bellman equation before fixed-policy evaluation.
These approaches introduce smoothing after solving the underlying optimization problem.
In contrast, we smooth the Bellman recursion itself and jointly determine the induced policy and its continuation values through a self-induced policy--value fixed point.

\paragraph{Off-policy evaluation and reinforcement learning inference.}

Off-policy evaluation has developed doubly robust and semiparametric methods for contextual bandits and reinforcement learning
\citep{dudik2011doubly,jiang2016doubly,zhan2021offpolicy,liu2018breaking,kallus2020double,kallus2022efficiently}.
In infinite-horizon MDPs, \citet{shi2022statistical} establish asymptotically normal value estimators for fixed and data-dependent policies using sequential sample splitting, while \citet{shi2024off} study robust and efficient off-policy inference under confounding.
Related work provides finite-sample minimax analyses of offline reinforcement learning
\citep{uehara2021finite}
and studies variance reduction and stability in policy evaluation
\citep{han2026variance}.

Our setting differs because the optimized long-run average reward is itself the inferential target and action maximization is internal to its Bellman recursion.
Fixed-policy orthogonality protects the evaluation stage against first-order nuisance error conditional on the target policy, but does not by itself remove error from learning that policy through a nonsmooth optimal Bellman map.
Our self-induced construction instead incorporates policy generation and policy evaluation into the same smooth fixed point before orthogonalization.
We also allow time-varying behavior policies and develop inference under both fixed and diverging trajectory horizons.
The resulting construction builds on the broader literature on Neyman-orthogonal scores, cross-fitting, and orthogonal statistical learning
\citep{chernozhukov2018double,dai2023orthogonalized,foster2023orthogonal}.

\paragraph{Average-reward reinforcement learning.}

Average-reward reinforcement learning builds on classical Bellman theory
\citep{mahadevan1996average}, with recent advances in feasible Q-learning, nonasymptotic convergence, and span-based sample complexity
\citep{pmlr-v238-jin24b,chen2026non,zurek2024span}.
On the inference side, \citet{liao2022batch} develop doubly robust and semiparametrically efficient estimation of stationary policy values and optimize over a prespecified stochastic policy class.
We instead target the global optimum and smooth the Bellman optimization itself before orthogonalization.

\paragraph{Sequential decisions in operations and model evaluation.}

Sequential decision problems are also central in operations and management science, including dynamic pricing and inventory control
\citep{gallego1994optimal,chen2022primal},
inventory systems with cyclic demand
\citep{gong2024bandits},
nonstationary decision environments
\citep{cheung2020reinforcement,mao2020model},
and queueing network control under long-run average objectives
\citep{dai2022}.
These works primarily focus on learning and deploying effective policies. Our work addresses a complementary offline statistical question, using observational trajectories to quantify uncertainty about the best achievable long-run performance before a newly optimized policy or agent is deployed.

\paragraph{Paper organization.}
Section~\ref{sec:formu} introduces the reinforcement learning setup and notation. Section~\ref{sec:est} defines the self-induced Bellman equation, establishes existence and uniqueness under identification, and analyzes the approximation error $J^*-J_\beta$. Section~\ref{sec:if} develops inference theory based on the orthogonal score and establishes central limit theorems under fixed and diverging horizon regimes. Section~\ref{sec:nuisance-estimation} develops estimators and rate guarantees for the self-induced Bellman solution and the adjoint weight. Section~\ref{sec:numerical} presents the synthetic experiments and two empirical analyses of bike repositioning and agentic tool use. Section~\ref{sec:discussion} concludes the paper. Proofs are deferred to the appendix.

\section{Problem Setup}\label{sec:formu}

Let $\cS$ and $\cA$ denote the state and action spaces. We observe $N$ independent trajectories generated under a common initial-state law and a common trajectory law induced by the behavior policy. Let $H$ denote the number of observed transitions. For trajectory index $i$ and time $h=0,\ldots,H-1$, let $S_{i,h}\in\cS$ and $A_{i,h}\in\cA$ denote the state and action at time $h$, and let $R_{i,h+1}$ be the reward observed after taking $A_{i,h}$ in state $S_{i,h}$. The $i$-th $H$-transition trajectory is
\[
\tau_i=(S_{i,0},A_{i,0},R_{i,1},S_{i,1},\ldots,A_{i,H-1},R_{i,H},S_{i,H}),
\qquad i=1,\ldots,N.
\]
We suppress the trajectory index $i$ for a generic trajectory.
The environment is a time-homogeneous Markov process, such that for every measurable set $B\subseteq\RR\times\cS$,
\[
\PP\!\left(
(R_{h+1},S_{h+1})\in B
\mid
S_0,A_0,R_1,\ldots,S_h,A_h
\right)
=
\PP\!\left(
(R_{h+1},S_{h+1})\in B
\mid
S_h,A_h
\right),
\qquad h=0,\ldots,H-1.
\]
For a generic one-step transition, write $(S,A)$ for the current state--action pair and $(R,S')$ for the subsequent reward and state. Define the conditional mean reward by
$R^*(s,a)\coloneqq\EE[R_h\mid S_{h-1}=s,A_{h-1}=a].$ The transition kernel is $
P(s'\mid s,a)\coloneqq\PP(S'=s'\mid S=s,A=a).$
We also use the transition kernel as a linear operator from state functions to state-action functions, with $
(Ph)(s,a)
\coloneqq
\EE[h(S')\mid S=s,A=a]
=
\sum_{s'\in\cS}P(s'\mid s,a)h(s'),$ where $h:\cS\to\RR. $ 
The Markov decision process is specified by $\cS$, $\cA$, the transition kernel $P$, and the mean reward function $R^*$. We consider the finite tabular specialization, where $\cS$ and $\cA$ are finite.
For any measure $\nu$ on the common domain of $f$ and $g$, write $\langle f,g\rangle_\nu:=\int fg\,d\nu$ and $\|f\|_{\nu,2}^2:=\int f^2\,d\nu$. For $q:\cS\times\cA\to\RR$ and a measure $\nu$ on $\cS$, we also use the shorthand $\|q\|_{\nu,2}^2:=\sum_{s\in\cS}\nu(s)\sum_{a\in\cA}q(s,a)^2$.
For finite-domain functions, write $\|f\|_\infty:=\max_x|f(x)|$, with $x$ ranging over the domain of $f$.

Let $\cP(\cA)$ denote the set of probability distributions on $\cA$.
 A stationary Markov policy is a map $\pi\colon\cS\to\cP(\cA)$. Let $\Pi$ denote the collection of all stationary Markov policies. We write $\pi(a\mid s)$ for the probability of choosing action $a$ at state $s$.
The behavior policy may vary across time and is denoted by $\pi^b=\{\pi_h^b\}_{h=0}^{H-1}$, where each $\pi_h^b\in\Pi$. Thus, $\pi^b$ is generally nonstationary, although each decision rule $\pi_h^b$ is Markov.
In the stationary behavior case, the decision rules $\{\pi_h^b\}_{h=0}^{H-1}$ are identical across $h$.
Let $\EE_\pi$ denote expectation under the trajectory law induced by policy $\pi$ together with the Markov decision process above. By default, $\EE$ denotes expectation under the same Markov decision process with behavior data law induced by $\pi^b$.
For $\pi\in\Pi$, define the average reward from initial state $s$ by
\[
J^\pi(s)
\coloneqq
\lim_{T\to\infty}
\EE_\pi\!\left[
\frac{1}{T}\sum_{h=1}^{T}R_h
\,\Bigm|\, S_0=s
\right].
\]
For a finite-state, time-homogeneous MDP with bounded rewards, this limit exists for every $\pi\in\Pi$ and $s\in\cS$. 
\begin{assumption}[Irreducibility]\label{ass:irreducible}
For every $\pi\in\Pi$, the policy-induced transition kernel $
P^\pi(s'\mid s)
\coloneqq
\sum_{a\in\mathcal A}\pi(a\mid s)P(s'\mid s,a) $
is irreducible.
\end{assumption}
This condition is the analogue of Assumption~1 in \citet{liao2022batch}. Under Assumption~\ref{ass:irreducible}, $J^\pi(s)$ is independent of the initial state $s$; see \citet[Theorem~8.2.6]{puterman2014markov}. We therefore write $J^\pi$ without ambiguity. We impose Assumption~\ref{ass:irreducible} throughout the paper.
For a fixed policy $\pi\in\Pi$, consider the relative action value at $(s,a)$
\[
\lim_{T\to\infty}
\frac{1}{T}\sum_{t=1}^T
\EE_\pi\!\left[
\sum_{k=1}^t\bigl(R_k-J^\pi\bigr)
\,\Bigm|\, S_0=s,A_0=a
\right].
\]
Here the process follows $\pi$ from the next decision onward. 
The form allows us to define the relative value without requiring aperiodicity; see \citet{liao2022batch}.
The Bellman representation of a relative value is invariant to adding a constant. We remove this ambiguity by requiring zero arithmetic mean over the state-action space and define the anchored space
\[
\cQ
\coloneqq
\left\{
Q:\cS\times\cA\to\RR:
\frac{1}{|\cS||\cA|}\sum_{s\in\cS}\sum_{a\in\cA}Q(s,a)=0
\right\}.
\]
For each $\pi\in\Pi$, let $Q^\pi\in\cQ$ be the zero-average representative of this probabilistic relative value, obtained by subtracting its arithmetic average over $\cS\times\cA$.
Classical theory for average rewards gives the following local recursive characterization of the same population object:
\begin{equation}
J^\pi+Q^\pi(s,a)
=
R^*(s,a)+\sum_{s'\in\cS}P(s'\mid s,a)
\sum_{a'\in\cA}\pi(a'\mid s')Q^\pi(s',a'),
\qquad (s,a)\in\cS\times\cA.
\label{eq:policy-evaluation}
\end{equation}

Our parameter of interest is the optimal average reward $
J^*
\coloneqq
\max_{\pi\in\Pi} J^\pi.$
By \citet[Theorem~8.1.2]{puterman2014markov}, there exists an optimal stationary Markov policy $ \pi^*\in\argmax_{\pi\in\Pi}J^\pi, $ and hence $J^*=J^{\pi^*}$.
Under Assumption~\ref{ass:irreducible}, all optimal stationary policies share the same anchored relative action-value function. We denote this uniquely defined common function by $Q^*\in\cQ$, so $Q^\pi=Q^*$ for every $\pi\in\Pi$ with $J^\pi=J^*$. This agreement does not require uniqueness of the optimal policy, and the common relative action-value function satisfies the Bellman optimality equation
\begin{equation}
J^*+Q^*(s,a)
=
R^*(s,a)
+\sum_{s'\in\cS}P(s'\mid s,a)\max_{b\in\cA}Q^*(s',b),
\qquad (s,a)\in\cS\times\cA.
\label{eq:opt-bellman}
\end{equation} 
We show this characterization in Appendix~\ref{app:setup-proofs}.

\section{Self-Induced Bellman Equation}\label{sec:est}
We note that each $Q^*(s,a)$ contains information about future states and rewards through continuation values generated recursively using the action maximum.
A natural way to consider optimal value inference is to first estimate $Q^*$ from the Bellman equation and then apply softmax. Define the population post-hoc softened policy by
\[
\pi_\beta^\circ(a\mid s)
\coloneqq
\frac{\exp(\beta Q^*(s,a))}{\sum_{b\in\cA}\exp(\beta Q^*(s,b))},
\qquad s\in\cS,\ a\in\cA.
\]
Given an estimator $\widehat Q^*$ of $Q^*$, the corresponding learned softened policy is
\[
\widehat\pi_\beta(a\mid s)
\coloneqq
\frac{\exp(\beta \widehat Q^*(s,a))}{\sum_{b\in\cA}\exp(\beta \widehat Q^*(s,b))},
\qquad s\in\cS,\ a\in\cA.
\]
The learned policy can then be evaluated by ordinary off-policy evaluation for a fixed policy \citep{wei2025characterization}.
However, this post-hoc route does not resolve the recursive difficulty. Each $Q^*(s,a)$ already summarizes future rewards, but its continuation values are generated recursively through action maximization. Applying softmax afterward changes the action probabilities used at future states without recomputing the continuation values under the softened policy. The policy itself is valid, but it leads to a separate policy evaluation problem. 

The continuation value mismatch has a direct inferential consequence. Let $\widehat J$ denote an off-policy estimate of the value $J^{\widehat\pi_\beta}$ of the learned softened policy. We have the decomposition
\[
\widehat J-J^*
=
\underbrace{
\widehat J-J^{\widehat\pi_\beta}}_{\text{policy evaluation error}}
+
\underbrace{
J^{\widehat\pi_\beta}-J^{\pi_\beta^\circ}}_{\text{upstream policy learning error}}
+
\underbrace{ J^{\pi_\beta^\circ}-J^*
}_{\text{post-hoc approximation bias}}.
\]
The first term is ordinary policy evaluation error. Conditional on the learned target policy, orthogonality for fixed policy evaluation can remove the first-order sensitivity of this term to evaluation nuisances, such as the policy evaluation $Q$-function and weighting nuisances, but it does not automatically eliminate the second term. Controlling the upstream policy learning error generally requires additional regularity, stability, regret or value control. 
The map defining $Q^*$ through the Bellman optimality equation is governed by the nonsmooth Bellman equation with a maximum and can remain nonregular at nonunique optima.

We instead smooth the Bellman problem by requiring policy generation and continuation value evaluation to be determined jointly. Specifically, we seek a single action-value function whose softened policy is also the policy under which that same action-value function evaluates continuation values. To formalize this construction, for any action-value function $Q\colon\cS\times\cA\to\RR$, define the softmax policy
\[
\pi_{\beta,Q}(a\mid s)
\coloneqq
\frac{\exp(\beta Q(s,a))}{\sum_{b\in\cA}\exp(\beta Q(s,b))},
\qquad s\in\cS,\ a\in\cA,
\]
and the associated softmax-smoothed value map
\[
V_\beta(Q)(s)
\coloneqq
\sum_{a\in\cA}\pi_{\beta,Q}(a\mid s)Q(s,a),
\qquad s\in\cS.
\]
With these definitions, the smoothed Bellman recursion can be closed through a stationary self-consistent fixed point. The resulting pair $(Q_\beta,J_\beta)$ is defined by the self-induced average-reward Bellman equation
\begin{equation}
J_\beta + Q_\beta(s,a)
=
R^*(s,a)+\EE\!\left[V_\beta(Q_\beta)(S')\mid S=s,A=a\right],
\qquad (s,a)\in\cS\times\cA.
\label{eq:self}
\end{equation}
Rather than learning a policy and then evaluating it, \eqref{eq:self} smooths the optimal value problem itself. We next study the resulting pair $(Q_\beta,J_\beta)$ and establish its basic properties. 

\begin{definition}[Self-induced Bellman solution]
A \emph{self-induced Bellman solution} is any pair $(Q_\beta,J_\beta)$, with $Q_\beta\colon\cS\times\cA\to\RR$ and $J_\beta\in\RR$, satisfying \eqref{eq:self}. The corresponding policy is
$
\pi_\beta\coloneqq\pi_{\beta,Q_\beta}.
$
\end{definition}
The next lemma shows that $Q_\beta$ is identifiable only up to an additive constant.

\begin{lemma}[Shift invariance]\label{lem:shift-invariance}
Suppose $(Q_\beta,J_\beta)$ satisfies \eqref{eq:self}. Then, for every constant $c\in\RR$, the pair $(Q_\beta+c\one,J_\beta)$ also satisfies \eqref{eq:self}, where $\one$ denotes the all-ones function on $\cS\times\cA$.
\end{lemma}

We present the proof in Appendix~\ref{app:self-induced-proofs}. As with the relative action-value functions for fixed policies defined in Section~\ref{sec:formu}, we remove the shift invariance by seeking $Q_\beta$ in the anchored space $\cQ$. Existence of a self-induced Bellman solution is not immediate because the policy depends on the unknown $Q_\beta$. The following result applies Brouwer's fixed-point theorem.

\begin{theorem}[Existence of a self-induced solution]\label{thm:self-induced-existence}
Under Assumption~\ref{ass:irreducible}, for every smoothing parameter $\beta>0$, there exists a scalar $J_\beta\in\RR$ and an anchored function $Q_\beta\in\cQ$ satisfying \eqref{eq:self}.
\end{theorem} 

\noindent
To state a condition that identifies the nonlinear fixed point, set $C_{\cA}\coloneqq1+2\log|\cA|$. A state-wise signed action aggregation operator $D$ is called \emph{allowable} if $(Df)(s)=\sum_{a\in\cA}w_s(a)f(s,a)$. For every $s\in\cS$, the weights satisfy $\sum_{a\in\cA}w_s(a)=1$ and $\sum_{a\in\cA}|w_s(a)|\le C_{\cA}$.

\begin{assumption}
\label{ass:uniform-signed-aggregation-identification}
For every allowable state-wise signed action aggregation operator $D$, the implication $ -j\one-q+PDq=0, q\in\cQ, j\in\RR,$ implies $q=0$ and $j=0$.
\end{assumption}

\begin{remark}
Assumption~\ref{ass:uniform-signed-aggregation-identification} rules out nontrivial perturbation directions in $Q$ that are invisible to the first-order variation of the Bellman residual, after the additive constant ambiguity has been removed by restricting $q\in\cQ$. This identification after normalization is natural in average-reward MDPs, where relative value functions are identified only up to additive constants; see \citet[Theorem~1]{wan2021learning}.
A stronger but more interpretable sufficient condition uses one-step overlap. With $ \|\mu-\nu\|_{\TV} \coloneqq \frac{1}{2} \sum_{s'\in\cS} |\mu(s')-\nu(s')| $, define the Dobrushin coefficient by $
\rho(P)
\coloneqq
\sup_{(s,a),(\widetilde s,\widetilde a)}
\left\|
P(\cdot\mid s,a)
-
P(\cdot\mid \widetilde s,\widetilde a)
\right\|_{\TV}. $
Thus, $\rho(P)$ measures the largest total variation discrepancy between one-step transition laws across state-action pairs.
In particular, $C_{\cA}\rho(P)<1$ is sufficient for Assumption~\ref{ass:uniform-signed-aggregation-identification}; see Proposition~\ref{prop:rho-sufficient}. The one-step overlap condition is commonly used in the average-reward MDP literature; see, for example, \citet{chen2026non} and references therein.
Assumption~\ref{ass:uniform-signed-aggregation-identification} provides uniform identification over all allowable signed aggregation operators, which is stronger than the local invertibility required at the population solution. As shown later in Proposition~\ref{prop:large-beta-linearization}, the derivative operator approaches that of the uniform policy over optimal actions as $\beta\to\infty$. Together with Assumption~\ref{ass:irreducible}, this suggests that Assumption~\ref{ass:uniform-signed-aggregation-identification} may be relaxed for large-$\beta$ population inference. 
We adopt this stronger condition to establish global identification and a unified nuisance estimation theory for all $\beta>0$, without relying on local arguments specific to the large-$\beta$ regime.
\end{remark}

\begin{proposition}[Uniqueness of the anchored self-induced solution]
\label{prop:self-induced-uniqueness}
Suppose Assumptions~\ref{ass:irreducible} and \ref{ass:uniform-signed-aggregation-identification} hold.
Then, for every $\beta>0$, the self-induced Bellman equation \eqref{eq:self} admits a unique anchored solution $ (Q_\beta,J_\beta)\in\cQ\times\RR.$
\end{proposition}

\begin{remark}\label{rem:self-induced-is-policy-evaluation}
Any solution of \eqref{eq:self} satisfies the ordinary Bellman equation for its induced policy.
Let $(Q_\beta,J_\beta)$ be any solution of \eqref{eq:self}, and let $\pi_\beta$ be its induced policy.
Substituting $ V_\beta(Q_\beta)(s) $ into \eqref{eq:self} gives
\[
J_\beta+Q_\beta(s,a)
=
R^*(s,a)+\sum_{s'\in\cS}P(s'\mid s,a)
\sum_{a'\in\cA}\pi_\beta(a'\mid s')Q_\beta(s',a'),
\qquad (s,a)\in\cS\times\cA.
\]
Thus, the self-induced equation is the ordinary Bellman equation for $\pi_\beta$. If $Q_\beta\in\cQ$, uniqueness of the anchored evaluation pair for a fixed policy gives $ (Q_\beta,J_\beta)=(Q^{\pi_\beta},J^{\pi_\beta}). $ Hence $Q_\beta$ both generates the softmax policy and evaluates its continuation values, while $J_\beta$ is the policy's ordinary long-run average reward under the original reward objective. This self-consistency removes the population-level split in the post-hoc route while keeping policy dependence inside the same smooth fixed-point formulation.
\end{remark}

We next compare the optimal pair $(Q^*,J^*)$ defined in Section~\ref{sec:formu} with the self-induced smoothed pair $(Q_\beta,J_\beta)$.
Define the smallest positive optimality gap by
\[
\Gamma
\coloneqq
\min
\left\{
\max_{b\in\cA}Q^*(s,b)-Q^*(s,a):
Q^*(s,a)<\max_{b\in\cA}Q^*(s,b)
\right\},
\]
with $\Gamma=+\infty$ if every action is optimal at every state. Otherwise, finiteness of $\cS$ and $\cA$ gives $\Gamma>0$.

\begin{theorem}\label{thm:exp}
Suppose that Assumptions~\ref{ass:irreducible} and~\ref{ass:uniform-signed-aggregation-identification} hold, and let $(Q_\beta,J_\beta)$ be the anchored self-induced solution of Proposition~\ref{prop:self-induced-uniqueness}. There exists a constant $C<\infty$, independent of $\beta$, such that, for every $\beta>0$, we have
\[
0\le J^*-J_\beta\le C e^{-\beta\Gamma}.
\]
If $\Gamma=+\infty$, then $J_\beta=J^*$ for every $\beta>0$.
\end{theorem}

\section{Inference for the Optimal Value}\label{sec:if}
We now construct an orthogonal score for inference on $J_\beta$ that removes the first-order effect of estimating the nuisance functions.
Unless stated otherwise, $c$ and $C$ denote positive finite constants whose values may change from line to line and do not depend on $N$, $H$, or $\beta$. Constants that may depend on $\beta$ carry an explicit $\beta$ subscript.
Let $Z_t=(S_t,A_t,R_{t+1},S_{t+1})$ denote the one-step transition at time $t$.
Let $\bar P$ denote the pooled one-step behavior law obtained by averaging over the $H$ time points, with $\bar\mu$ and $\bar\nu$ denoting its $(S,A)$- and $S'$-marginals, respectively. We suppress their dependence on $H$.
\subsection{Adjoint Weights and Orthogonal Scores}\label{sec:if-orth-decomp}
We first characterize the first-order variation of the smoothed continuation value map with respect to $Q$. For any action-value function $Q\colon\cS\times\cA\to\RR$, define
\begin{equation}\label{eq:omega}
\omega_{\beta,Q}(a\mid s)
\coloneqq
\pi_{\beta,Q}(a\mid s)
\Bigl[1+\beta(Q(s,a)-V_\beta(Q)(s))\Bigr].
\end{equation}
These weights are the coordinates of the directional derivative of $V_\beta$ at $Q$. Writing $D_QV_\beta(Q)(s)[q]$ for this derivative in direction $q$, we have $D_QV_\beta(Q)(s)[q]=\sum_{a\in\cA}\omega_{\beta,Q}(a\mid s)q(s,a)$. Since $V_\beta(Q)$ averages $Q$ under the $Q$-dependent policy $\pi_{\beta,Q}$, this derivative includes both the direct perturbation of the action values and the perturbation of the softmax probabilities induced by changing $Q$. In particular, for every $s\in\cS$, $ \sum_{a\in\cA}\omega_{\beta,Q}(a\mid s)=1, $ although the individual weights need not be nonnegative.
Evaluating these weights at $Q_\beta$, define the derivative aggregation operator
\[
(D_\beta q)(s)
\coloneqq
\sum_{a\in\cA}
\omega_{\beta,Q_\beta}(a\mid s)q(s,a).
\]
Thus, $D_\beta$ is a state-wise signed action aggregation operator whose weights sum to one.
The row-sum and bounds on the absolute weights make $D_\beta$ allowable in the sense defined in Section~\ref{sec:est}; see Lemma~\ref{lem:if-omega-tv-bound}.

\begin{assumption}\label{ass:pooled-support}

The pooled state-action and next-state marginals satisfy $\min_{(s,a)\in\cS\times\cA}\bar\mu(s,a)>0$ and $\min_{s'\in\cS}\bar\nu(s')>0$, uniformly over $H$.
\end{assumption}

\begin{remark}
Assumption~\ref{ass:pooled-support} provides a convenient sufficient condition for identification and stable inference. More generally, what matters is sufficient coverage to identify the target functional and ensure stable estimation of the associated Bellman and adjoint quantities. This suggests weaker, target-dependent coverage conditions based on occupancy or adjoint-weight ratios. We retain the stronger assumption to avoid introducing additional target-specific conditions.
\end{remark}

The derivative operator $D_\beta$ characterizes the first-order variation of the Bellman residual with respect to $Q$.
For any square-integrable one-step function $f$, write $\EE_{\bar P}[f(Z)]:=\int f\,d\bar P=H^{-1}\sum_{t=0}^{H-1}\EE[f(Z_t)]$.
We call a square-integrable function
$\lambda_\beta\colon\cS\times\cA\to\RR$
an \emph{adjoint weight} associated with $D_\beta$ if
\[\EE_{\bar P}[\lambda_\beta(S,A)]=1 \text{   and   }
\EE_{\bar P}\!\left[
\lambda_\beta(S,A)
\left(
q(S,A)-(D_\beta q)(S')
\right)
\right]
=0
\]
for every $q\in\cQ$.
The second condition requires the adjoint weight to annihilate the first-order effect of perturbing $Q_\beta$ in any admissible direction. Since $D_\beta$ captures both the direct perturbation of $Q_\beta$ inside the continuation value and the induced perturbation of the softened policy $\pi_{\beta,Q_\beta}$, the orthogonality condition removes their combined first-order effect.

\begin{proposition}[Existence and uniqueness of the adjoint weight]
\label{prop:adjoint-existence}
Suppose Assumptions~\ref{ass:uniform-signed-aggregation-identification} and \ref{ass:pooled-support} hold.
Then there exists a unique adjoint weight $\lambda_\beta\colon\cS\times\cA\to\RR$ associated with $D_\beta$.
\end{proposition}

For a transition $Z=(S,A,R,S')$ and a generic weight function $\lambda\colon\cS\times\cA\to\RR$, define
\[
\varphi_\beta(Z;J,Q,\lambda)
\coloneqq
J+\lambda(S,A)(R-J+V_\beta(Q)(S')-Q(S,A)).
\]
When the nuisance arguments are omitted, we write $\varphi_\beta(Z):=\varphi_\beta(Z;J_\beta,Q_\beta,\lambda_\beta)$.
The adjoint weight removes first-order sensitivity to nuisance perturbations.

\begin{proposition}[Neyman orthogonality]
\label{prop:neyman-orthogonality}
Suppose that Assumptions~\ref{ass:uniform-signed-aggregation-identification} and \ref{ass:pooled-support} hold. Then the map $(J,Q,\lambda)\mapsto\EE_{\bar P}[\varphi_\beta(Z;J,Q,\lambda)]$ has vanishing first-order directional derivatives at $(J_\beta,Q_\beta,\lambda_\beta)$. Specifically, for every $j\in\RR$, every $q\in\cQ$, and every square-integrable $\ell$, we have
\[
\begin{aligned}
\left.\frac{d}{dt}\EE_{\bar P}[\varphi_\beta(Z;J_\beta+tj,Q_\beta,\lambda_\beta)]\right|_{t=0}&=0,\\
\left.\frac{d}{dt}\EE_{\bar P}[\varphi_\beta(Z;J_\beta,Q_\beta+tq,\lambda_\beta)]\right|_{t=0}&=0,\\
\left.\frac{d}{dt}\EE_{\bar P}[\varphi_\beta(Z;J_\beta,Q_\beta,\lambda_\beta+t\ell)]\right|_{t=0}&=0.
\end{aligned}
\]
\end{proposition}

For the cross-fitted estimator, split the $N$ independent trajectories into a fixed number $K\ge2$ of folds $I_1,\ldots,I_K$, where $I_k$ is the set of trajectory indices in fold $k$. We allow any fixed $K$ for generality and assume $ cN\le |I_k|\le CN$ for $k=1,\ldots,K$.
Taking $K=2$ is sufficient for the theoretical results. Let $\mathcal D_{-k}\coloneqq\{\tau_i:i\notin I_k\}$ denote the training trajectories outside fold $k$.
Let $\widehat J_{-k}$, $\widehat Q_{-k}$, and $\widehat\lambda_{-k}$ be the estimators trained on $\cD_{-k}$.
For each fold $k$, we take $\widehat Q_{-k}$ to be the anchored representative of the estimated action-value function. 
For notation simplicity, we suppress the $\beta$ subscript for these nuisance estimators. With $Z_{i,t}=(S_{i,t},A_{i,t},R_{i,t+1},S_{i,t+1})$, $t=0,\ldots,H-1$, the cross-fitted estimator is
\[
\widehat J_\beta
\coloneqq
\frac1N\sum_{k=1}^K\sum_{i\in I_k}
\frac1H\sum_{t=0}^{H-1}
\varphi_\beta(Z_{i,t};\widehat J_{-k},\widehat Q_{-k},\widehat\lambda_{-k}).
\]

\subsection{Theoretical Results on Inference}\label{sec:if-clts}

To derive the asymptotic normality of the estimator $\widehat J_\beta$, we first impose the following assumptions.

\begin{assumption}
\label{ass:if-envelope}

There exists $C<\infty$ such that $|R|\le C$ almost surely.
\end{assumption}

\begin{assumption} \label{ass:optimal-bellman-innovation-nondegenerate}
Let $V^*(s)\coloneqq\max_{b\in\cA}Q^*(s,b)$, and set $v_{\min}^*\coloneqq\min_{(s,a)\in\cS\times\cA}\Var(R+V^*(S')\mid S=s,A=a)$. Assume $v_{\min}^*>0$.
\end{assumption}

This condition requires $R+V^*(S')$ to have positive conditional variance at every state-action pair. The variation may arise from either reward noise or random state transitions, so deterministic rewards are allowed. The condition only excludes state-action pairs for which $R+V^*(S')$ is conditionally deterministic. A similar condition is used in \citet{shi2022statistical}.

\begin{assumption}
  \label{ass:if-consistency}
The nuisance estimators satisfy \[
\max_{1\le k\le K}
\left(
|\widehat J_{-k}-J_\beta|
+\|\widehat Q_{-k}-Q_\beta\|_{\bar\nu,2}
+\|\widehat\lambda_{-k}-\lambda_\beta\|_{\bar\mu,2}
\right)
=o_p(1).\]
\end{assumption}

To state the inference results, we introduce a shorthand for the fold-specific nuisance remainder term.
For each fold $k$, define
\begin{equation}\label{eq:bk}
\begin{split}
b_k
\coloneqq{}&
\beta\|\widehat Q_{-k}-Q_\beta\|_{\bar\nu,2}^{2}
+\|\widehat\lambda_{-k}-\lambda_\beta\|_{\bar\mu,2}
|\widehat J_{-k}-J_\beta| \\
&\quad+ \|\widehat\lambda_{-k}-\lambda_\beta\|_{\bar\mu,2}
\|\widehat Q_{-k}-Q_\beta\|_{\bar\nu,2} +
\beta\|\widehat\lambda_{-k}-\lambda_\beta\|_{\bar\mu,2}
\|\widehat Q_{-k}-Q_\beta\|_{\bar\nu,2}^{2}.
\end{split}
\end{equation}
For the fixed-$H$ central limit theorem, we impose the following nuisance estimation rate.
\begin{assumption}[Fixed-$H$ second-order rate]
  \label{ass:if-fixed-rate}
The fold-specific nuisance remainder terms satisfy $ \max_{1\le k\le K} b_k =o_p(N^{-1/2}).$
\end{assumption}

For the diverging-$H$ regime, we impose the following assumptions to control the remainder term and the time dependence of the behavior chain.
\begin{assumption}[Diverging-$H$ second-order rate]\label{ass:if-growing-rate}
  The nuisance remainder terms defined in \eqref{eq:bk} satisfy
$\max_{1\le k\le K} b_k =o_p\bigl((NH)^{-1/2}\bigr).$
\end{assumption}

\begin{assumption}[Behavior-process mixing]
\label{ass:behavior-process-mixing}
Let $X_t:=(S_t,A_t)$. For a trajectory of horizon $H$, define the behavior-process strong-mixing coefficient at lag $m$ by
\[
\alpha_H^b(m)
:=
\sup_{0\le t\le H-m-1}
\sup_{\substack{
A\in\sigma(X_0,\ldots,X_t)\\
B\in\sigma(X_{t+m},\ldots,X_{H-1})
}}
\left|
\PP(A\cap B)-\PP(A)\PP(B)
\right|,
\qquad
1\le m\le H-1.
\]
Assume $\sup_H\sum_{m=1}^{H-1}\alpha_H^b(m)<\infty$, uniformly over the relevant horizons. 
\end{assumption}

\begin{remark}
\label{rem:doeblin-overlap}

Assumption~\ref{ass:behavior-process-mixing} imposes uniform summability of the mixing coefficients for the potentially nonstationary behavior process. Related mixing conditions appear in \citet[Assumption~3]{jenish2009central}, \citet[Assumption~6]{bai2012joint}, and \citet[Condition~A3(ii)]{shi2022statistical}.
\end{remark}

For both regimes, define
\[
\xi_{i,t}
\coloneqq
\varphi_\beta(Z_{i,t})-J_\beta
=
\lambda_\beta(S_{i,t},A_{i,t})
(R_{i,t+1}-J_\beta+V_\beta(Q_\beta)(S_{i,t+1})
-Q_\beta(S_{i,t},A_{i,t})).
\]
Set $\sigma^2\coloneqq H^{-1}\sum_{t=0}^{H-1}\EE[\xi_{i,t}^2]=\EE_{\bar P}[\lambda_\beta(S,A)^2\Var(R+V_\beta(Q_\beta)(S')\mid S,A)]$. We suppress the dependence of $\sigma$ on $\beta$ and $H$. By the Bellman equation, $\xi_{i,t}$ forms a martingale-difference sequence under the natural trajectory filtration; see Appendix~\ref{sec:app-martingale-qv}. Thus, $\Var(H^{-1}\sum_{t=0}^{H-1}\xi_{i,t})=\sigma^2/H$. We show in Proposition~\ref{prop:bellman-innovation-variance-lower-bounds} of the appendix that $\sigma^2$ is uniformly bounded away from zero and infinity whenever $\beta$ is sufficiently large relative to $1/v_{\min}^*$. This yields the following fixed and diverging horizon inference result.
\begin{theorem}[Inference under fixed and diverging horizons]
\label{thm:clt-N}
Suppose Assumptions~\ref{ass:irreducible}, \ref{ass:uniform-signed-aggregation-identification}, \ref{ass:pooled-support}, \ref{ass:if-envelope}, \ref{ass:optimal-bellman-innovation-nondegenerate} and~\ref{ass:if-consistency} hold and $\beta\ge C/v_{\min}^*$ for a sufficiently large constant $C$. Consider either of the following regimes:
\begin{enumerate}[label=(\roman*)]
\item $H$ is fixed, $N\to\infty$, and Assumption~\ref{ass:if-fixed-rate} holds.
\item $H\to\infty$, and Assumptions~\ref{ass:if-growing-rate} and~\ref{ass:behavior-process-mixing} hold.
\end{enumerate}
In either regime, we have
\[
\frac{\sqrt{NH}(\widehat J_\beta-J_\beta)}{\sigma}
\xrightarrow{d}
N(0,1).
\]
Moreover, if $\sqrt{NH}e^{-\beta\Gamma}\to0$, then we have
\[
\frac{\sqrt{NH}(\widehat J_\beta-J^*)}{\sigma}
\xrightarrow{d}
N(0,1).
\]
\end{theorem}

The Bellman equation makes the leading score increments a martingale difference sequence, so their cross-time covariances vanish and the variance of the trajectory-level leading term scales as $H^{-1}$. When $H$ diverges, nuisance-induced empirical fluctuations are controlled by decomposing them into a martingale and a predictable term, with Assumption~\ref{ass:behavior-process-mixing} used for the variation stabilization. Together, these arguments yield $\sqrt{NH}$-rate asymptotic normality; see Appendices~\ref{sec:app-growing-oracle}--\ref{sec:app-martingale-qv}.

\begin{remark}
\label{rem:efficiency-interpretation}
When the optimal action is unique at every state, $J^*$ is locally regular and our estimator attains its semiparametric efficiency bound as the smoothing bias becomes negligible. At nonunique optima, however, $J^*$ is generally not pathwise differentiable and need not admit a classical efficient influence function \citep{luedtke2016statistical}. For the smoothed target $J_\beta$, the estimator admits a semiparametric efficiency characterization under fixed horizons; see Appendix~\ref{sec:app-efficiency-pe-limit}.
\end{remark}

\begin{remark}
\label{rem:large-beta-pe}
Let $\pi^{\mathrm{unif}}$ denote the policy that assigns equal probability to the maximizers of $Q^*(s,\cdot)$ at each state $s$. Let $D_{\pi^{\mathrm{unif}}}$ be the ordinary policy evaluation averaging operator for this policy, and let $\lambda^{\mathrm{unif}}$ be its normalized adjoint weight under the same pooled behavior law. Under Assumptions~\ref{ass:irreducible}, \ref{ass:uniform-signed-aggregation-identification}, \ref{ass:pooled-support}, and~\ref{ass:if-envelope}, as $\beta\to\infty$, we have $\pi_\beta\to\pi^{\mathrm{unif}}$ and $D_\beta\to D_{\pi^{\mathrm{unif}}}.$ Moreover, $\|\lambda_\beta-\lambda^{\mathrm{unif}}\|_{\bar\mu,2}\to0$ and
\[
\left|
\sigma^2
-
\EE_{\bar P}\!\left[
\{\lambda^{\mathrm{unif}}(S,A)\}^2
\Var\!\left(R+V^*(S')\mid S,A\right)
\right]
\right|
\to0.
\]
Thus, when optimal actions are tied, the inference quantities approach their ordinary policy evaluation counterparts for the uniform policy over those actions. In particular, the limiting influence function is exactly the efficient influence function for the fixed policy $\pi^{\rm unif}$; see Appendix~\ref{sec:app-efficiency-pe-limit}.
\end{remark}

For feasible inference, we use the following variance estimator in both regimes. For $i\in I_k$, define the fold-specific $
\widehat\xi_{i,t}
\coloneqq
\varphi_\beta(
Z_{i,t};
\widehat J_{-k},
\widehat Q_{-k},
\widehat\lambda_{-k}
)
-
\widehat J_{-k}. $
Set
\[
\widehat\sigma
\coloneqq
\left(
\frac1{NH}
\sum_{k=1}^K
\sum_{i\in I_k}
\sum_{t=0}^{H-1}
\widehat\xi_{i,t}^2
\right)^{1/2}.
\]
For fixed $H$, consistency follows by averaging the bounded trajectory-level quantities $H^{-1}\sum_{t=0}^{H-1}\widehat\xi_{i,t}^2$ across independent trajectories. When $H$ grows, consistency follows from the quadratic variation argument under Assumption~\ref{ass:behavior-process-mixing}. 

\begin{corollary}
\label{cor:studentized-optimal-value}
Suppose the conditions for the $J^*$ inference of Theorem~\ref{thm:clt-N} hold under the corresponding regime. Then we have
\[
\frac{\widehat\sigma}{\sigma}
\xrightarrow{p}1,
\qquad
\frac{\sqrt{NH}(\widehat J_\beta-J^*)}{\widehat\sigma}
\xrightarrow{d}N(0,1).
\]
Consequently, for $\alpha\in(0,1)$, let $z_{1-\alpha/2}$ denote the $1-\alpha/2$ standard normal quantile. In either regime, an asymptotically valid $100(1-\alpha)$\% confidence interval for $J^*$ is
\[
\left[
\widehat J_\beta
-z_{1-\alpha/2}\frac{\widehat\sigma}{\sqrt{NH}},
\;
\widehat J_\beta
+z_{1-\alpha/2}\frac{\widehat\sigma}{\sqrt{NH}}
\right].
\]
\end{corollary}

\section{Nuisance Estimation}\label{sec:nuisance-estimation}
\subsection{Self-Induced Bellman Pair Estimation}\label{sec:tabular-fast-rate}
We first study nuisance estimation under an independent one-step sampling scheme, which isolates the statistical difficulty of estimating the nonlinear self-induced Bellman solution and the adjoint weight. Sections~\ref{sec:tabular-fast-rate} and~\ref{sec:plug-in-adjoint-erm} establish nuisance estimation rates in this setting, including an explicit finite-sample bound for the Bellman pair. Section~\ref{sec:dependent-trajectory-nuisance-rates} then returns to the trajectory setting and shows that the same estimators can be trained on all observed transitions under the behavior-process mixing condition used for diverging horizon inference.
Throughout Sections~\ref{sec:tabular-fast-rate}--\ref{sec:plug-in-adjoint-erm}, let $Z_i=(S_i,A_i,R_i,S_i')$, $i=1,\ldots,n$, denote i.i.d. observations of individual transitions. A similar one-step sampling setup is used by \citet{kallus2022efficiently} to derive nuisance estimation rates. For trajectory data, such independent transition samples can be obtained by drawing one time point independently from each trajectory. Under uniform sampling over $0,\ldots,H-1$, the resulting transitions are i.i.d. with common law $\bar P$.
We write $\EE$ for expectation under this one-step sampling setup, let $\Pn$ denote the empirical measure over $Z_1,\ldots,Z_n$, and let $\bar\mu$ be the state-action marginal of the common one-step law. We assume that $\bar\mu(s,a)>0$ for every $(s,a)\in\mathcal S\times\mathcal A$. When the common one-step law is the pooled law from Section~\ref{sec:if}, this $\bar\mu$ is its state-action marginal.

We first develop a minimax estimator of $(Q_\beta,J_\beta)$ defined in Section~\ref{sec:est}.
Estimating this pair from one-step data requires converting a conditional Bellman restriction into observable unconditional moments. A single transition gives a noisy realized residual. We therefore minimize a regularized minimax criterion over test functions of the realized Bellman residual.
The additional difficulty is that the residual is nonlinear in $Q$, because the next state value map $V_\beta(Q)$ is self-induced. We show that $Q_\beta$ and $J_\beta$ are uniformly bounded in $\beta$; see Corollary~\ref{cor:app-smoothed-target-bounded}. We therefore choose fixed bounds $B_Q,B_J<\infty$ such that $\|Q_\beta\|_\infty\le B_Q$ and $|J_\beta|\le B_J$ for all relevant $\beta$. We optimize over the resulting fixed bounded subset of $\cQ\times\RR$, on which the continuous empirical criterion admits an exact minimizer.

For $(Q,J)\in\cQ\times\mathbb R$, define the conditional Bellman residual $m_{Q,J}(s,a):=\EE[R-J+V_\beta(Q)(S')-Q(S,A)\mid S=s,A=a]$.
The self-induced Bellman equation is equivalent to the conditional moment restriction $m_{Q_\beta,J_\beta}(s,a)=0$ for $(s,a)\in\mathcal S\times\mathcal A$.
For a fixed constant $B_G>0$, define the uniformly bounded test-function class $\cG:=\{g:\mathcal S\times\mathcal A\to\mathbb R:\|g\|_\infty\le B_G\}$.
Because the state-action space is finite, $\cG$ is a finite-dimensional, symmetric, star-shaped function class (that is, $g\in\cG$ implies $-g\in\cG$, and $g\in\cG$ implies $tg\in\cG$ for every $t\in[0,1]$; see \citet[Chapter~14]{wainwright2019high}).
For a fixed regularization tuning parameter $\rho>0$, the minimax estimator is any minimizer satisfying
\[
\begin{aligned}
(\widehat Q,\widehat J)
&\in
\argmin_{\substack{Q\in\cQ,\ J\in\RR\\ \|Q\|_\infty\le B_Q,\ |J|\le B_J}}\widehat L_n(Q,J),\\
\widehat L_n(Q,J)
&\coloneqq
\sup_{g\in\cG}
\biggl\{
2\Pn\!\left[g(S,A)(R-J+V_\beta(Q)(S')-Q(S,A))\right]
-\rho\Pn[g(S,A)^2]
\biggr\}.
\end{aligned}
\]

To establish the finite-sample rate, we first obtain quantitative identification through the Bellman residual using an exact secant representation of the nonlinear residual; see Proposition~\ref{prop:bellman-residual-identification} in Appendix~\ref{app:fast-rate-proofs}.
We combine the residual identification bound with the regularized empirical criterion to obtain a finite-sample rate.
\begin{theorem}
\label{thm:fast_rate}
Suppose Assumptions~\ref{ass:uniform-signed-aggregation-identification} and~\ref{ass:if-envelope} hold. Then there exists $C<\infty$, independent of $n$, $\beta$, and $\delta$, such that, for any $n\ge2$ and $\delta\in(0,1)$, with probability at least $1-\delta$, we have
\[
\|\widehat Q-Q_\beta\|_{\bar\mu,2}+|\widehat J-J_\beta|
\le
C
\sqrt{\frac{|\cS||\cA|\log n+\log(1/\delta)}{n}}.
\]
\end{theorem}

\subsection{Adjoint Weight Estimation}\label{sec:plug-in-adjoint-erm}
We continue under the one-step sampling setup of Section~\ref{sec:tabular-fast-rate}. Let $\lambda_\beta$ denote the adjoint weight associated with $D_\beta$.  We first derive an error bound for $\widehat\lambda_\beta$ in terms of the estimation error of a generic first-stage estimator $\widehat Q$ of $Q_\beta$, and then specialize the result to the Bellman estimator developed in Section~\ref{sec:tabular-fast-rate}. To estimate $\lambda_\beta$, we construct an empirical analogue of its adjoint moment equation.
Because the inner supremum ranges over the anchored space $\cQ$, the empirical quadratic term may be degenerate along poorly sampled directions. We therefore use the stabilized norm $\|q\|_\tau^2:=\Pn[q(S,A)^2]+(n|\cS||\cA|)^{-1}\sum_{s\in\cS}\sum_{a\in\cA}q(s,a)^2$, for $q\in\cQ$. The empirical quadratic term $\Pn[q(S,A)^2]$ can vanish for a nonzero element of
$\cQ$.  
The added term makes $\|\cdot\|_\tau$ a norm on $\cQ$ and keeps the empirical suprema finite on every sample path.

Recall the derivative weights $\omega_{\beta,Q_\beta}$ from \eqref{eq:omega}. The population criterion for $\lambda$ is
\[
G_\beta(\lambda)
\coloneqq
\sup_{q\in\cQ}
\left\{2\EE\!\left[
\lambda(S,A)
\left(
q(S,A)-\sum_{a\in\cA}\omega_{\beta,Q_\beta}(a\mid S')q(S',a)
\right)
\right]-\|q\|_{\bar\mu,2}^2\right\}
+(\EE[\lambda(S,A)]-1)^2.
\]
In particular, $G_\beta(\lambda_\beta)=0$, and the criterion measures violation of the adjoint moment equation and the normalization constraint. The adjoint weight $\lambda_\beta$ is uniformly bounded over $\beta$ and the relevant horizons; see Corollary~\ref{cor:target-adjoint-bound}.
Choose a sufficiently large fixed constant $B_\Lambda$ and define $\Lambda:=\{\lambda:\mathcal S\times\mathcal A\to\mathbb R:\|\lambda\|_\infty\le B_\Lambda\}$, so that $\lambda_\beta\in\Lambda$.
This criterion is infeasible because the derivative weights evaluated at $Q_\beta$ are unknown. We therefore replace them by derivative weights evaluated at the first stage estimator $\widehat Q$. Define $\widehat\omega_\beta(a\mid s):=\pi_{\beta,\widehat Q}(a\mid s)[1+\beta(\widehat Q(s,a)-V_\beta(\widehat Q)(s))]$.
Replacing $\omega_{\beta,Q_\beta}$ in the empirical criterion with $\widehat\omega_\beta$ yields the feasible criterion
\[
\widehat G_\beta(\lambda)
\coloneqq
\sup_{q\in\cQ}
\left\{
2\Pn\!\left[
\lambda(S,A)
\left(
q(S,A)-\sum_{a\in\cA}\widehat\omega_\beta(a\mid S')q(S',a)
\right)
\right]
-\|q\|_\tau^2
\right\}
+(\Pn\lambda-1)^2.
\]
We estimate the adjoint weight by any exact minimizer $\widehat\lambda_\beta\in\argmin_{\lambda\in\Lambda}\widehat G_\beta(\lambda)$. The estimator obeys the following generic bound. We present the proof in Appendix~\ref{app:plug-in-adjoint-erm-proofs}.

\begin{theorem}\label{thm:main}
Suppose Assumption~\ref{ass:uniform-signed-aggregation-identification} holds and let $\widehat Q$ be a generic first-stage estimator of $Q_\beta$ satisfying $\|\widehat Q-Q_\beta\|_{\bar\mu,2}=O_P(r_Q)$, where $r_Q\to0$. Then we have
\[
\|\widehat\lambda_\beta-\lambda_\beta\|_{\bar\mu,2}
=
O_P\left( 
\left[
\sqrt{\frac{\log n}{n}}
+\beta r_Q
\right]^{1/2}
\right).
\]
\end{theorem}
The contribution from replacing $Q_\beta$ by $\widehat Q$ is of order $\beta r_Q$, since the derivative weight map is $O(\beta)$-Lipschitz in $Q$ on bounded sets.
Combining the adjoint bound above with the Bellman estimator $\widehat Q$ from Theorem~\ref{thm:fast_rate} yields the following rate.
\begin{corollary}
\label{cor:density-first-stage}
Suppose the conditions of Theorems~\ref{thm:fast_rate} and \ref{thm:main} hold. Then we have
\[
\|\widehat\lambda_\beta-\lambda_\beta\|_{\bar\mu,2}
=
O_P\left( 
\left[
\sqrt{\frac{\log n}{n}}
+
\beta
\sqrt{\frac{\log n}{n}}
\right]^{1/2}
\right).
\]
\end{corollary}

\subsection{Nuisance Estimation from Dependent Trajectories}
\label{sec:dependent-trajectory-nuisance-rates}
We now return to the trajectory setting of Section~\ref{sec:if} and use all observed transitions for nuisance estimation. We form the Bellman and adjoint estimators by replacing $\Pn$ in Sections~\ref{sec:tabular-fast-rate} and~\ref{sec:plug-in-adjoint-erm} with the empirical average over the $NH$ observed transitions, and replacing $n$ by $NH$ in the stabilizing term of $\|\cdot\|_\tau$.  The resulting nuisance estimators satisfy the following rates.

\begin{theorem}
\label{thm:dependent-trajectory-nuisance-rates}
Suppose $NH\to\infty$. Under Assumptions~\ref{ass:uniform-signed-aggregation-identification}, \ref{ass:pooled-support}, \ref{ass:if-envelope}, and~\ref{ass:behavior-process-mixing}, the estimator satisfies
$$
\|\widehat Q-Q_\beta\|_{\bar\mu,2}
+
|\widehat J-J_\beta|
=
O_P((NH)^{-1/2}).
$$
Using $\widehat Q$ as the first-stage estimator, we have
$$
\|\widehat\lambda_\beta-\lambda_\beta\|_{\bar\mu,2}
=
O_P
\left(
\sqrt{1+\beta}\,(NH)^{-1/4}
\right).
$$
\end{theorem}

\begin{remark}
\label{rem:dependent-nuisance-verifies-if}
Theorem~\ref{thm:fast_rate} gives a finite-sample high-probability bound under independent one-step sampling, whereas Theorem~\ref{thm:dependent-trajectory-nuisance-rates} uses all $NH$ dependent transitions and gives an asymptotic $O_P((NH)^{-1/2})$ Bellman pair rate under Assumption~\ref{ass:behavior-process-mixing}.
Consider trajectory-level cross-fitting with a fixed number of folds. Each training sample contains a fixed fraction of the $N$ independent trajectories, so Theorem~\ref{thm:dependent-trajectory-nuisance-rates} applies foldwise with the same $NH$ rate. If the conditions of Theorem~\ref{thm:dependent-trajectory-nuisance-rates} hold and $\beta=o(\sqrt{NH})$, then the all-transition nuisance estimators satisfy Assumptions~\ref{ass:if-consistency} and~\ref{ass:if-growing-rate}; see Appendix~\ref{app:dependent-inference-verification}. If, in addition, $\sqrt{NH}e^{-\beta\Gamma}\to0$, then the smoothing-bias condition in Theorem~\ref{thm:clt-N} also holds. Thus, for fixed $\Gamma>0$, a simple sufficient range for valid optimal value inference is
$$
\log(NH)\ll\beta\ll\sqrt{NH}.
$$
\end{remark}
\section{Numerical Results}\label{sec:numerical}

\subsection{Synthetic Experiments}\label{sec:num-synthetic}

The synthetic study uses an MDP with $|\cS|=100$ states and $|\cA|=100$ actions. Five actions share the optimum in every state, while a designated runner-up is separated from them, so the design represents the nonregular setting induced by nonunique optimal actions.
Each Monte Carlo sample contains $N$ independent trajectories observed for $H$ transitions. We consider a stationary behavior policy mixture and linear and polynomial time-varying behavior policy mixtures with comparable average overlap. We use an intuitive model-based rule to calibrate $\beta$, selecting the smallest candidate whose smoothed value is within $0.5\%$ of the corresponding model benchmark. We use $\beta=4$ based on the simulated MDP. We evaluate coverage of $J^*$, root mean squared error (RMSE), confidence interval (CI) length, and the distribution of the studentized error.
All feasible methods use trajectory-level cross-fitting, with trajectories split into two folds. Proposed denotes our proposed method. Oracle uses the same orthogonal score but substitutes the population quantities $(Q_\beta,J_\beta,\lambda_\beta)$. It is infeasible and serves only to isolate the finite-sample cost of nuisance estimation. The doubly robust comparator, denoted Hardmax-DR, first estimates an action-value function from the Bellman optimality equation and forms the corresponding deterministic greedy policy, then refits the relative action value and average reward for that fixed learned policy and applies a fixed-policy doubly robust correction. The Hardmax sequential DR comparator instead evaluates the deterministic greedy policy induced by the fitted action-value function using temporal-difference corrections weighted by cumulative importance ratios.

\begin{figure}[!t]
\centering
\begin{tabular}{@{}ccc@{}}
\begin{minipage}{0.31\textwidth}
\centering
\small{(a) Coverage versus $H$}\\[-1pt]
\includegraphics[width=\textwidth]{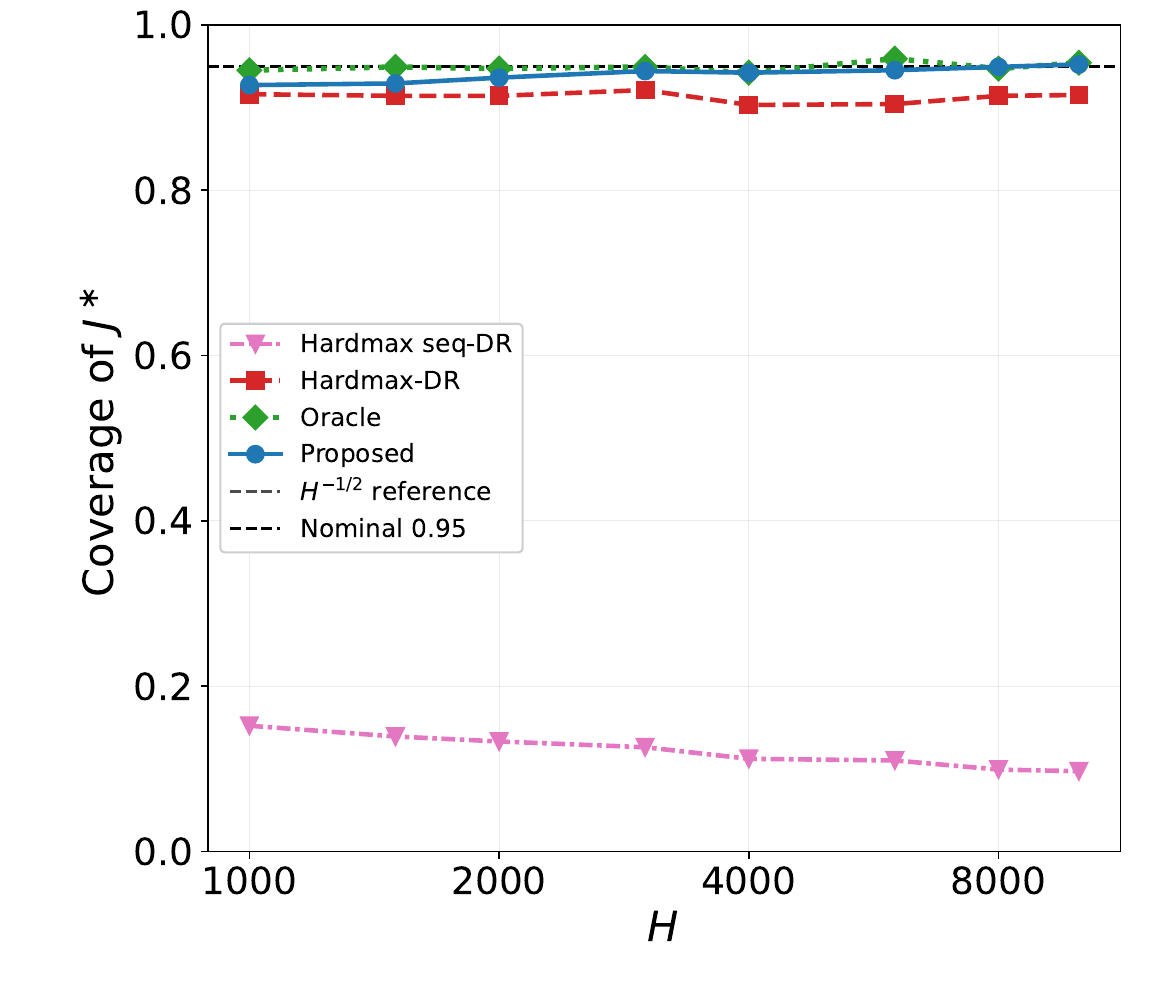}
\end{minipage} &
\begin{minipage}{0.31\textwidth}
\centering
\small{(b) RMSE versus $H$}\\[-1pt]
\includegraphics[width=\textwidth]{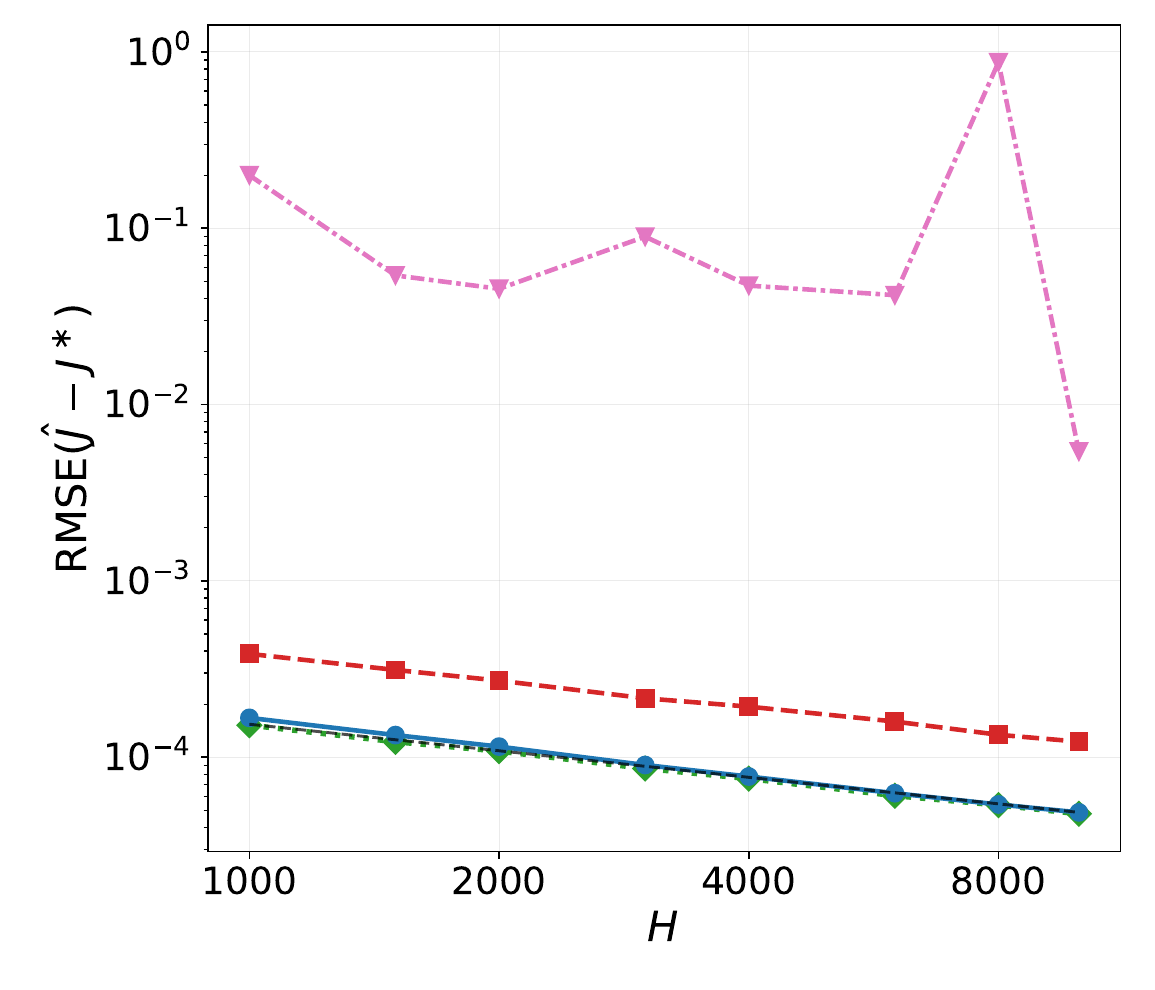}
\end{minipage} &
\begin{minipage}{0.31\textwidth}
\centering
\small{(c) CI length versus $H$}\\[-1pt]
\includegraphics[width=\textwidth]{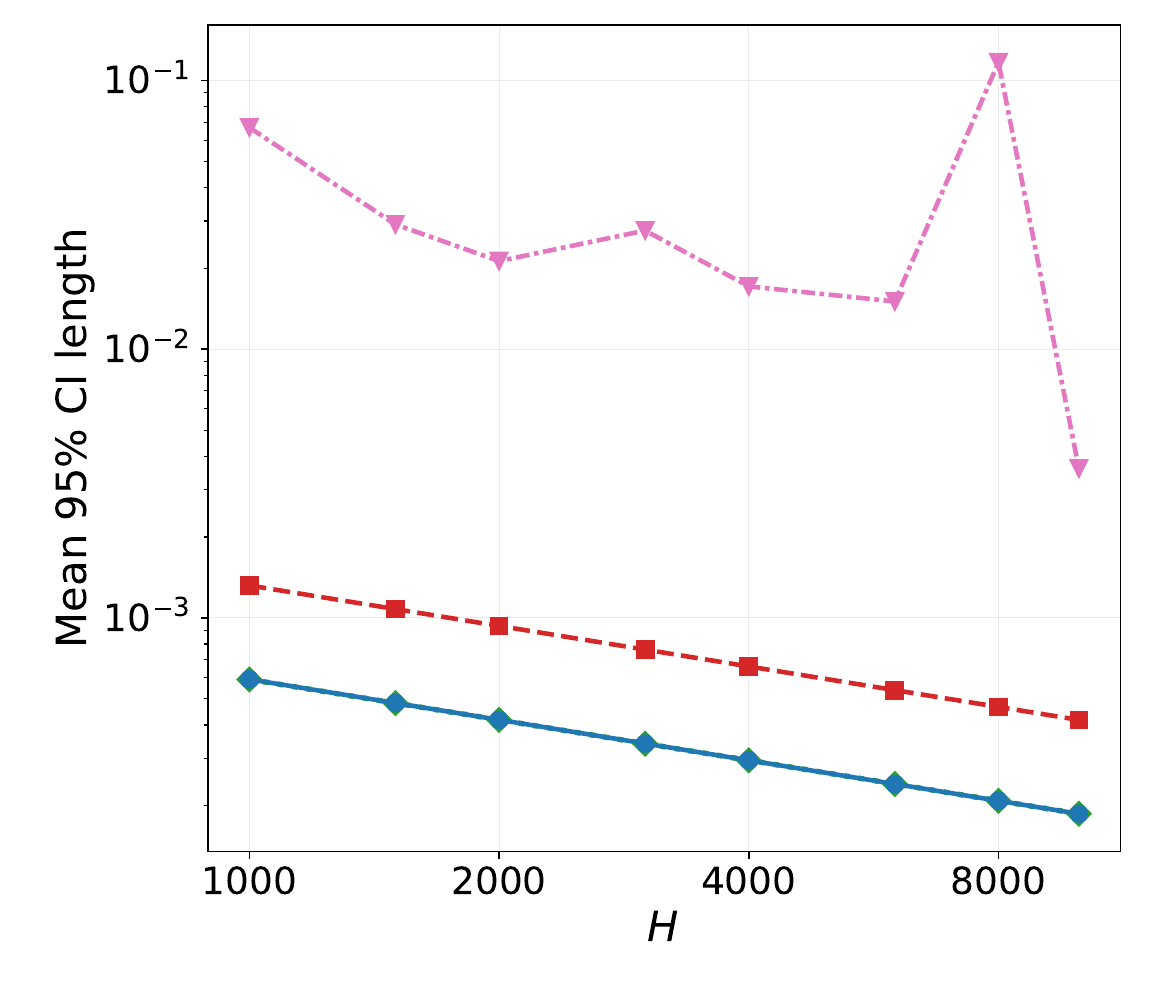}
\end{minipage}\\[-2pt]
\begin{minipage}{0.31\textwidth}
\centering
\small{(d) Coverage versus $N$}\\[-1pt]
\includegraphics[width=\textwidth]{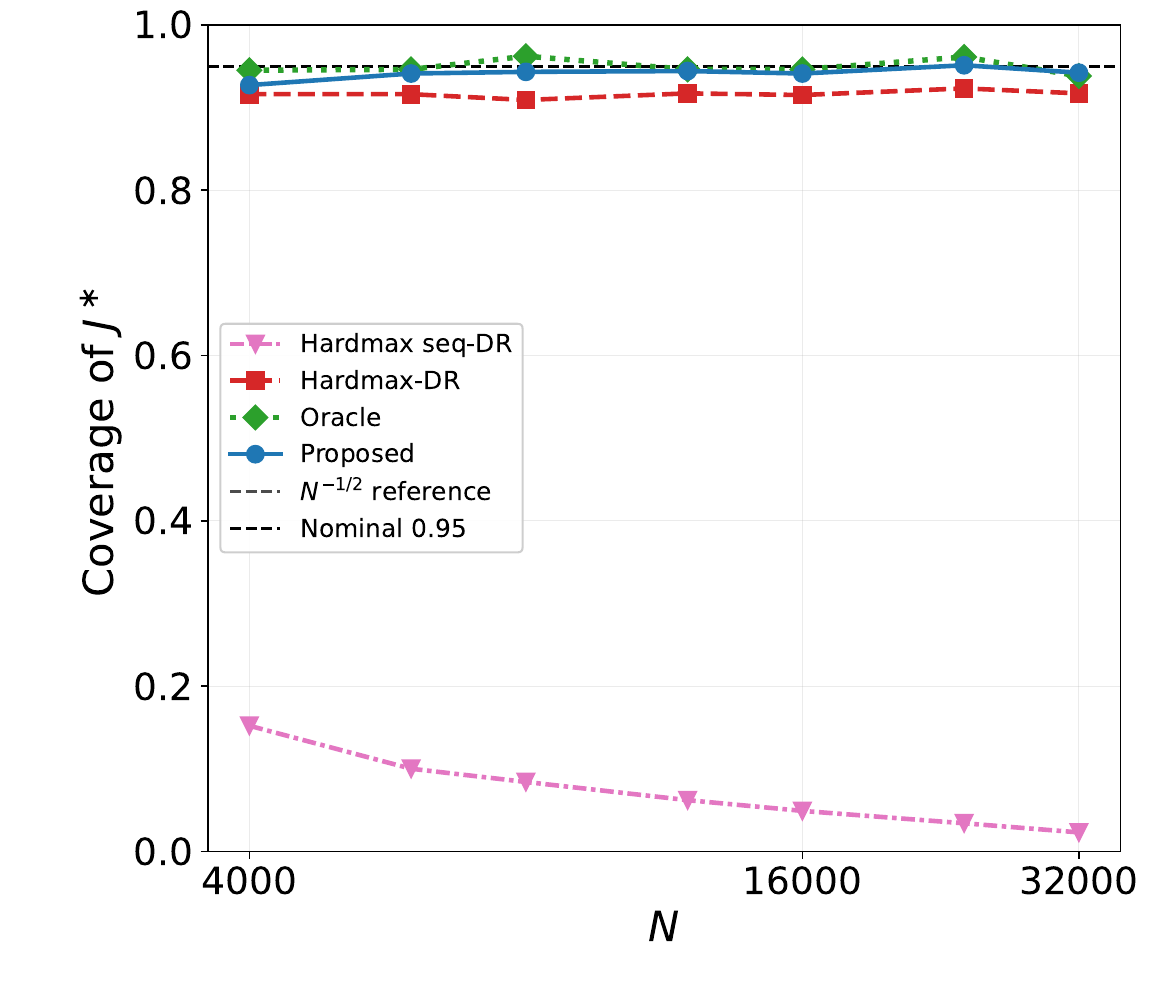}
\end{minipage} &
\begin{minipage}{0.31\textwidth}
\centering
\small{(e) RMSE versus $N$}\\[-1pt]
\includegraphics[width=\textwidth]{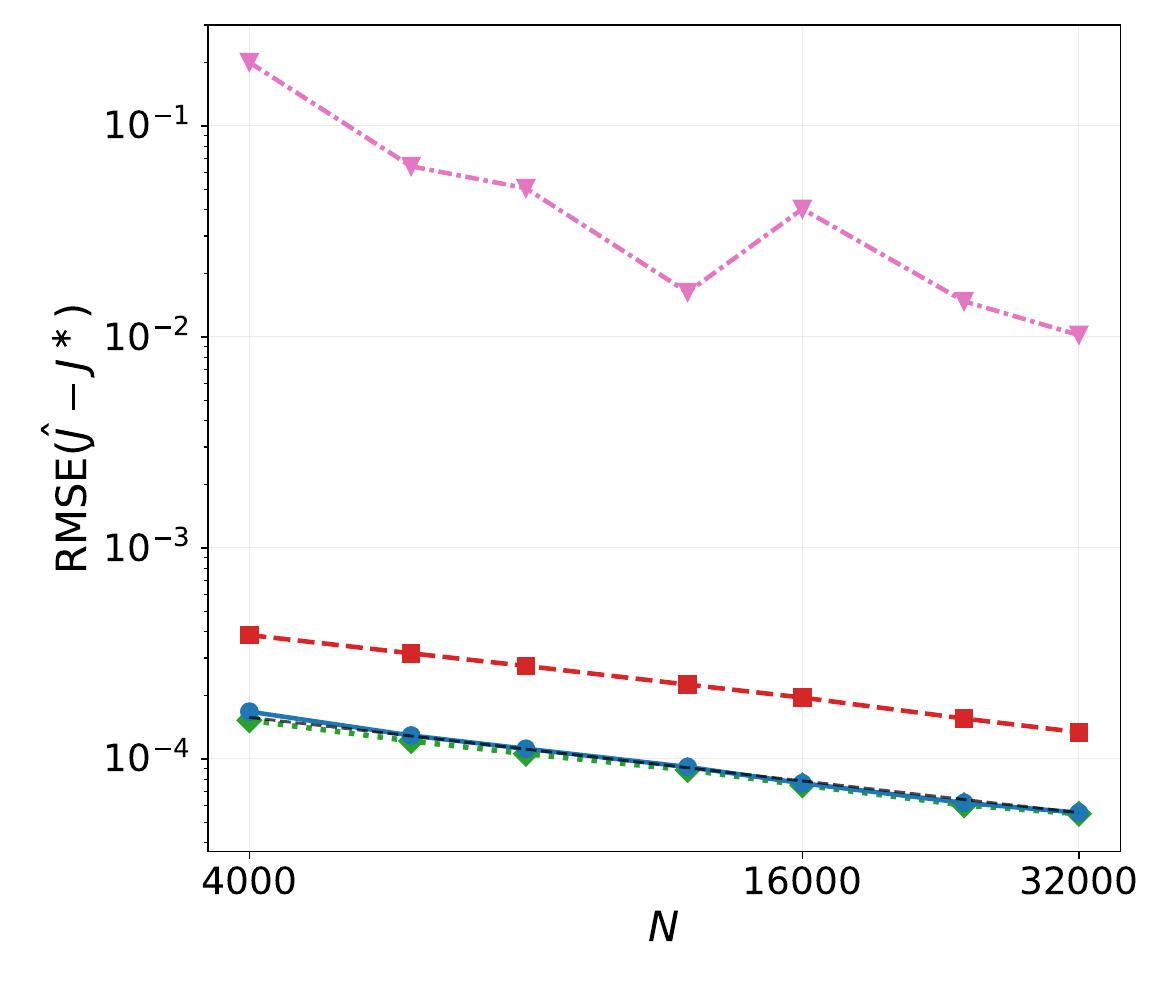}
\end{minipage} &
\begin{minipage}{0.31\textwidth}
\centering
\small{(f) CI length versus $N$}\\[-1pt]
\includegraphics[width=\textwidth]{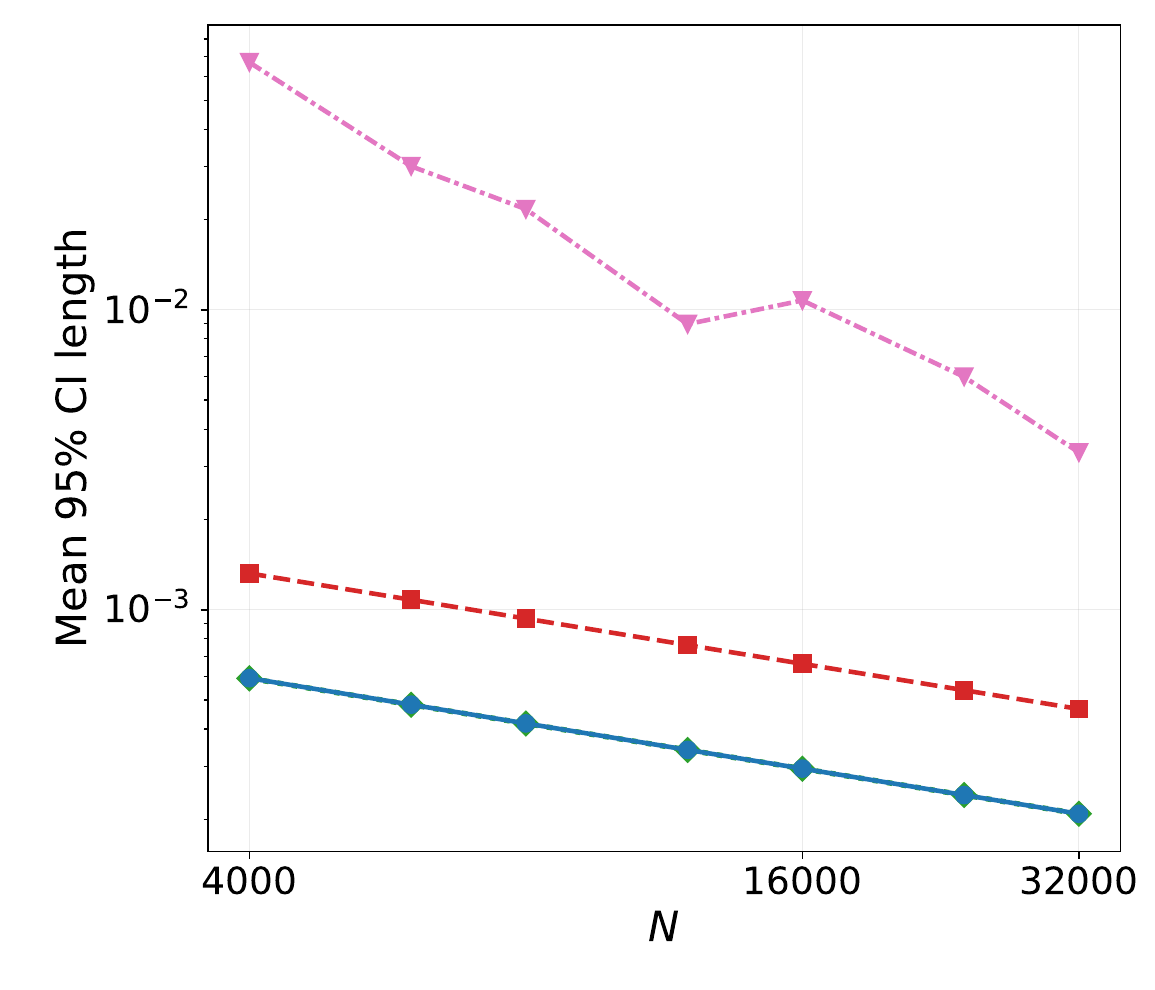}
\end{minipage}
\end{tabular}
\caption{Finite-sample performance across trajectory lengths and sample sizes. The top row varies trajectory length $H$ at $N=4{,}000$, and the bottom row varies $N$ at $H=1{,}000$. Columns show Monte Carlo coverage of $J^*$, RMSE relative to $J^*$, and confidence interval length.
}
\label{fig:synthetic-stress-main}
\end{figure}

In Figure~\ref{fig:synthetic-stress-main}, the top row varies the trajectory length $H$ at fixed $N$, while the bottom row varies $N$ at fixed $H$. Proposed remains close to Oracle as either source of information increases. Its RMSE decreases approximately at the $H^{-1/2}$ and $N^{-1/2}$ reference rates, respectively, while coverage stays near the nominal level. Hardmax-DR shows persistent undercoverage, while the sequential Hardmax comparator is substantially more variable and can have errors and intervals orders of magnitude larger.
Figure~\ref{fig:synthetic-methods} varies $H$ at $N=4{,}000$. Recall that $\pi^{\mathrm{unif}}$ is uniform over the five tied optimal actions, and set $u_t=t/(H-1)$. The fixed, linear, and polynomial regimes use $\rho_t=0.5$, $\rho_t=0.2+0.6u_t$, and $\rho_t=0.2+0.6(3u_t^2-2u_t^3)$, respectively. In each regime, the behavior policy is $\pi_t^b(a\mid s)=(1-\rho_t)/|\mathcal A|+\rho_t\pi^{\mathrm{unif}}(a\mid s)$. All three schedules have average mixture weight $0.5$ and hence comparable average overlap. Proposed remains close to Oracle across all three regimes, with stable coverage and steadily decreasing RMSE and interval length as $H$ grows. In contrast, the sequential Hardmax comparator becomes especially unstable under time-varying behavior policies, with large spikes in RMSE and interval length.
At $N=32{,}000$ and $H=4{,}000$, the studentized errors in Figure~\ref{fig:synthetic-hist} are centered near zero and closely follow the standard normal benchmark under the fixed, linear, and polynomial behavior policy regimes. 

\begin{figure}[!t]
\centering
\begin{tabular}{@{}ccc@{}}
\begin{minipage}{0.31\textwidth}
\centering
\small{(a) Fixed: coverage}\\[-1pt]
\includegraphics[width=\textwidth]{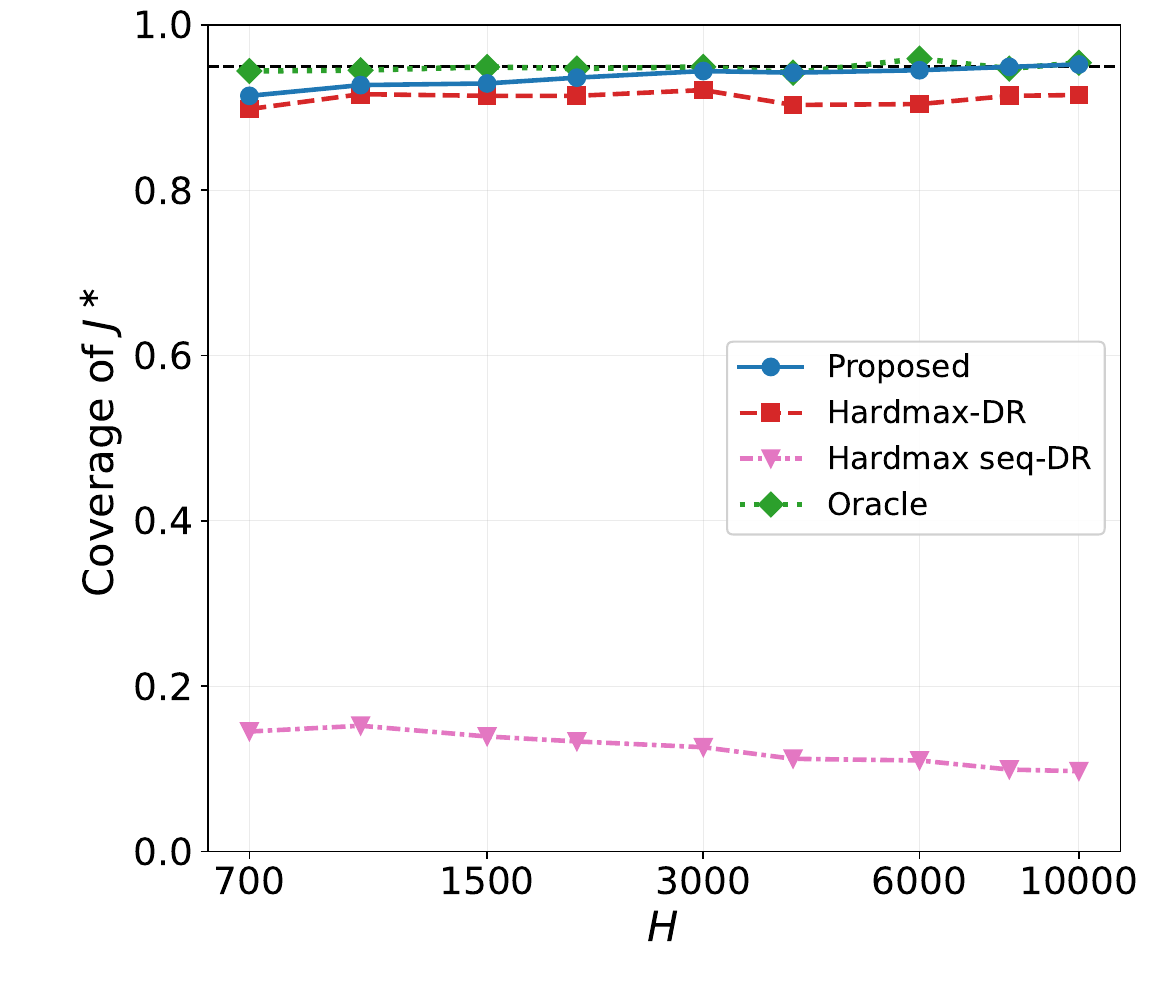}
\end{minipage} &
\begin{minipage}{0.31\textwidth}
\centering
\small{(b) Fixed: RMSE}\\[-1pt]
\includegraphics[width=\textwidth]{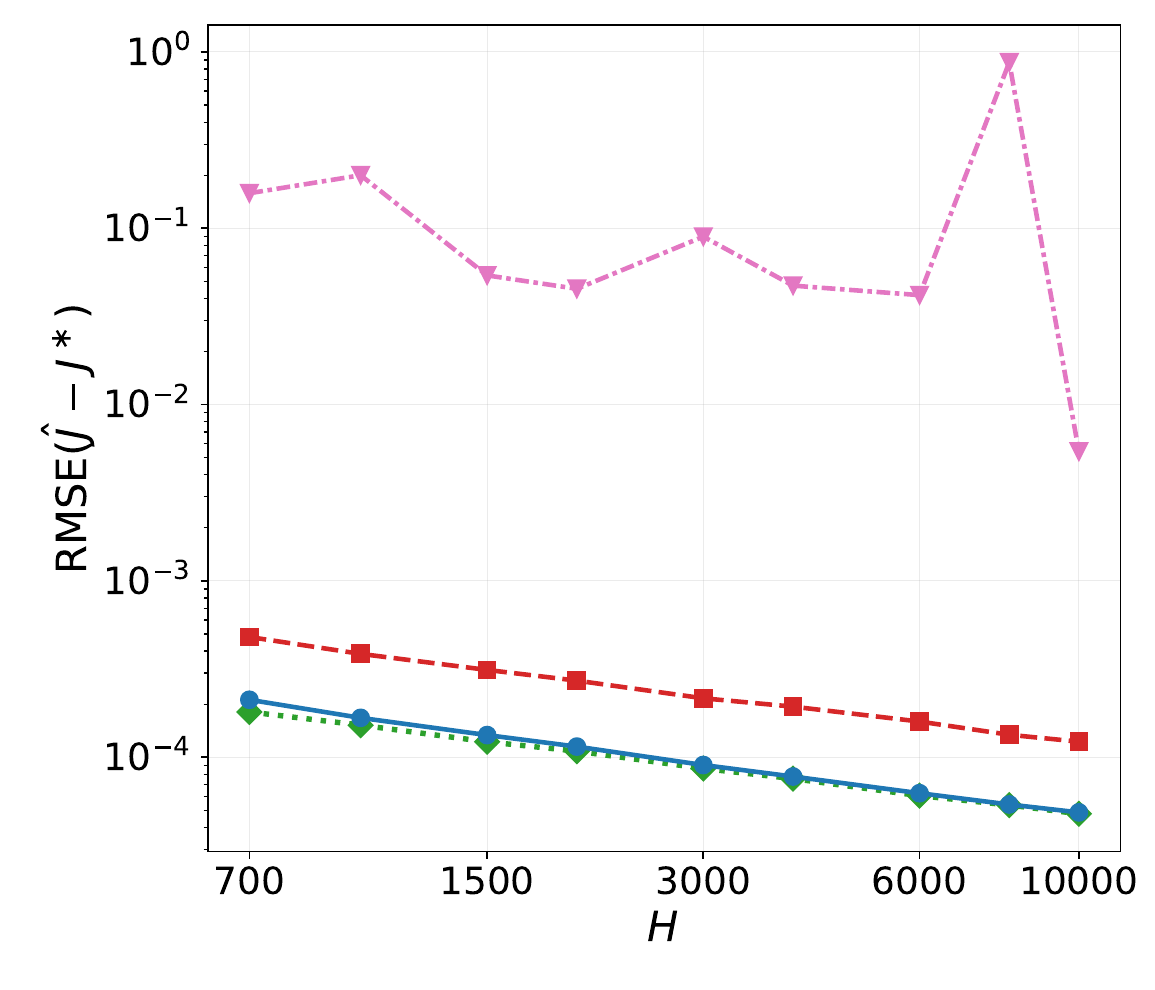}
\end{minipage} &
\begin{minipage}{0.31\textwidth}
\centering
\small{(c) Fixed: CI length}\\[-1pt]
\includegraphics[width=\textwidth]{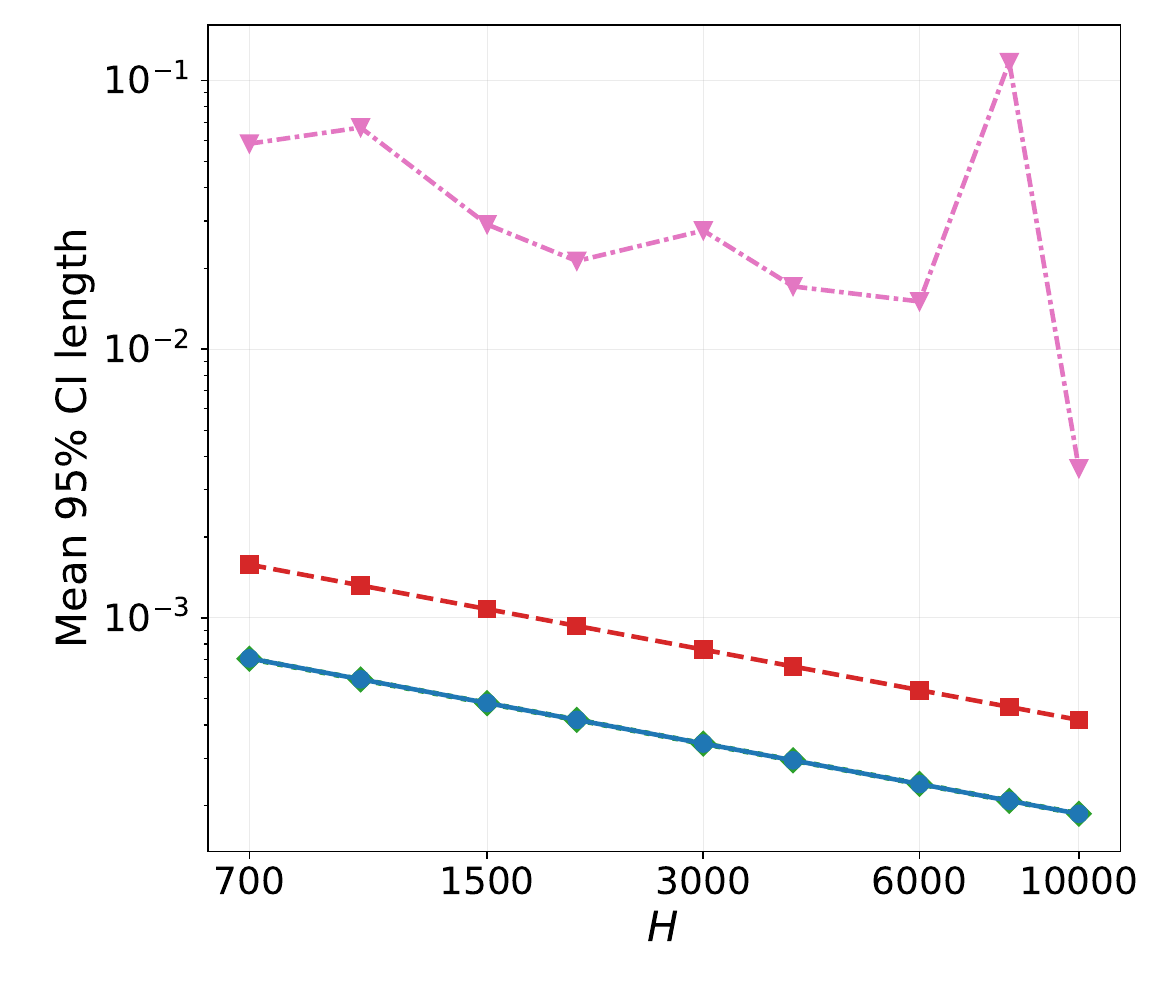}
\end{minipage}\\[-2pt]
\begin{minipage}{0.31\textwidth}
\centering
\small{(d) Linear: coverage}\\[-1pt]
\includegraphics[width=\textwidth]{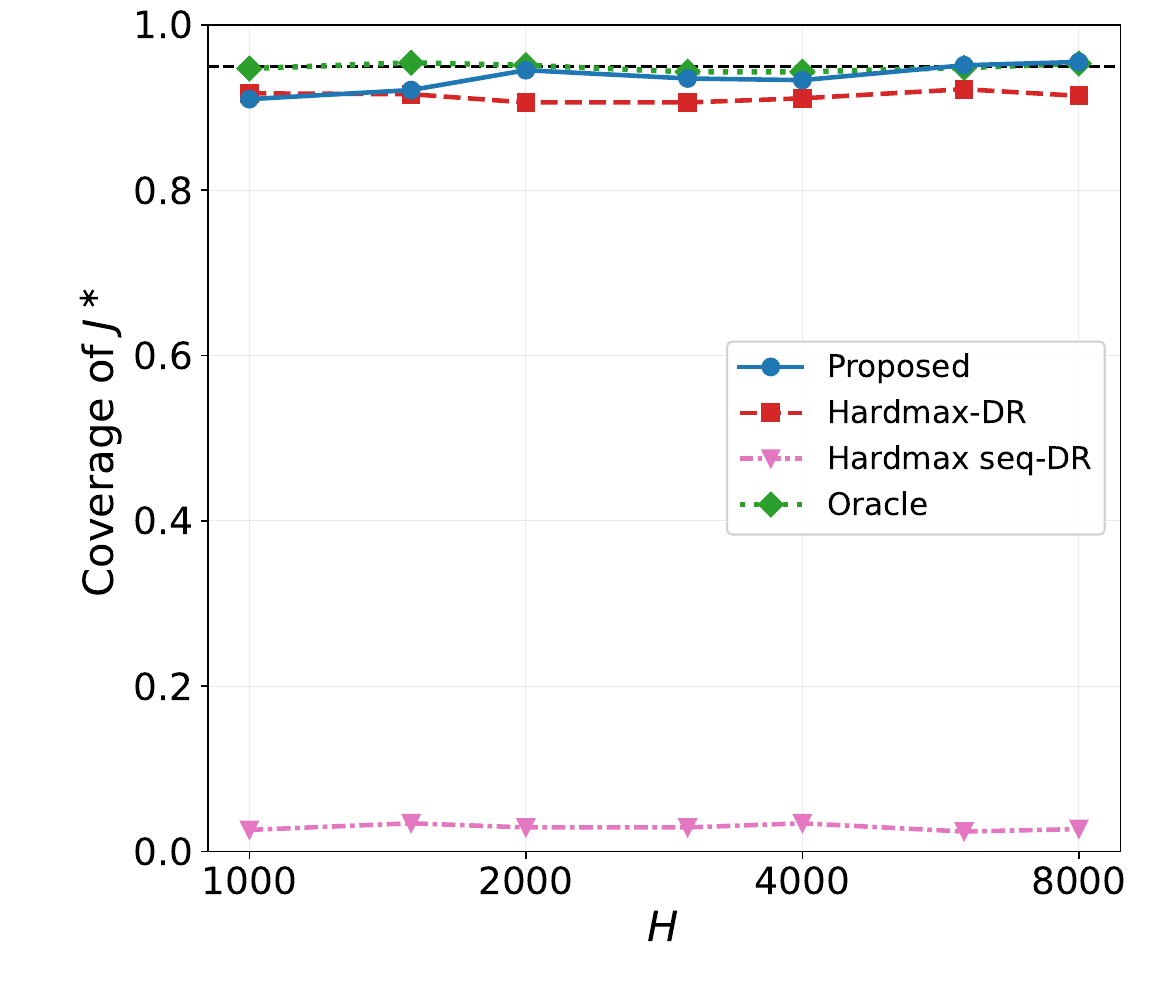}
\end{minipage} &
\begin{minipage}{0.31\textwidth}
\centering
\small{(e) Linear: RMSE}\\[-1pt]
\includegraphics[width=\textwidth]{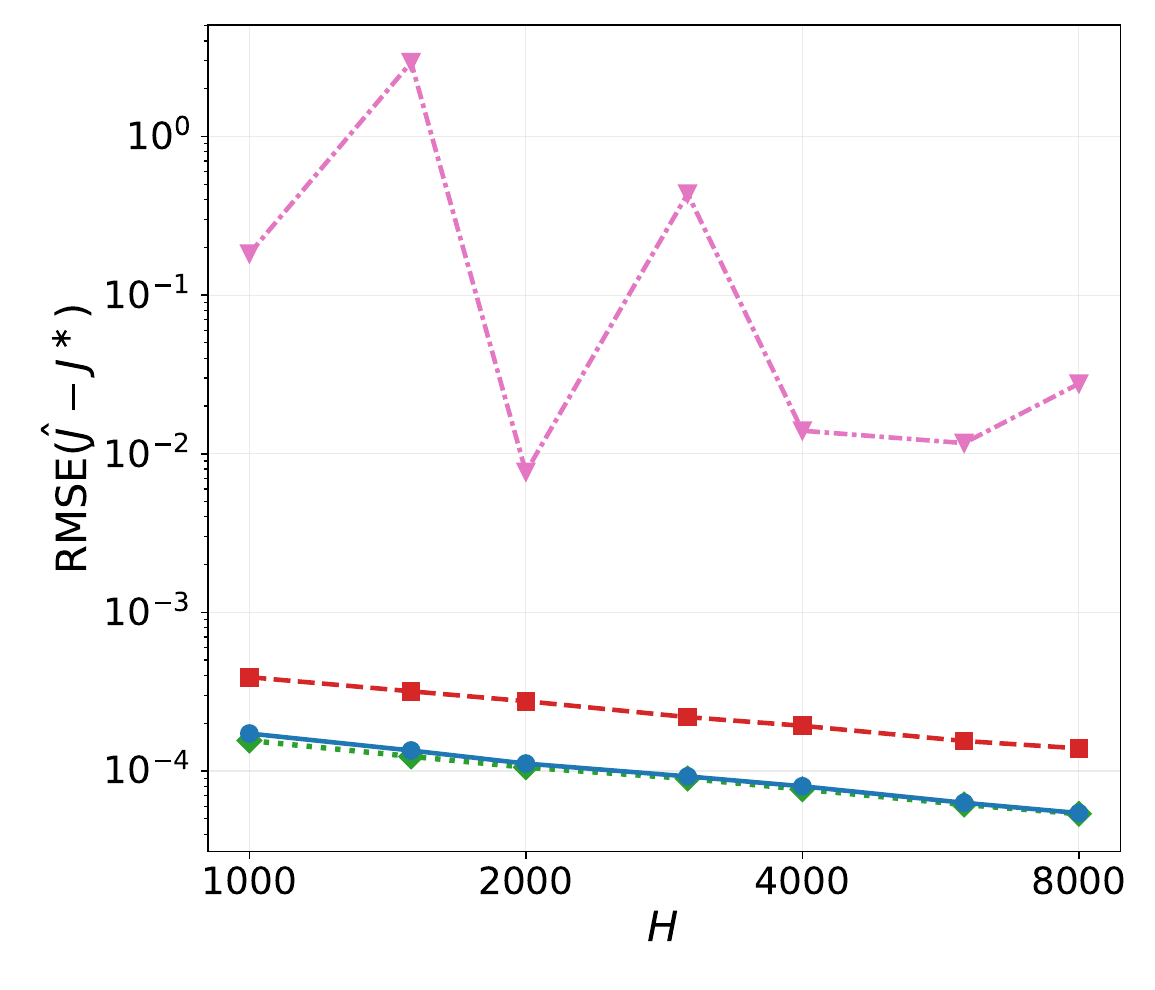}
\end{minipage} &
\begin{minipage}{0.31\textwidth}
\centering
\small{(f) Linear: CI length}\\[-1pt]
\includegraphics[width=\textwidth]{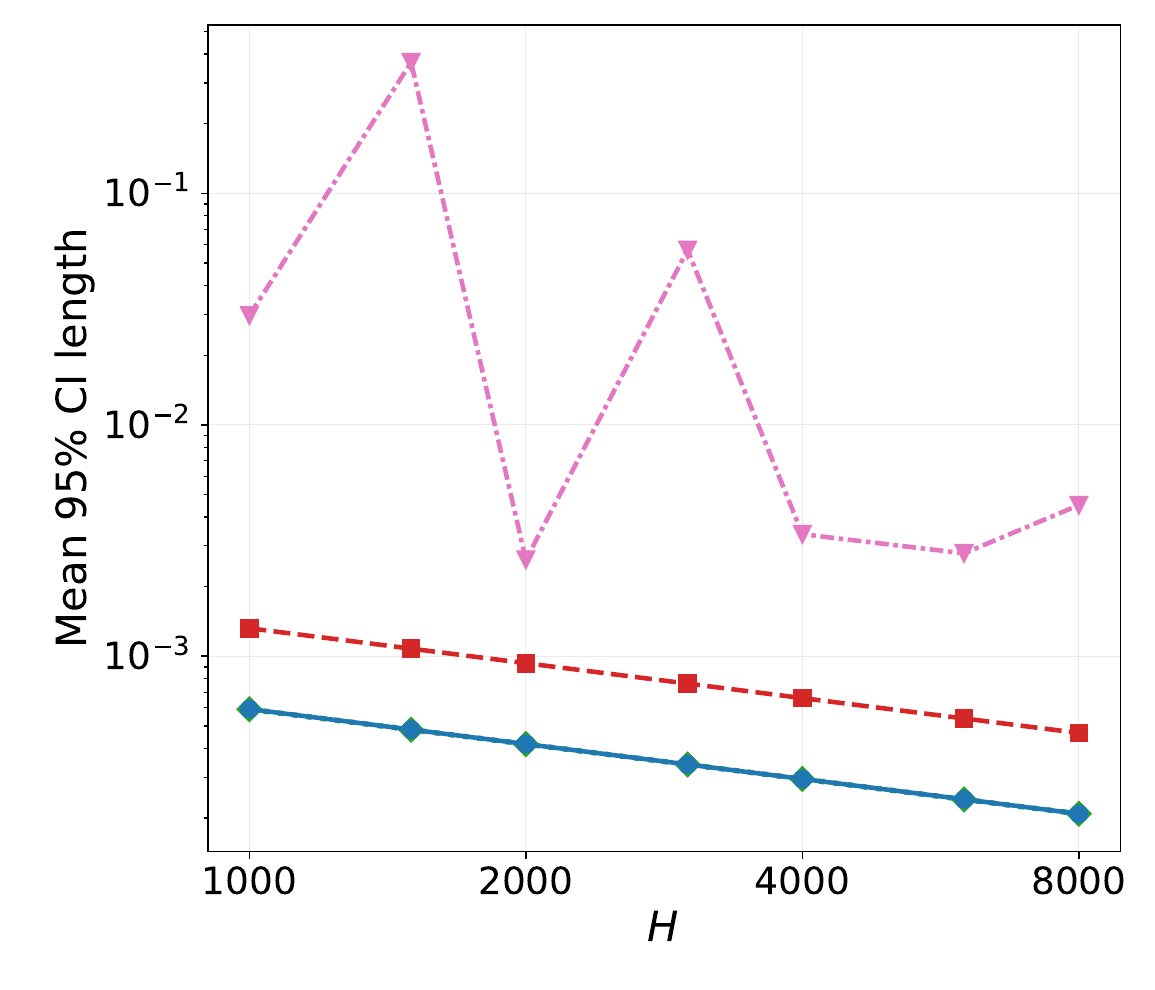}
\end{minipage}\\[-2pt]
\begin{minipage}{0.31\textwidth}
\centering
\small{(g) Polynomial: coverage}\\[-1pt]
\includegraphics[width=\textwidth]{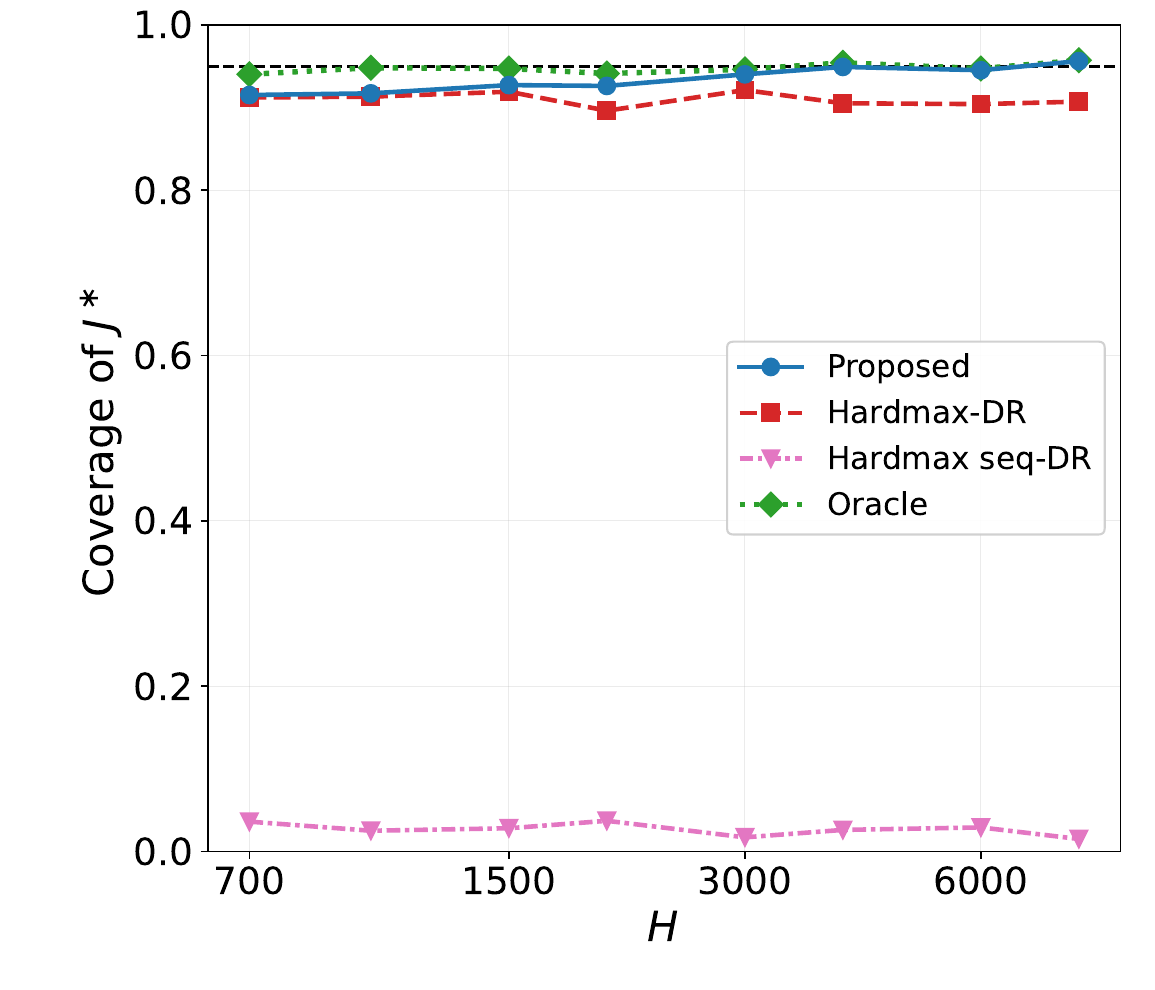}
\end{minipage} &
\begin{minipage}{0.31\textwidth}
\centering
\small{(h) Polynomial: RMSE}\\[-1pt]
\includegraphics[width=\textwidth]{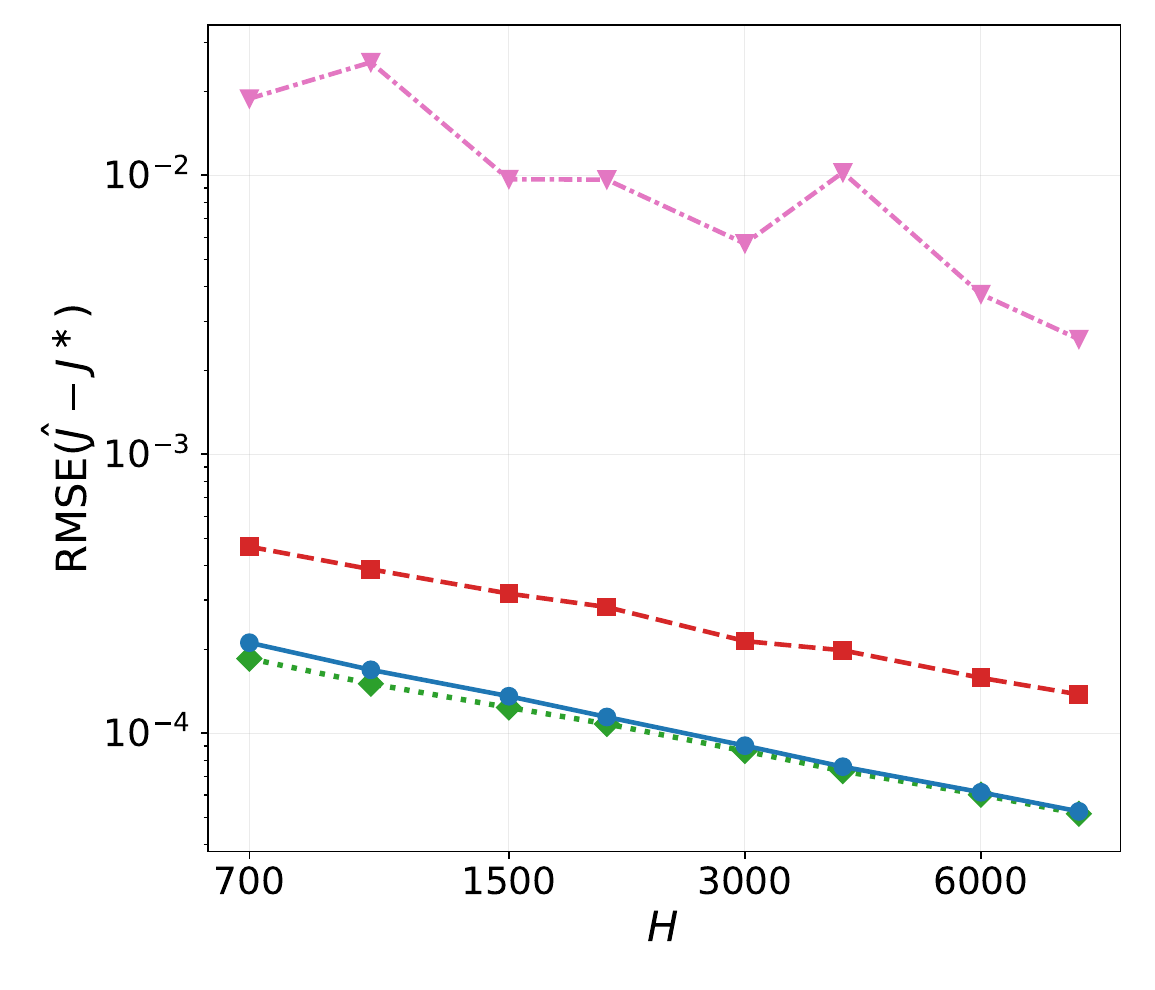}
\end{minipage} &
\begin{minipage}{0.31\textwidth}
\centering
\small{(i) Polynomial: CI length}\\[-1pt]
\includegraphics[width=\textwidth]{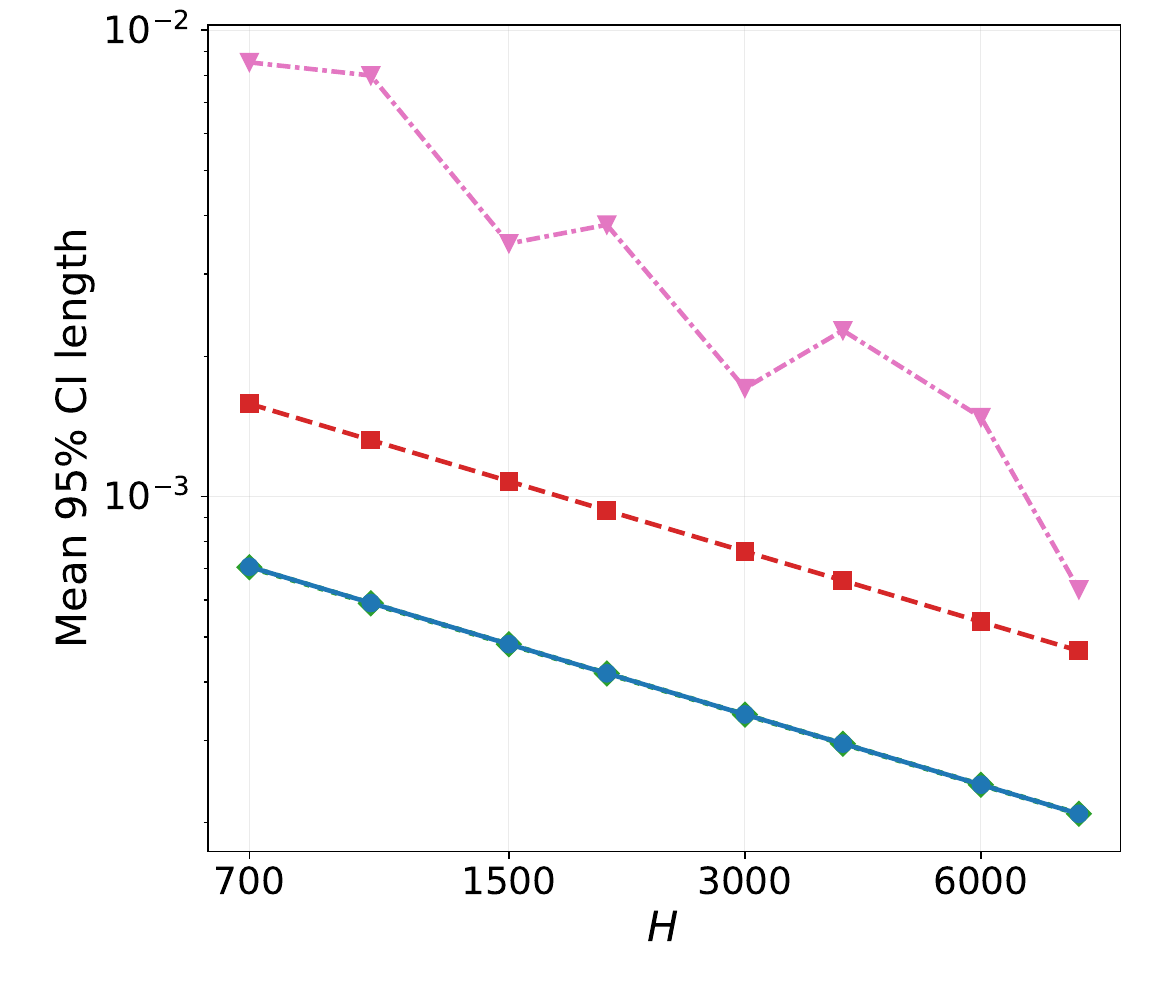}
\end{minipage}
\end{tabular}
\caption{Method comparisons across behavior policy regimes. Rows correspond to the stationary mixture and the linear and polynomial time-varying mixtures, which have comparable average overlap. Columns show coverage of $J^*$, RMSE, and confidence interval length.
}
\label{fig:synthetic-methods}
\end{figure}

\begin{figure}[!t]
\centering
\begin{tabular}{@{}ccc@{}}
\begin{minipage}{0.31\textwidth}
\centering
\small{(a) Fixed mixture}\\[-1pt]
\includegraphics[width=\textwidth]{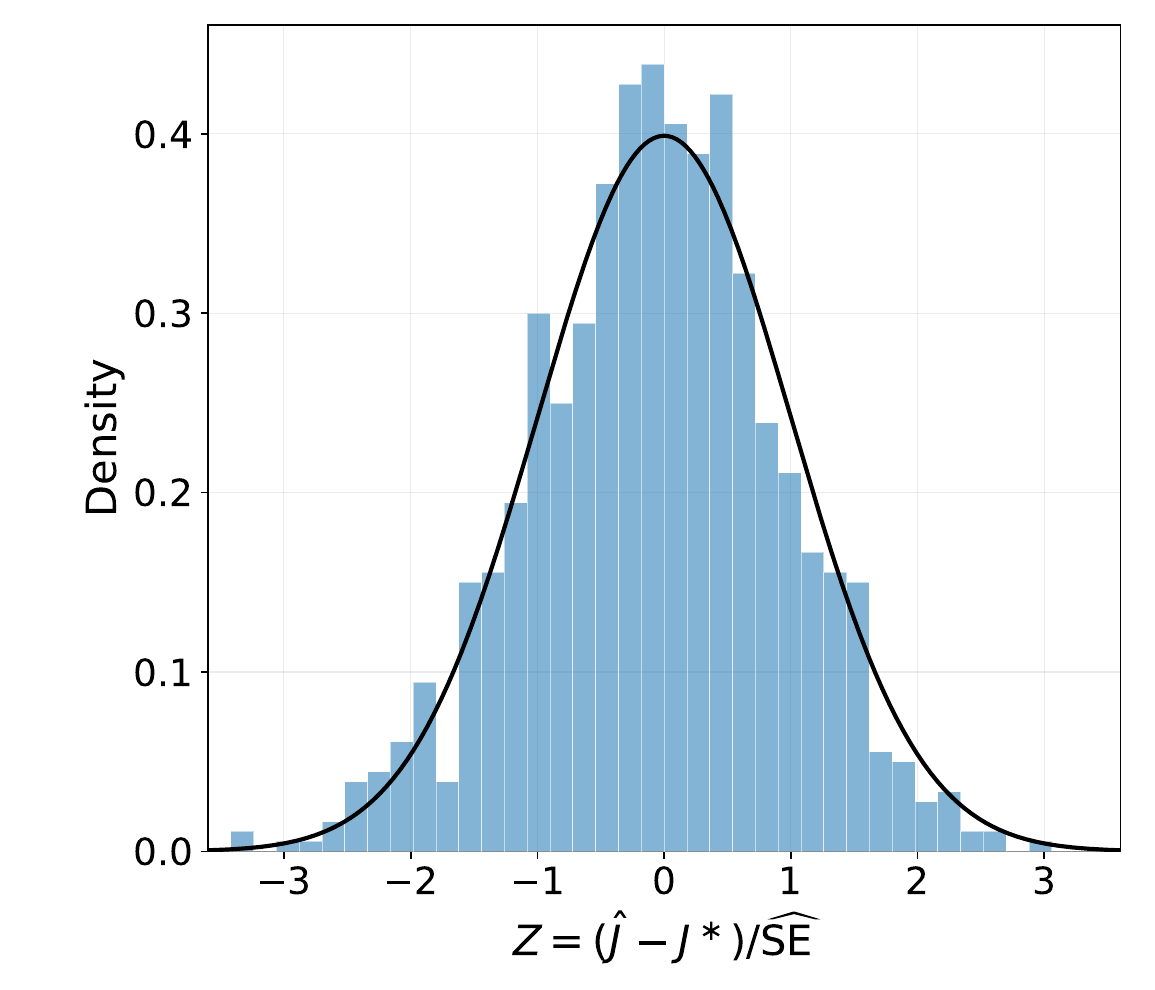}
\end{minipage} &
\begin{minipage}{0.31\textwidth}
\centering
\small{(b) Linear mixture}\\[-1pt]
\includegraphics[width=\textwidth]{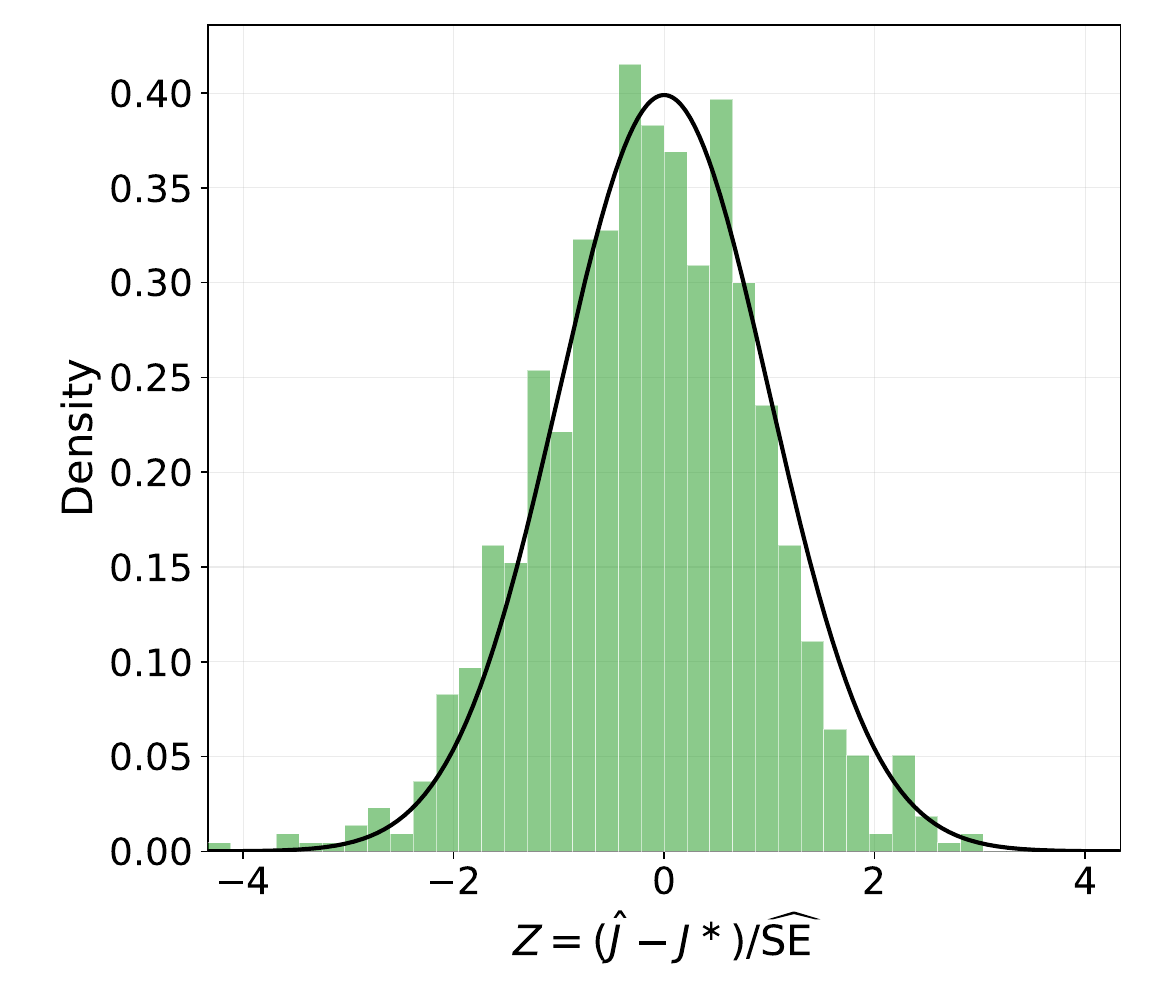}
\end{minipage} &
\begin{minipage}{0.31\textwidth}
\centering
\small{(c) Polynomial mixture}\\[-1pt]
\includegraphics[width=\textwidth]{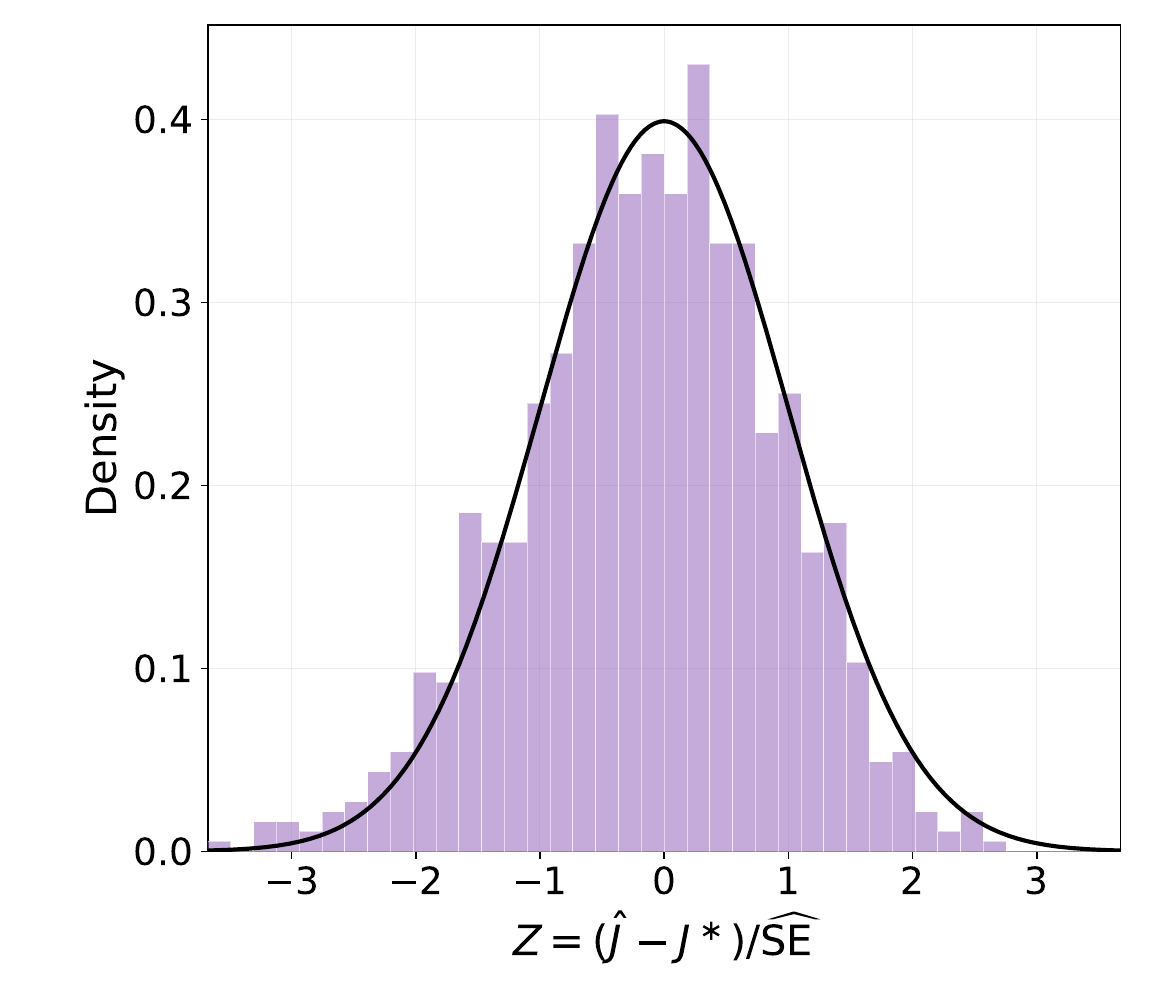}
\end{minipage}
\end{tabular}
\caption{Distribution of the studentized estimation error under three behavior policy regimes.  Writing $\widehat{\mathrm{SE}}$ for the estimated standard error, the bars show the empirical density of $Z\coloneqq(\widehat J_\beta-J^*)/\widehat{\mathrm{SE}}$. The black curve is the standard normal density.}
\label{fig:synthetic-hist}
\end{figure}

\subsection{Bike Repositioning}\label{sec:num-bike}

We use the Multi-Agent Resource Optimization (MARO) bike repositioning simulator \citep{jiang2020maro} with public February 2025 Citi Bike trip records \citep{citibike2025systemdata}. The data contain approximately 2.03 million trips, which we aggregate into hourly demand while retaining all calendar hours in the month. Stations are grouped into three zones. The state is the joint inventory of the zones, with each inventory classified as low, middle, or high, giving $|\cS|=27$ states. The five actions are no repositioning, three specified batch transfers between zones, and a state-dependent watermark rule that moves bikes from the currently fullest zone to the currently emptiest zone. The MARO simulator combines the current inventories, action, and realized demand to produce the next inventories and a reward measuring served demand net of unmet demand and repositioning cost. Here $N$ denotes independent simulated behavior trajectories, each spanning 16 calendar days.

We estimate a reference MDP from 1,000 behavior chains of 3,000 transitions each. Its optimal average reward is $J_{\mathrm{ref}}=0.7996$, which we use as the benchmark in panel~(a). Panel~(b) reports the average reward of each fitted policy evaluated on the same reference MDP. Calibration on this MDP gives $\beta=16$.
Figure~\ref{fig:bike-main} shows that the proposed estimates are stable as $N$ increases, with confidence intervals tightening around point estimates. At $\beta=16$, both the estimated value and the policy value on the reference MDP lie close to the reference model optimum and substantially exceed those of Hardmax-DR and the behavior policy. The operational panels show that the proposed policy moves markedly fewer bikes while also leaving less demand unmet. Thus, the gain in value corresponds to more efficient repositioning rather than simply more intervention.

\begin{figure}[!t]
\centering
\begin{tabular}{@{}cccc@{}}
\begin{minipage}[t]{0.23\textwidth}
\centering
\small{(a) Value estimate}\\[-1pt]
\includegraphics[width=\textwidth]{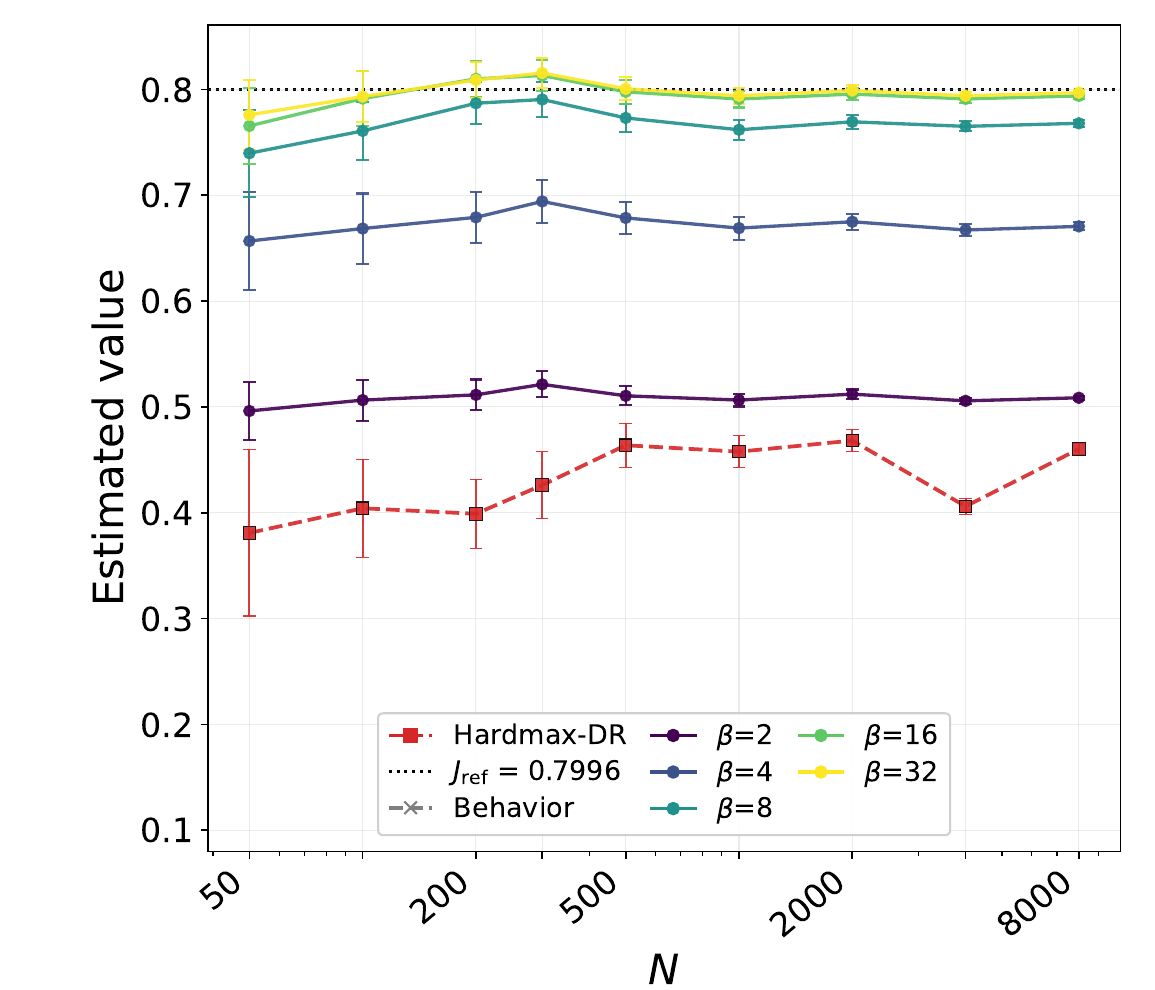}
\end{minipage} &
\begin{minipage}[t]{0.23\textwidth}
\centering
\small{(b) Policy value}\\[-1pt]
\includegraphics[width=\textwidth]{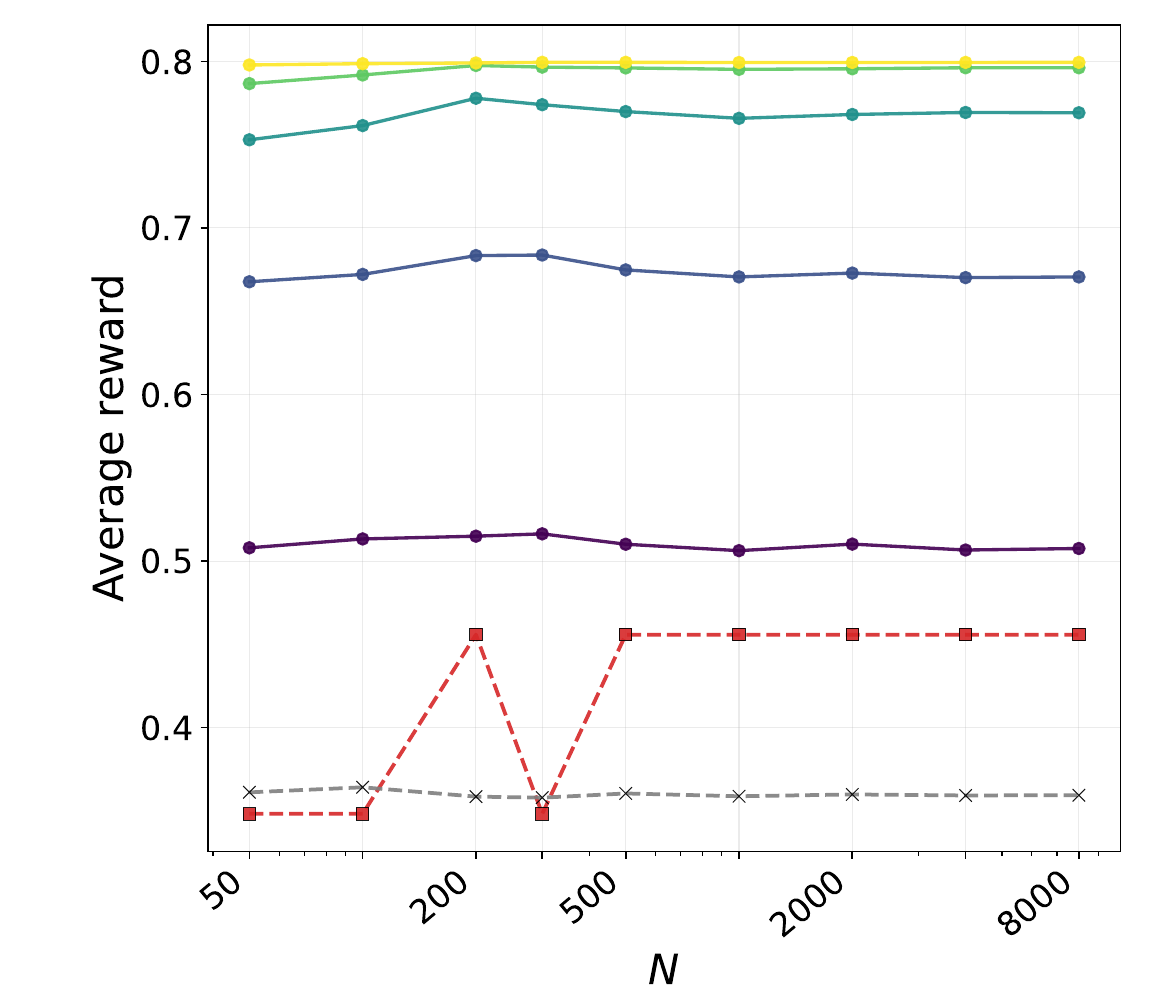}
\end{minipage} &
\begin{minipage}[t]{0.23\textwidth}
\centering
\small{(c) Repositioning}\\[-1pt]
\includegraphics[width=\textwidth]{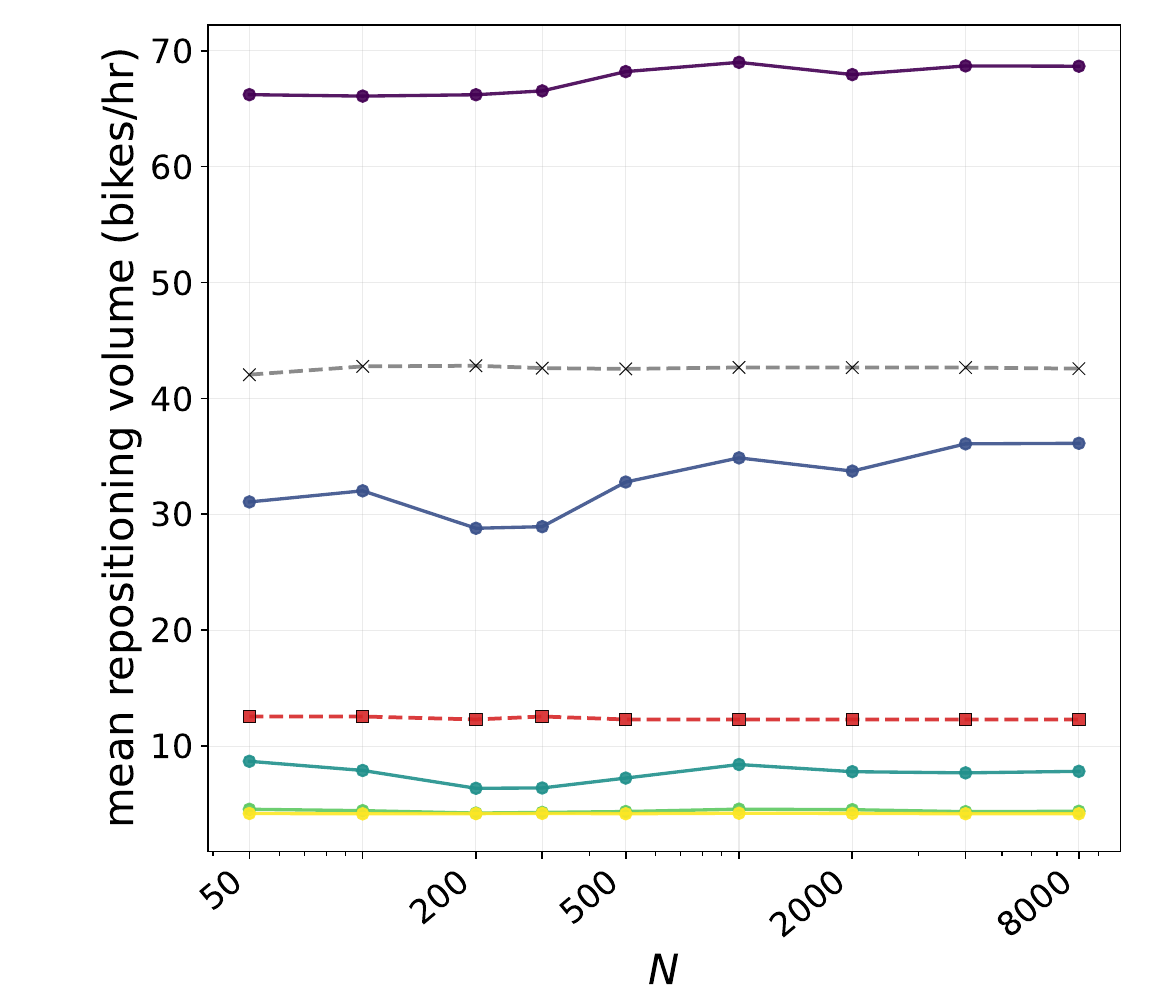}
\end{minipage} &
\begin{minipage}[t]{0.23\textwidth}
\centering
\small{(d) Unmet demand}\\[-1pt]
\includegraphics[width=\textwidth]{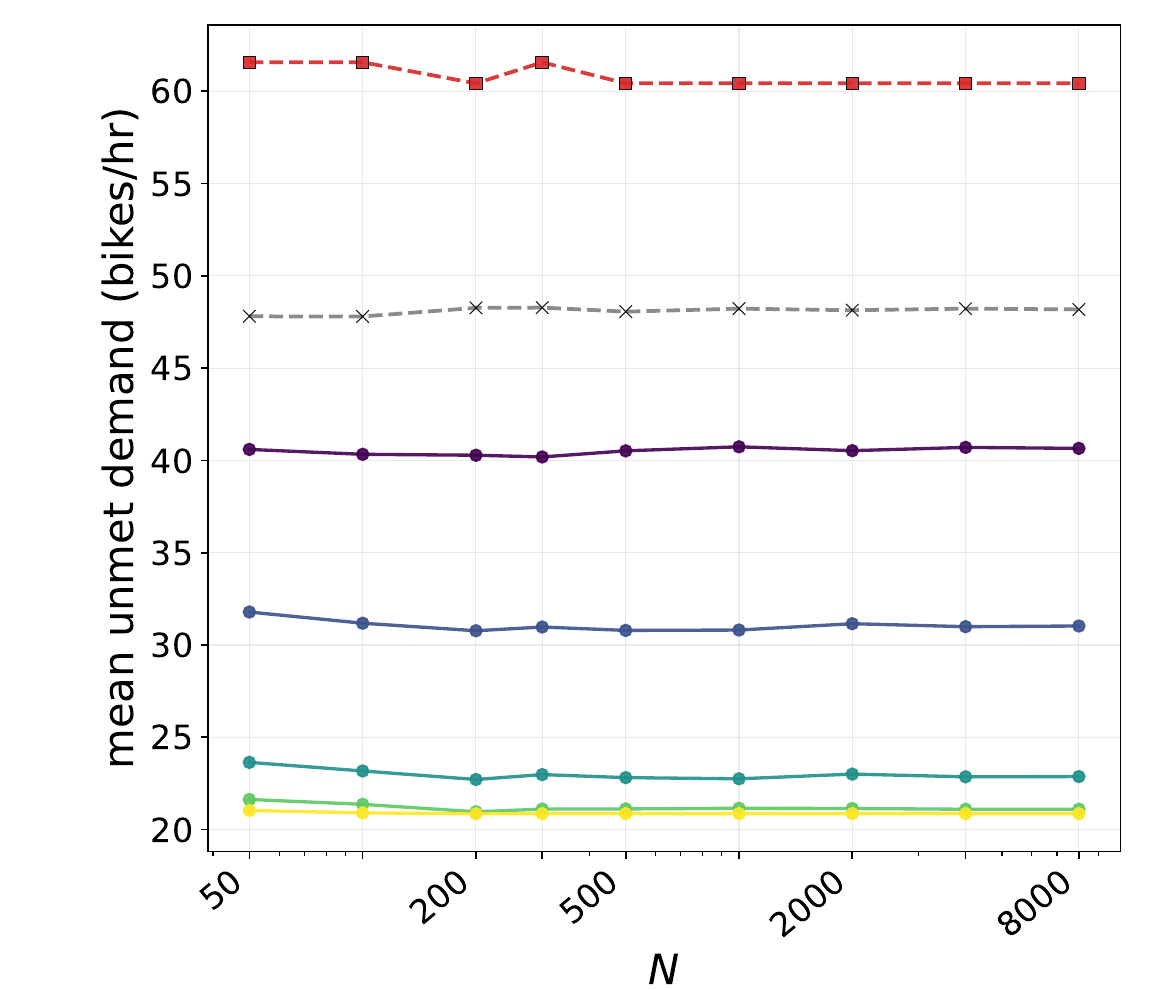}
\end{minipage}
\end{tabular}
\caption{Citi Bike repositioning performance across sample sizes. Panel (a) reports $\widehat J_\beta$ with 95\% Wald intervals and marks the reference value $J_{\mathrm{ref}}=0.7996$ by the black dotted line. Panel (b) reports the average reward of each fitted policy evaluated on the reference MDP. Panels (c) and (d) report mean bikes repositioned and mean unmet demand per hour.}
\label{fig:bike-main}
\end{figure}

\subsection{AI Agentic Tool Use}\label{sec:num-toolsandbox}
We next consider an agentic tool use setting, where a language-model agent makes a sequence of tool choices to complete a user task.
We use the released GPT-4o expert trajectories from ToolSandbox \citep{lu2025toolsandbox} that accompany Talk, Evaluate, Diagnose (TED) \citep{chong2026talk}. The data contain 37 tasks with eight trials each, yielding 296 episodes and 1,047 micro-steps. The state is the product of task domain, progress bucket, and whether a tool has already been called. The five domains are CONTACT, MESSAGING, REMINDER, SETTING, and OTHER. Progress is none, partial, or complete, and \texttt{called\_before} indicates whether a tool has already been called earlier in the episode, giving $|\cS|=30$ possible states. The resulting action space contains $|\cA|=36$ discrete actions. A transition corresponds to one micro-step, and its reward is the increment in benchmark progress, credited on the final micro-step of a round and set to zero for intra-round calls. Let $L^{\min}$ denote the minimum number of micro-steps required for task completion according to the benchmark. We stratify tasks into $L^{\min}=1$ and $L^{\min}\geq2$. To account for within-task dependence, tasks are the units for two-fold cross-fitting and clustered standard errors. In the empirical MDP constructed from all 1,047 micro-steps, the greedy policy has average reward 0.8002. Calibration on the empirical MDP gives $\beta=0.5$. For the two plotted strata, the corresponding empirical reference values are $0.9825$ for $L^{\min}=1$ and $0.4386$ for $L^{\min}\geq2$.

For the $L^{\min}=1$ tasks in Figure~\ref{fig:toolsandbox-one}, Proposed gives value estimates that remain close to the empirical reference across task subsets and smoothing levels, while Hardmax-DR stays substantially below the benchmark. The diagnostic panels show a corresponding difference in nuisance stability. Hardmax-DR produces fitted action values that can grow by several orders of magnitude and is associated with much larger Bellman residuals, whereas the proposed fits remain comparatively well behaved.
The contrast is sharper for the more difficult $L^{\min}\ge2$ tasks in Figure~\ref{fig:toolsandbox-multiple}. Uncertainty is naturally larger and the Proposed fits exhibit larger adjoint weights, reflecting the greater difficulty of the underlying evaluation problem. Nevertheless, the proposed value estimates remain on the scale of the empirical reference and their Bellman residuals remain relatively small. Hardmax-DR, by contrast, can produce extreme value estimates together with fitted $Q$-functions reaching $10^4$--$10^5$ in magnitude. These diagnostics indicate that the instability of Hardmax-DR arises already in the fitted greedy policy nuisance and becomes more severe on the harder multi-step tasks.

\begin{figure}[!t]
\centering
\begin{tabular}{@{}cccc@{}}
\begin{minipage}{0.23\textwidth}
\centering
\small{(a) Value estimate}\\[-1pt]
\includegraphics[width=\textwidth]{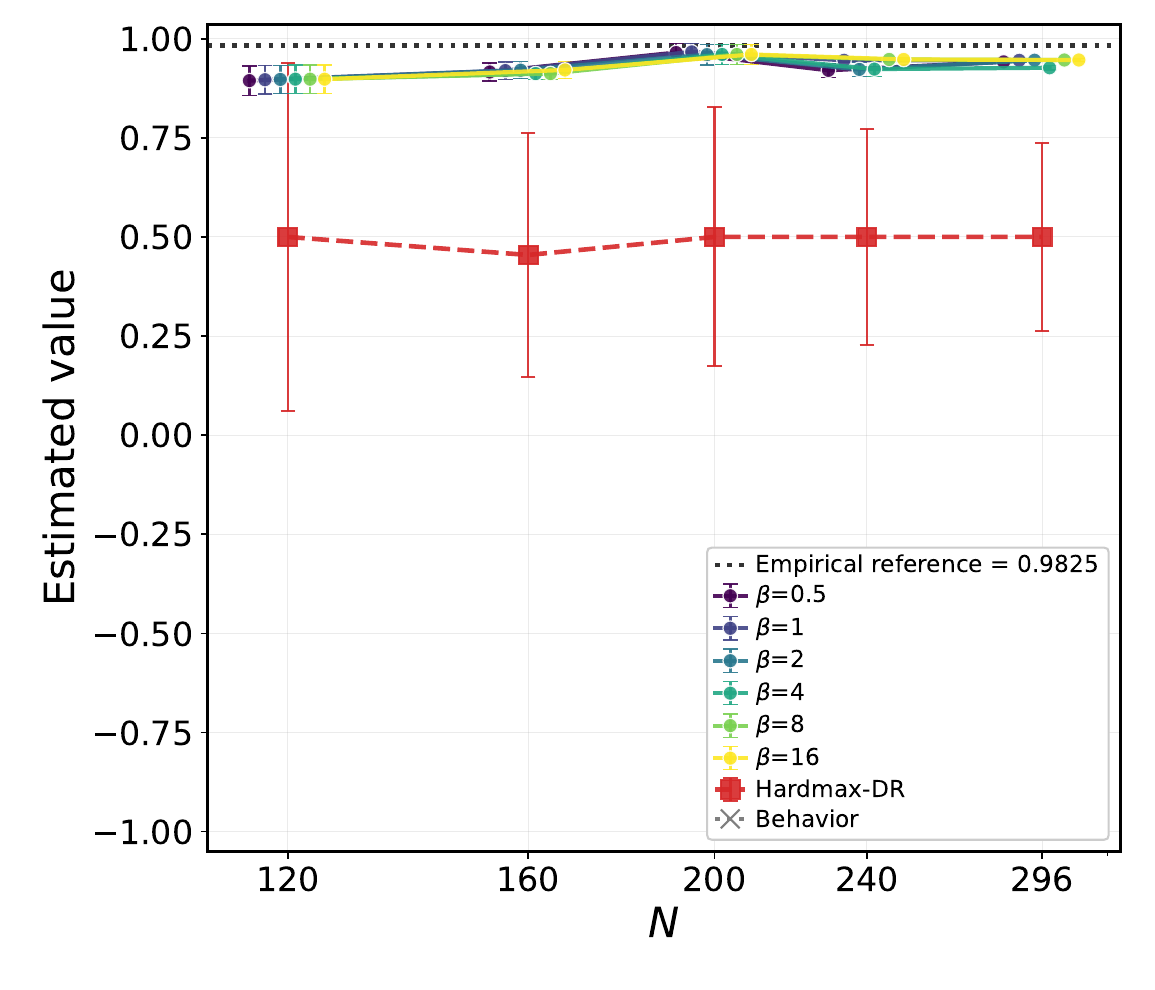}
\end{minipage} &
\begin{minipage}{0.23\textwidth}
\centering
\small{(b) Policy entropy}\\[-1pt]
\includegraphics[width=\textwidth]{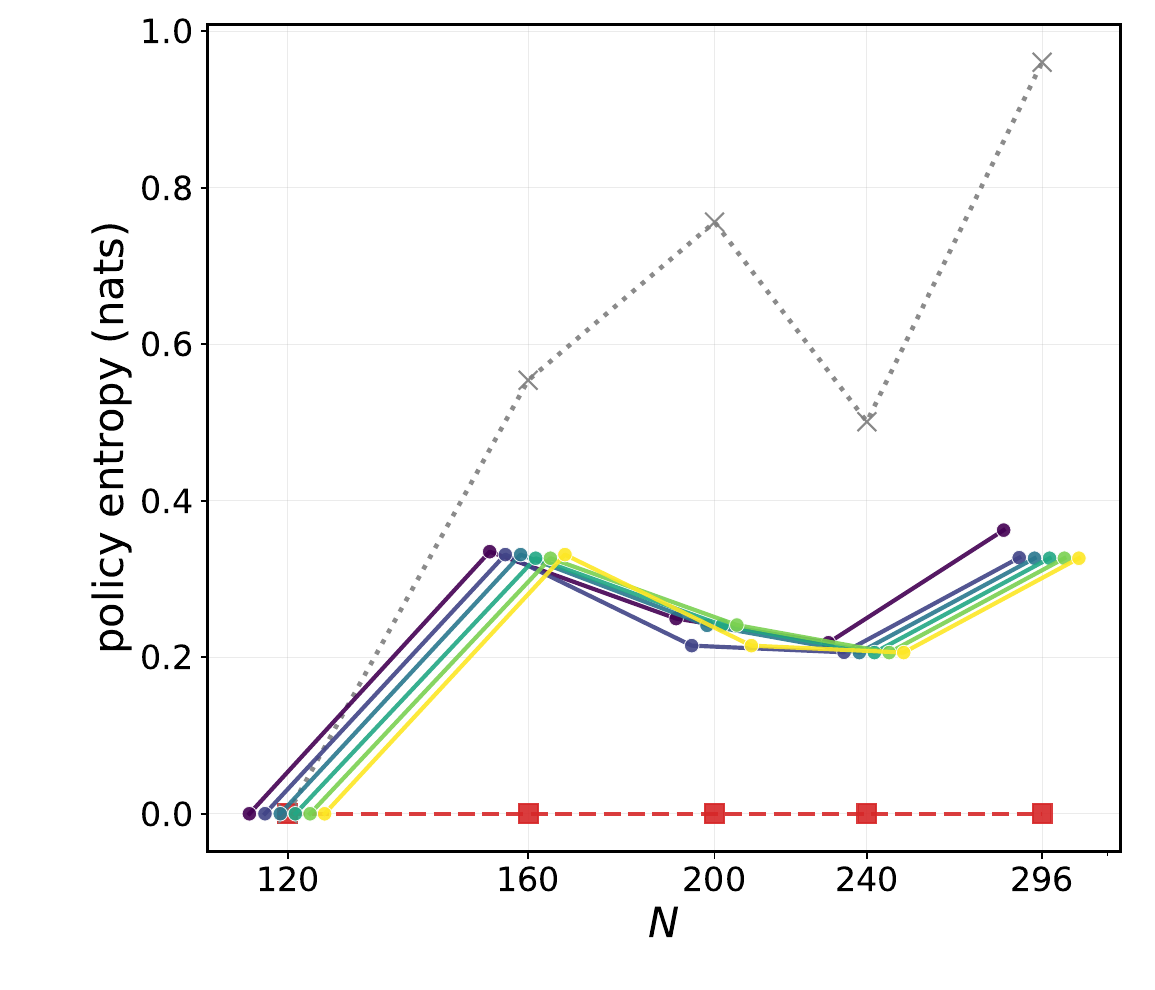}
\end{minipage} &
\begin{minipage}{0.23\textwidth}
\centering
\small{(c) Worst fitted $Q$}\\[-1pt]
\includegraphics[width=\textwidth]{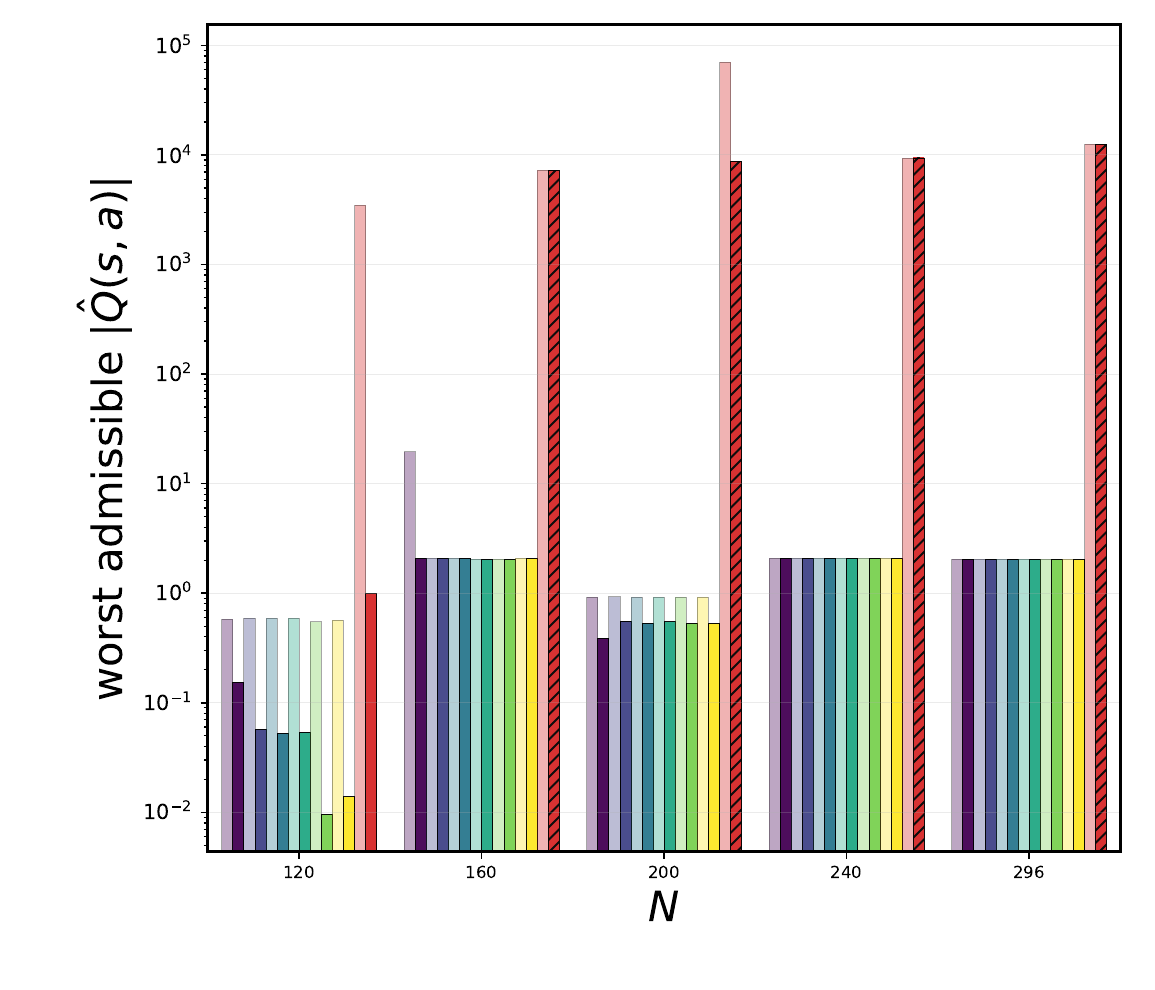}
\end{minipage} &
\begin{minipage}{0.23\textwidth}
\centering
\small{(d) Residual and weight}\\[-1pt]
\includegraphics[width=\textwidth]{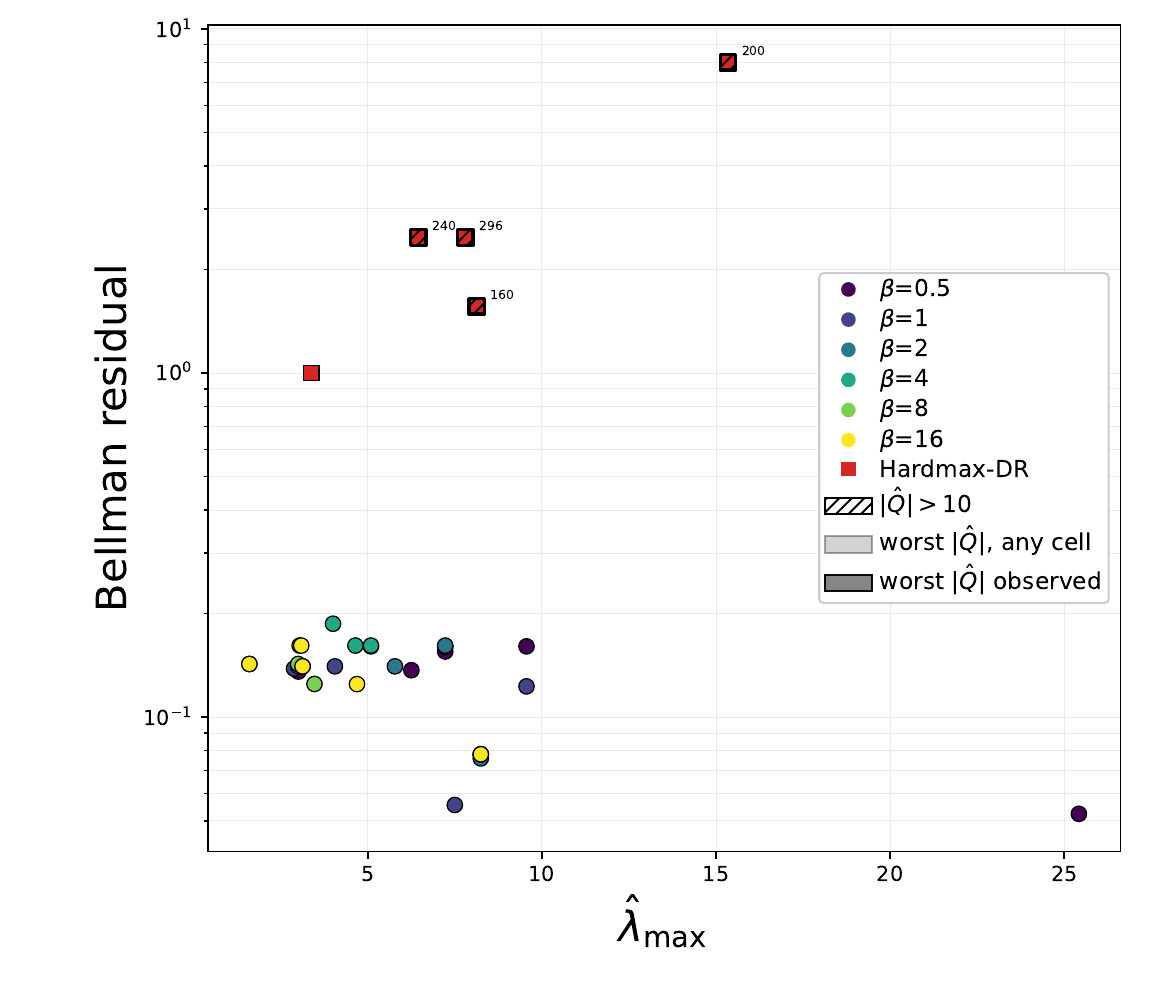}
\end{minipage}
\end{tabular}
\caption{ToolSandbox results for tasks with $L^{\min}=1$. Panel (a) reports value estimates with 95\% Wald intervals; the black dotted line is the in-sample empirical reference value $0.9825$. Panel (b) reports policy entropy. Panel (c) reports the largest absolute fitted $Q$-value over all admissible cells and over cells observed in the evaluation data. Panel (d) compares the Bellman residual with the maximum adjoint weight across fitted models. Hatching indicates fits with $\max|\widehat Q|>10$, and enlarged markers indicate estimates with $|\widehat J|>1$.}
\label{fig:toolsandbox-one}
\end{figure}

\begin{figure}[!htb]
\centering
\begin{tabular}{@{}cccc@{}}
\begin{minipage}{0.23\textwidth}
\centering
\small{(a) Value estimate}\\[-1pt]
\includegraphics[width=\textwidth]{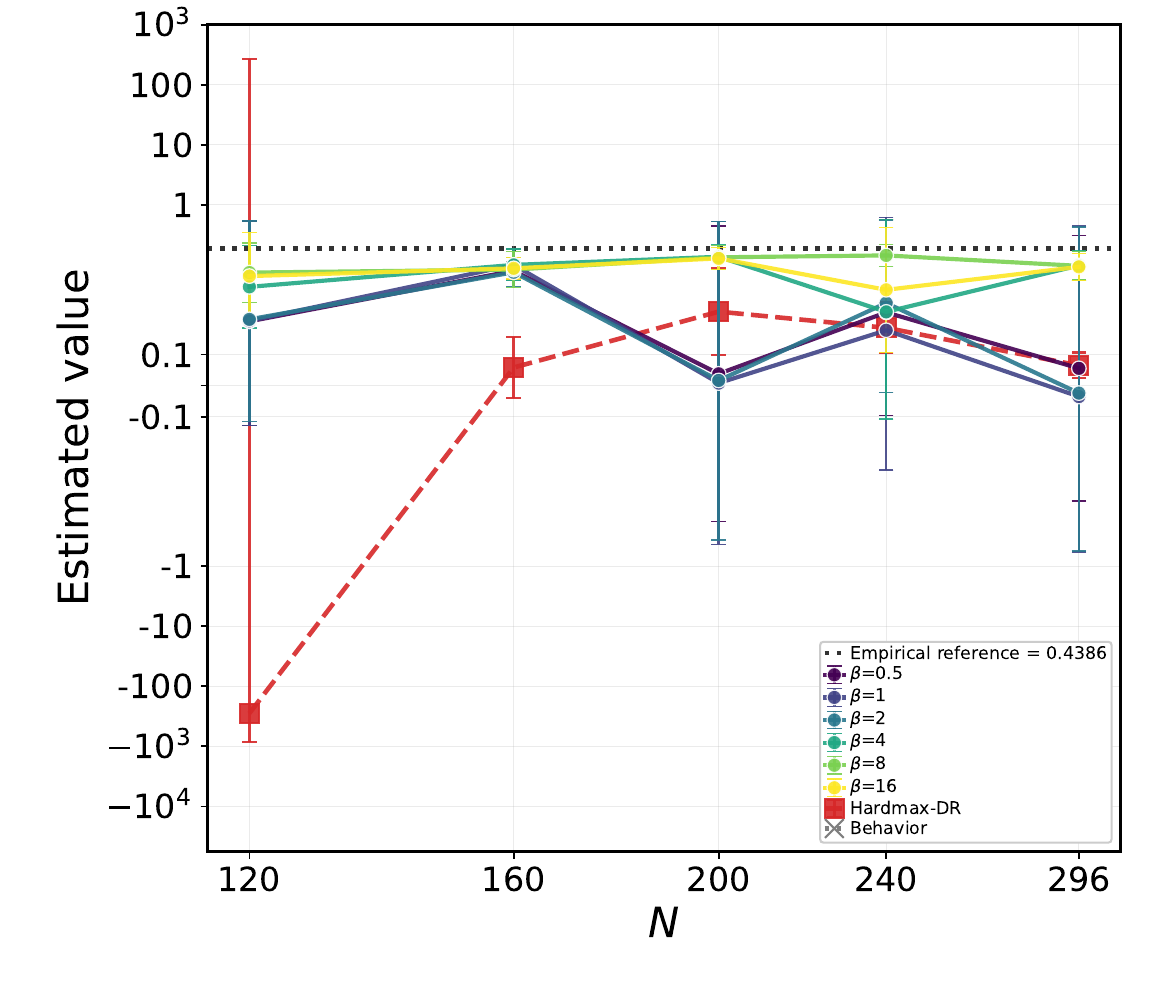}
\end{minipage} &
\begin{minipage}{0.23\textwidth}
\centering
\small{(b) Policy entropy}\\[-1pt]
\includegraphics[width=\textwidth]{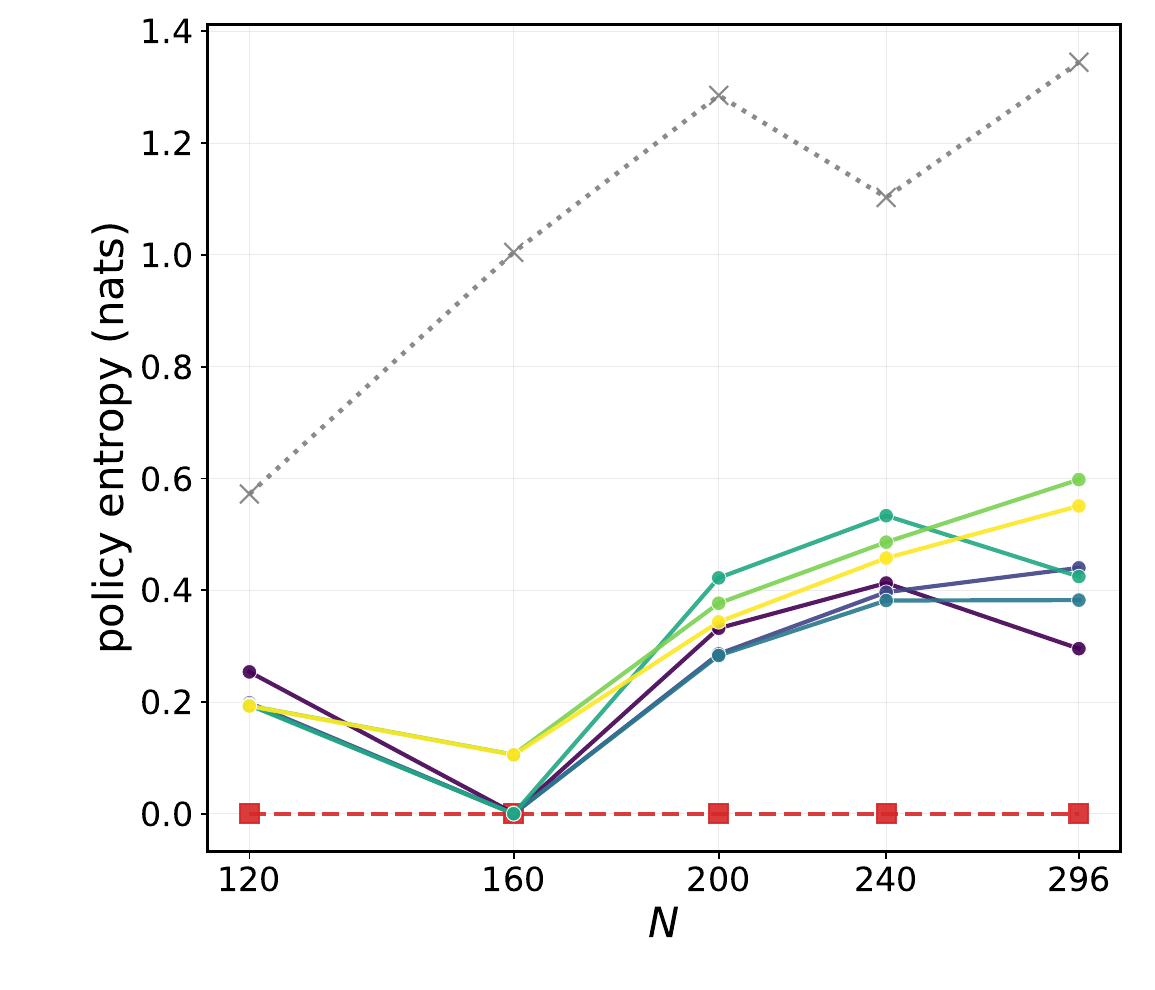}
\end{minipage} &
\begin{minipage}{0.23\textwidth}
\centering
\small{(c) Worst fitted $Q$}\\[-1pt]
\includegraphics[width=\textwidth]{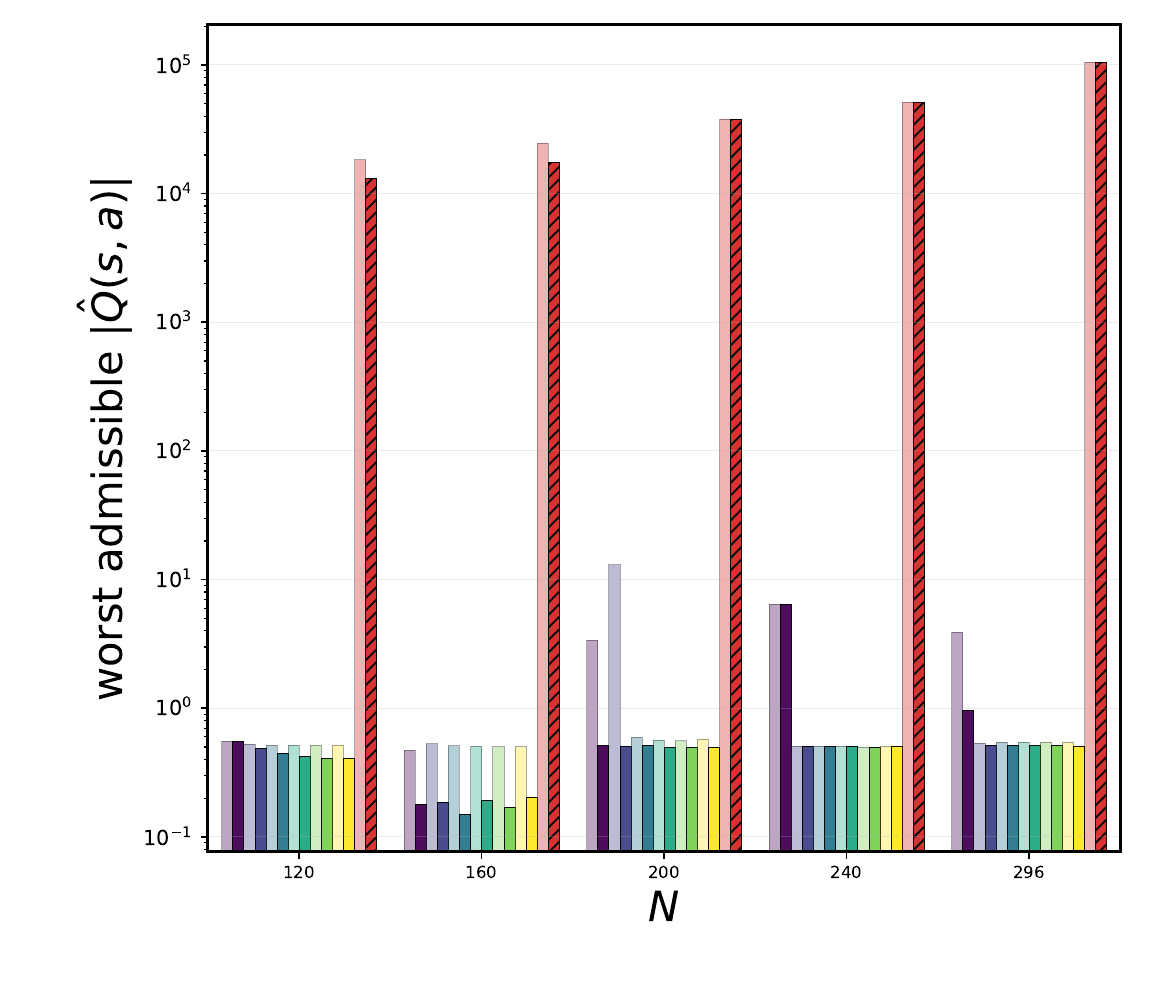}
\end{minipage} &
\begin{minipage}{0.23\textwidth}
\centering
\small{(d) Residual and weight}\\[-1pt]
\includegraphics[width=\textwidth]{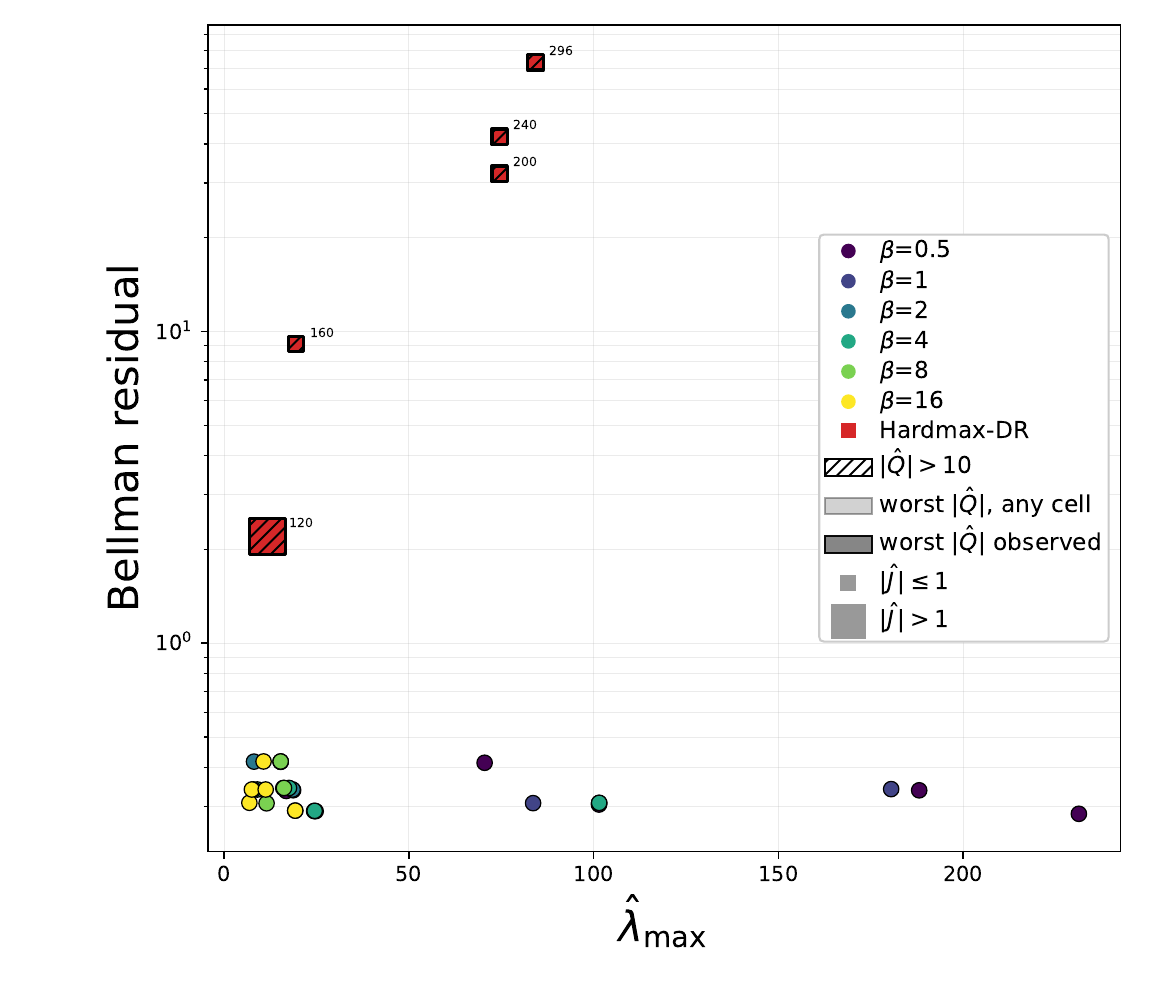}
\end{minipage}
\end{tabular}
\caption{ToolSandbox results for tasks with $L^{\min}\geq2$. Panel (a) reports value estimates with 95\% Wald intervals on a symmetric log scale; the black dotted line is the in-sample empirical reference value $0.4386$. Panel (b) reports policy entropy. Panel (c) reports the largest absolute fitted $Q$-value over all admissible cells and over cells observed in the evaluation data. Panel (d) compares the Bellman residual with the maximum adjoint weight across fitted models. Hatching indicates fits with $\max|\widehat Q|>10$, and enlarged markers indicate estimates with $|\widehat J|>1$.}
\label{fig:toolsandbox-multiple}
\end{figure}

\section{Discussion}\label{sec:discussion}

Our results show that inference for the optimal value of a decision system can be built directly around its optimized Bellman structure. Maximization is statistically irregular when optimal actions are tied, and in a Markov decision process this irregularity propagates recursively through future state distributions and continuation values. We address this difficulty through a self-induced Bellman equation in which policy generation and continuation value evaluation are closed within a single smooth population fixed point.
This self-consistency turns the recursively coupled optimization problem into a smooth target for inference in which perturbations of the continuation values and the induced policy can be handled jointly. The adjoint weight then yields an orthogonal score that removes the combined first-order effect, while the approximation analysis controls the remaining smoothing bias $J^*-J_\beta$. The same structure also leads to estimable Bellman and adjoint weight nuisances, for which we establish estimation rates sufficient for the inference procedure. 
We provide a unified route to optimal value inference under both fixed and diverging horizons.

Several extensions are particularly promising. The self-induced approach could be extended to other smooth approximations of the maximum, continuous or high-dimensional state-action spaces, learned reward models and  discounted objectives. On the applied side, the framework can be useful in pricing, recommendation, and agentic systems, where domain structure such as monotonicity, one-sided feedback, capacity constraints, or cyclic demand could be exploited to improve estimation. A further direction is to connect offline optimal value inference with online deployment. Historical data can first quantify the scope for improvement, after which the policy can be updated as new operational data become available and decisions are revised in light of accumulating information.

\bibliographystyle{plainnat}
\bibliography{ref}

\appendix
\newpage
\begin{center}
\textit{\large Supplementary Material to ``Optimal Value Inference for Reinforcement Learning" }
\end{center}

\section{Notation Summary}\label{app:notation-summary}
\vspace{-20pt}
\small
\renewcommand{\arraystretch}{1.15}
\begin{longtable}{@{}p{0.22\textwidth}p{0.70\textwidth}@{}}
\hline
Symbol & Meaning \\
\hline
\endfirsthead
\hline
Symbol & Meaning \\
\hline
\endhead
\hline
\endfoot
$\cS,\cA$ & state and action spaces \\
$P,R^*$ & transition kernel and conditional mean reward \\
$\Pi,\pi^b$ & stationary Markov policies and the behavior policy \\
$\tau_i,Z_t$ & observed trajectory and one-step transition $(S_t,A_t,R_{t+1},S_{t+1})$ \\
$J^\pi$ & long-run average reward under stationary policy $\pi$ \\
$Q^\pi$ & zero-average fixed-policy relative action-value function \\
$J^*,Q^*$ & optimal average reward and the common anchored relative action-value function of optimal stationary policies \\
$\pi_{\beta,Q}$ & softmax policy induced by a generic action-value function $Q$ \\
$V_\beta(Q)$ & softmax-smoothed value map \\
$(Q_\beta,J_\beta)$ & anchored self-induced Bellman solution \\
$\pi_\beta$ & self-induced policy $\pi_{\beta,Q_\beta}$ \\
$\Gamma$ & smallest positive optimality gap of $Q^*$ \\
$\cQ$ & anchored action-value space \\
$\omega_{\beta,Q}$ & signed derivative weights of $V_\beta$ at $Q$ \\
$D_\beta$ & derivative aggregation operator at $Q_\beta$ \\
$\bar P$ & pooled one-step behavior law for trajectory data \\
$\bar\mu$ & state-action marginal of the relevant one-step law \\
$\bar\nu$ & pooled marginal of the next state for trajectory data \\
$\varphi_\beta(Z;J,Q,\lambda)$ & orthogonal score for $J_\beta$\\
$\lambda_\beta$ & adjoint weight in the orthogonal score  \\
$\xi_{i,t}$ & martingale score increment $\varphi_\beta(Z_{i,t})-J_\beta$ \\
$b_k$ & fold-specific nuisance remainder term \\
$\sigma^2$ &  variance of the score increment \\
\hline
\end{longtable}

\vspace{-20pt}
\section{Proof for Section~\ref{sec:formu}}\label{app:setup-proofs}
\vspace{-8pt}
\begin{lemma}[Agreement across optimal policies and Bellman characterization]\label{lem:optimal-relative-value-uniqueness}
Under Assumption~\ref{ass:irreducible}, there exists a unique $Q^*\in\cQ$ such that $Q^\pi=Q^*$ for every $\pi\in\Pi$ with $J^\pi=J^*$. Moreover, this relative action-value function common to optimal policies satisfies the Bellman optimality equation \eqref{eq:opt-bellman}.
\end{lemma}

\begin{proof}
Assumption~\ref{ass:irreducible} implies that every stationary policy induces a unichain. By \citet[Theorem~1]{mahadevan1996average}, there exist a scalar $g$ and a function $h\colon\cS\to\RR$ satisfying
\[
g+h(s)
=
\max_{a\in\cA}
\left(
R^*(s,a)+(Ph)(s,a)
\right),
\qquad s\in\cS.
\]
The same theorem shows that a stationary greedy policy attains $g$, so the definition of $J^*$ gives $g=J^*$. Define, only for this proof, $q_h(s,a)\coloneqq R^*(s,a)-J^*+(Ph)(s,a)$. The optimality equation implies $h(s)=\max_{a\in\cA}q_h(s,a)$, so $g_h(s,a)\coloneqq h(s)-q_h(s,a)\ge0$.

Fix any $\pi\in\Pi$ with $J^\pi=J^*$, and let $d^\pi$ be the unique stationary distribution of $P^\pi$. Irreducibility gives $d^\pi(s)>0$ for every state. Stationarity and the stationary reward representation of $J^\pi$ yield
\begin{align*}
\sum_{s\in\cS}d^\pi(s)\sum_{a\in\cA}\pi(a\mid s)g_h(s,a)
&= \sum_{s\in\cS}d^\pi(s)h(s)
-\sum_{s,a}d^\pi(s)\pi(a\mid s)R^*(s,a)
+J^*
-\sum_{s,a}d^\pi(s)\pi(a\mid s)(Ph)(s,a)\\
&=J^*-J^\pi
=0.
\end{align*}
Every summand on the left is nonnegative. Since $d^\pi$ has full support, $\pi(a\mid s)>0$ therefore implies $g_h(s,a)=0$, or equivalently $q_h(s,a)=h(s)=\max_b q_h(s,b)$. Consequently, $\sum_{a\in\cA}\pi(a\mid s)q_h(s,a)=h(s)$ for every $s\in\cS$, and the definition of $q_h$ shows that $(J^*,q_h)$ satisfies the fixed policy equation \eqref{eq:policy-evaluation} for every optimal $\pi$. By the characterization for fixed policies in Section~\ref{sec:formu}, $q_h$ differs from each $Q^\pi$ only by an additive constant. Because every $Q^\pi$ obeys the same anchoring constraint, all optimal policies therefore have the common relative action-value function
\[
Q^*
\coloneqq
q_h-\frac{1}{|\cS||\cA|}\sum_{s\in\cS}\sum_{a\in\cA}q_h(s,a)
\in\cQ,
\]
which also proves uniqueness of the common anchored representative.

Finally, $h(s)=\max_a q_h(s,a)$ and the definition of $q_h$ give $J^*+q_h(s,a)=R^*(s,a)+\sum_{s'\in\cS}P(s'\mid s,a)\max_{b\in\cA}q_h(s',b)$.
Subtracting the state-action average of $q_h$ from both sides preserves this identity and yields \eqref{eq:opt-bellman} for $Q^*$.
\end{proof}

\section{Proofs for Section~\ref{sec:est}}\label{app:self-induced-proofs}

\subsection{Proof of Lemma~\ref{lem:shift-invariance}}
\begin{proof}
Define $\widetilde Q\coloneqq Q_\beta+c\one$. The softmax normalization cancels the common shift, and the induced value shifts by the same constant.
\[
\pi_{\beta,\widetilde Q}(a\mid s)
=
\frac{\exp(\beta(Q_\beta(s,a)+c))}{\sum_{b\in\cA}\exp(\beta(Q_\beta(s,b)+c))}
=
\pi_\beta(a\mid s).
\]
The induced value identity is
\[
V_\beta(\widetilde Q)(s)
=
\sum_{a\in\cA}\pi_{\beta,\widetilde Q}(a\mid s)(Q_\beta(s,a)+c)
=
V_\beta(Q_\beta)(s)+c.
\]
Substituting the value-shift identity into the self-induced Bellman equation gives
\begin{align*}
R^*(s,a)+\EE\!\left[V_\beta(\widetilde Q)(S')\mid S=s,A=a\right]
&= R^*(s,a)+\EE\!\left[V_\beta(Q_\beta)(S')+c\mid S=s,A=a\right] \\
&= J_\beta+Q_\beta(s,a)+c
= J_\beta+\widetilde Q(s,a).
\end{align*}
Thus $(\widetilde Q,J_\beta)$ also satisfies \eqref{eq:self}.
\end{proof}

\subsection{Proof of Theorem~\ref{thm:self-induced-existence}}
\begin{proof}
Recall that $\Pi$ denotes the set of stationary policies $\pi\colon\cS\to\cP(\cA)$. Since $\cS$ and $\cA$ are finite, $\Pi$ is a compact subset of a finite-dimensional Euclidean space.
We first show that Assumption~\ref{ass:irreducible} implies the unichain condition.
Fix any stationary policy $\pi\in\Pi$.
By Assumption~\ref{ass:irreducible}, the transition matrix $P^\pi(s'\mid s)=\sum_{a\in\cA}\pi(a\mid s)P(s'\mid s,a)$ is irreducible.
An irreducible finite Markov chain has a single communicating class, which is necessarily recurrent, so there are no transient states and the chain is unichain.
In particular, the average reward $J^\pi$ exists and is independent of the initial state.

Therefore each $\pi\in\Pi$ has a unique scalar $J^\pi$, and the probabilistically defined $Q^\pi$ is the unique anchored solution in $\cQ$ to \eqref{eq:policy-evaluation}. The evaluation map $\pi\mapsto Q^\pi$ is continuous on $\Pi$, because the anchored Poisson equation together with the zero-average constraint is a finite-dimensional linear system in $(Q^\pi,J^\pi)$, Assumption~\ref{ass:irreducible} gives a unique solution for every $\pi$, its coefficient matrix depends continuously on $\pi$, and matrix inversion is continuous on the set of nonsingular matrices. Therefore, $\mathcal Q_\Pi\coloneqq\{Q^\pi:\pi\in\Pi\}\subseteq\cQ$ is compact and bounded. Choose $M<\infty$ such that $\|Q^\pi\|_\infty\le M$ for all $\pi\in\Pi$, and define $B\coloneqq\{Q\in\cQ:\|Q\|_\infty\le M\}$, which is nonempty, compact, and convex, together with $\Phi\colon\cQ\to\cQ$ given by $\Phi(Q)\coloneqq Q^{\pi_{\beta,Q}}$.

For $Q\in B$, the softmax policy $\pi_{\beta,Q}$ belongs to $\Pi$, so $\Phi(Q)=Q^{\pi_{\beta,Q}}\in\mathcal Q_\Pi\subseteq B$. Thus $\Phi$ maps $B$ into itself. The softmax map $Q\mapsto\pi_{\beta,Q}$ is continuous in finite dimension, and the evaluation map $\pi\mapsto Q^\pi$ is continuous. Hence $\Phi$ is continuous on $B$.

Brouwer's fixed-point theorem gives $Q_\beta\in B$ such that $ \Phi(Q_\beta)=Q_\beta. $
Recall that we define $\pi_\beta\coloneqq\pi_{\beta,Q_\beta}$. Equivalently, we have $ Q_\beta=Q^{\pi_\beta}.$
Define $ J_\beta\coloneqq J^{\pi_\beta}.$
Since $Q_\beta=Q^{\pi_\beta}$ is the anchored evaluation of $\pi_\beta$, equation \eqref{eq:policy-evaluation} yields
\begin{align*}
J_\beta+Q_\beta(s,a)
&= R^*(s,a)+\sum_{s'\in\cS}P(s'\mid s,a)\sum_{a'\in\cA}\pi_\beta(a'\mid s')Q_\beta(s',a') \\
&= R^*(s,a)+\EE\!\left[\sum_{a'\in\cA}\pi_\beta(a'\mid S')Q_\beta(S',a')\,\Bigm|\, S=s,A=a\right] \\
&= R^*(s,a)+\EE\!\left[V_\beta(Q_\beta)(S')\mid S=s,A=a\right].
\end{align*}
Therefore $(Q_\beta,J_\beta)$ satisfies \eqref{eq:self}.
\end{proof}

\subsection{Proof of Proposition~\ref{prop:self-induced-uniqueness}}

\subsubsection{Uniform Bound on Signed Weights}
\begin{lemma}\label{lem:if-omega-tv-bound}
There exists a constant $C_{\cA}<\infty$, depending only on $|\cA|$, such that, for every $Q$, $s\in\cS$, and $\beta>0$,
\[
\sum_{a\in\cA}\omega_{\beta,Q}(a\mid s)=1,
\qquad
\sum_{a\in\cA}|\omega_{\beta,Q}(a\mid s)|\le C_{\cA}.
\]
\end{lemma}

\begin{proof}
Fix $s\in\cS$ and $Q$. Write $q_a\coloneqq Q(s,a)$ and $p_a\coloneqq\pi_{\beta,Q}(a\mid s)$, $a\in\cA$. Then $V_\beta(Q)(s)=\sum_{b\in\cA}p_bq_b$, and $\omega_{\beta,Q}(a\mid s)=p_a(1+\beta(q_a-V_\beta(Q)(s)))$.
The row-sum identity follows immediately.
\[
\begin{aligned}
\sum_{a\in\cA}\omega_{\beta,Q}(a\mid s)
&= \sum_{a\in\cA}p_a+\beta\sum_{a\in\cA}p_a\left(q_a-\sum_{b\in\cA}p_bq_b\right)=1.
\end{aligned}
\]

It remains to prove the signed envelope. Let $z_a\coloneqq\beta q_a$. Since $p_a=e^{z_a}/\sum_{b\in\cA}e^{z_b}$, there exists a constant $c$ such that $z_a=\log p_a+c$. The common additive constant cancels from the centered expression.
\[
\begin{aligned}
\omega_{\beta,Q}(a\mid s)
&= p_a\left(1+z_a-\sum_{b\in\cA}p_bz_b\right)
= p_a\left(1+\log p_a-\sum_{b\in\cA}p_b\log p_b\right).
\end{aligned}
\]
The entropy bound then controls the signed mass.
\[
\begin{aligned}
\sum_{a\in\cA}|\omega_{\beta,Q}(a\mid s)|
&\le \sum_{a\in\cA}p_a
+
\sum_{a\in\cA}p_a|\log p_a|
+
\left|
\sum_{b\in\cA}p_b\log p_b
\right|
\sum_{a\in\cA}p_a .
\end{aligned}
\]
Since $p_a\in(0,1]$, we have $|\log p_a|=-\log p_a$, and the two entropy terms satisfy
\[
\sum_{a\in\cA}p_a|\log p_a|
=
-\sum_{a\in\cA}p_a\log p_a
\le
\log|\cA|.
\]
The averaged logarithm satisfies the same entropy bound,
\[
\left|
\sum_{b\in\cA}p_b\log p_b
\right|
=
-\sum_{b\in\cA}p_b\log p_b
\le
\log|\cA|.
\]
Thus $\sum_{a\in\cA}|\omega_{\beta,Q}(a\mid s)|\le1+2\log|\cA|$, and taking $C_{\cA}\coloneqq1+2\log|\cA|$ proves the claim. The bound depends only on $|\cA|$, and is uniform over $Q$, $s$, and $\beta$.
\end{proof}

\subsubsection{A Sufficient Condition for Signed Aggregation Identification}
Write $\spn(f)\coloneqq\sup_x f(x)-\inf_x f(x)$ for the span seminorm of a real-valued function $f$ on a finite domain.
\begin{proposition}[Dobrushin overlap implies uniform signed aggregation identification]
\label{prop:rho-sufficient}
If $ C_{\cA}\rho(P)<1, $ then Assumption~\ref{ass:uniform-signed-aggregation-identification} holds.
\end{proposition}

\begin{proof} 
Fix an allowable aggregation operator $D$. Suppose that $-j\one-q+PDq=0$, $q\in\cQ$, $j\in\mathbb R$. Equivalently, $ q=PDq-j\one.$
Since the span seminorm is invariant under additive constants, we have $ \spn(q)=\spn(PDq).$

We first recall the Dobrushin contraction bound $\spn(Pg)\le\rho(P)\spn(g)$, valid for any state function $g$. Indeed, for any two state-action pairs $x=(s,a)$ and $\tilde x=(\tilde s,\tilde a)$, subtracting an arbitrary constant from $g$, taking absolute values, and optimizing over the constant gives
\[
(Pg)(x)-(Pg)(\tilde x)
=
\sum_{s'\in\mathcal S}
(P(s'\mid x)-P(s'\mid \tilde x))g(s').
\]
Taking absolute values gives
\[
|(Pg)(x)-(Pg)(\tilde x)|
\le
\|P(\cdot\mid x)-P(\cdot\mid \tilde x)\|_{\mathrm{TV}}
\spn(g).
\]
Taking the supremum over $x,\tilde x$ gives the stated bound.

Next we bound the span of $Dq$. For any constant $c$, the property that every row of $D$ sums to one gives $Dq(s)-c=\sum_{a\in\mathcal A}w_s(a)(q(s,a)-c)$, and therefore
\[
|Dq(s)-c|
\le
\sum_{a\in\mathcal A}|w_s(a)|\,|q(s,a)-c|
\le
C_{\mathcal A}\|q-c\one\|_\infty.
\]
Taking the supremum over $s$ and then optimizing over $c$, using $\spn(f)=2\inf_c\|f-c\one\|_\infty$, yields $\spn(Dq)\le C_{\mathcal A}\spn(q)$.

Combining the preceding displays,
\[
\spn(q)
=
\spn(PDq)
\le
\rho(P)\spn(Dq)
\le
C_{\mathcal A}\rho(P)\spn(q).
\]
Since $C_{\mathcal A}\rho(P)<1$, we must have $ \spn(q)=0.$
Thus $q$ is constant on $\mathcal S\times\mathcal A$. Because $q\in\cQ$, the anchoring constraint implies $q=0$. Substituting $q=0$ into $-j\one-q+PDq=0$ gives $j=0$. Hence Assumption~\ref{ass:uniform-signed-aggregation-identification} holds.
\end{proof}

\subsubsection{Proof of Proposition~\ref{prop:self-induced-uniqueness}}

\begin{lemma}[Softmax calculus]\label{lem:softmax-calculus}
For $q=(q_a)_{a\in\cA}$, define $ \mathrm{aso}_\beta(q) \coloneqq \frac{\sum_{a\in\cA}q_a e^{\beta q_a}} {\sum_{b\in\cA}e^{\beta q_b}}.$
Its derivative is
\[
\frac{\partial}{\partial q_a}\mathrm{aso}_\beta(q)
=
w_a(q)(1+\beta(q_a-\mathrm{aso}_\beta(q))),
\qquad
w_a(q)\coloneqq\frac{e^{\beta q_a}}{\sum_{b\in\cA}e^{\beta q_b}}.
\]
With $\omega_{\beta,Q}$ defined by \eqref{eq:omega} with $Q_\beta$ replaced by $Q$,
\[
D_QV_\beta(Q)(s)[h]
=
\sum_{a\in\cA}\omega_{\beta,Q}(a\mid s)h(s,a),
\]
where the bracketed argument $h$ is the perturbation direction. Moreover, we have
\[
\sum_{a\in\cA}\omega_{\beta,Q}(a\mid s)=1,
\qquad
\sum_{a\in\cA}|\omega_{\beta,Q}(a\mid s)|\le C_{\cA},
\qquad
\|D_QV_\beta(Q)(s)\|_{\mathrm{op}}\le C,\quad
\text{  
and  }\quad
\|\nabla_Q^2V_\beta(Q)(s)\|_{\mathrm{op}}
\le C\beta.
\]
Consequently, for every increment $u$,
\[
|V_\beta(Q+u)(s)-V_\beta(Q)(s)-D_QV_\beta(Q)(s)[u]|
\le
C\beta\|u(s,\cdot)\|_2^2.
\]
For arbitrary action-value functions $Q_1$ and $Q_2$, set $Q_t\coloneqq Q_2+t(Q_1-Q_2)$, $t\in[0,1]$, and define
\[
(D_{\beta;Q_1,Q_2}f)(s)
\coloneqq
\sum_{a\in\cA}
\left(\int_0^1\omega_{\beta,Q_t}(a\mid s)\,dt\right)f(s,a).
\]
Then $
V_\beta(Q_1)-V_\beta(Q_2)
=
D_{\beta;Q_1,Q_2}(Q_1-Q_2),$
and $D_{\beta;Q_1,Q_2}$ is an allowable state-wise signed action aggregation operator.
The constants depend only on $|\cA|$ and are uniform over $Q$, $s$, and $\beta$.
\end{lemma}

\begin{proof}
Note that  $ \mathrm{aso}_\beta(q)=\sum_{a\in\cA}w_a(q)q_a.$
Since $ \frac{\partial w_a(q)}{\partial q_c} = \beta w_a(q)(\mathbf 1(a=c)-w_c(q)),$ we have
\begin{align*}
\frac{\partial}{\partial q_c}\mathrm{aso}_\beta(q)
&= w_c(q)+\sum_{a\in\cA}q_a\frac{\partial w_a(q)}{\partial q_c}  \\
&= w_c(q)+\beta w_c(q)\left(q_c-\sum_{a\in\cA}w_a(q)q_a\right) 
= w_c(q)(1+\beta(q_c-\mathrm{aso}_\beta(q))).
\end{align*}
Substituting $q_a=Q(s,a)$ gives the displayed derivative of $V_\beta$. The row-sum and signed-envelope bounds follow from Lemma~\ref{lem:if-omega-tv-bound}. In particular, Cauchy--Schwarz and finite-dimensional norm equivalence give $\|D_QV_\beta(Q)(s)\|_{\mathrm{op}}\le C$ uniformly in $Q$, $s$, and $\beta$.

It remains to establish the second derivative bound. Write $z=\beta q$, $p_a=e^{z_a}/\sum_b e^{z_b}$, and $f(z)\coloneqq\sum_{a\in\cA}p_a z_a$. Then $\mathrm{aso}_\beta(q)=\beta^{-1}f(\beta q)$, so $\nabla_q\mathrm{aso}_\beta(q)=\nabla f(z)$ and $\nabla_q^2\mathrm{aso}_\beta(q)=\beta\nabla^2 f(z)$. The derivative calculation above gives $g_a(z)\coloneqq\partial_{z_a}f(z)=p_a(1+z_a-f(z))$. Differentiating with respect to $z_c$ gives
\[
\partial_{z_c}g_a(z)
=
p_a(\mathbf 1(a=c)-p_c)(1+z_a-f(z))
+p_a(\mathbf 1(a=c)-g_c(z)).
\]
Because $z_a=\log p_a+c$ for a common constant $c$, we have $z_a-f(z)=\log p_a-\sum_b p_b\log p_b$. The quantities $p_a|\log p_a|$ and $|\sum_b p_b\log p_b|$ are uniformly bounded over the simplex. Hence $\|\nabla f(z)\|_2\le C$. The first term in the preceding Hessian display is uniformly bounded because
$p_a|z_a-f(z)|\le p_a|\log p_a|+p_a|\sum_b p_b\log p_b|\le C$, and the second is uniformly bounded because $\|\nabla f(z)\|_2\le C$. Since the action space is fixed and finite, $\|\nabla^2 f(z)\|_{\mathrm{op}}\le C$. Therefore $\|\nabla_q^2\mathrm{aso}_\beta(q)\|_{\mathrm{op}}\le C\beta$. Substituting $q=Q(s,\cdot)$ proves the Hessian bound for $V_\beta$, and Taylor's theorem gives the stated remainder bound.

For the exact secant claim, fix a state $s$. The fundamental theorem of calculus and the derivative formula above, with $Q_1-Q_2$ as the perturbation direction, give
\[
\begin{aligned}
V_\beta(Q_1)(s)-V_\beta(Q_2)(s)
&=\int_0^1D_QV_\beta(Q_t)(s)[Q_1-Q_2]\,dt =\sum_{a\in\cA}\left(\int_0^1\omega_{\beta,Q_t}(a\mid s)\,dt\right)(Q_1-Q_2)(s,a),
\end{aligned}
\]
which is the exact secant identity. The row-sum identity above also gives
\[
\sum_{a\in\cA}\int_0^1\omega_{\beta,Q_t}(a\mid s)\,dt=1.
\]
The signed row norm satisfies
\[
\sum_{a\in\cA}\left|\int_0^1\omega_{\beta,Q_t}(a\mid s)\,dt\right|
\le
\int_0^1\sum_{a\in\cA}|\omega_{\beta,Q_t}(a\mid s)|\,dt
\le C_{\cA}.
\]
Thus every row sums to one and has signed $\ell_1$ norm at most $C_{\cA}$, so $D_{\beta;Q_1,Q_2}$ is allowable.
\end{proof}

\begin{proof}[Proof of Proposition~\ref{prop:self-induced-uniqueness}]
Existence of an anchored self-induced solution follows from Theorem~\ref{thm:self-induced-existence}. We first show that this anchored solution is unique.

Suppose that $ (Q_\beta^{(1)},J_\beta^{(1)}), (Q_\beta^{(2)},J_\beta^{(2)}) $ are two anchored solutions of the self-induced Bellman equation.
Set $ q\coloneqq Q_\beta^{(1)}-Q_\beta^{(2)}$ and $ j\coloneqq J_\beta^{(1)}-J_\beta^{(2)}.$
Because both $Q_\beta^{(1)}$ and $Q_\beta^{(2)}$ belong to $\cQ$, we have $q\in\cQ$.
Apply Lemma~\ref{lem:softmax-calculus} with $(Q_1,Q_2)=(Q_\beta^{(1)},Q_\beta^{(2)})$. The lemma gives $V_\beta(Q_\beta^{(1)})-V_\beta(Q_\beta^{(2)})=D_{\beta;Q_\beta^{(1)},Q_\beta^{(2)}}q$ and verifies that $D_{\beta;Q_\beta^{(1)},Q_\beta^{(2)}}$ is allowable.

Subtracting the two self-induced Bellman equations gives
\[
j\one+q
=
P\left(
V_\beta(Q_\beta^{(1)})
-
V_\beta(Q_\beta^{(2)})
\right)
=
PD_{\beta;Q_\beta^{(1)},Q_\beta^{(2)}}q.
\]
Equivalently, we have $ -j\one-q+PD_{\beta;Q_\beta^{(1)},Q_\beta^{(2)}}q=0.$
Since $q\in\cQ$ and $D_{\beta;Q_\beta^{(1)},Q_\beta^{(2)}}$ is allowable, Assumption~\ref{ass:uniform-signed-aggregation-identification} implies $ q=0, j=0.$
Thus $ Q_\beta^{(1)}=Q_\beta^{(2)}, J_\beta^{(1)}=J_\beta^{(2)}.$
Together with Theorem~\ref{thm:self-induced-existence}, this proves existence and uniqueness of the anchored pair $(Q_\beta,J_\beta)$.\end{proof}

\subsection{Uniform Boundedness and Approximation}\label{app:large-beta-limit}

This section first establishes the approximation bound needed for Theorem~\ref{thm:exp}. We also establish convergence results for the continuation value, induced policy, and derivative operator for subsequent use.

\subsubsection{Preliminary Bounds}

\begin{corollary}[Uniform boundedness of the self-induced Bellman solution]\label{cor:app-smoothed-target-bounded}
Under Assumption~\ref{ass:irreducible}, the solutions constructed in Theorem~\ref{thm:self-induced-existence} satisfy
\[
\sup_{\beta>0}\|Q_\beta\|_\infty<\infty,
\qquad
\sup_{\beta>0}|J_\beta|<\infty.
\]
\end{corollary}

\begin{proof}
The proof of Theorem~\ref{thm:self-induced-existence} establishes that $\{Q^\pi:\pi\in\Pi\}$ is compact and hence uniformly bounded in sup norm. Since $Q_\beta=Q^{\pi_\beta}$ for every $\beta>0$, it follows that $\sup_{\beta>0}\|Q_\beta\|_\infty<\infty$.

Likewise, $J_\beta=J^{\pi_\beta}$. Since $\mathcal S\times\mathcal A$ is finite and $R^*$ is finite-valued, $\max_{s\in\mathcal S,a\in\mathcal A}|R^*(s,a)|<\infty$. Therefore $|J^\pi|\le C$ for every stationary policy $\pi$, for some $C<\infty$, and hence $
\sup_{\beta>0}|J_\beta|<\infty.$
\end{proof}

Assumption~\ref{ass:uniform-signed-aggregation-identification} is an injectivity statement for each fixed allowable operator. The argument below upgrades it to a quantitative inverse uniformly over the whole allowable family. The supremum-norm formulation involves no behavior law and is therefore free of the horizon.

\begin{lemma}[Uniform quantitative inverse over allowable aggregation operators]
\label{lem:uniform-sup-inverse}
For an allowable state-wise signed action aggregation operator $D$, define
\[
\cL_D(q,j)
\coloneqq
-j\one-q+PDq,
\qquad (q,j)\in\cQ\times\RR.
\]
Suppose Assumption~\ref{ass:uniform-signed-aggregation-identification} holds. Then there exists $C<\infty$ such that
\[
\|q\|_\infty+|j|
\le
C\|\cL_D(q,j)\|_\infty,
\qquad (q,j)\in\cQ\times\RR,
\]
for every allowable $D$. The constant depends only on $P$ and $|\cA|$, and in particular does not depend on $D$, $\beta$, $n$, $N$, or $H$.
\end{lemma}

\begin{proof}
Parametrize the allowable family by its weight arrays,
\[
\cW
\coloneqq
\left\{
w=(w_s(a))_{s\in\cS,a\in\cA}\in\RR^{\cS\times\cA}:
\sum_{a\in\cA}w_s(a)=1
\text{ and }
\sum_{a\in\cA}|w_s(a)|\le C_{\cA}
\text{ for every }s\in\cS
\right\},
\]
and write $D_w$ for the operator with weights $w$. The set $\cW$ is closed and bounded in finite dimension, hence compact. The set $\cU\coloneqq\{(q,j)\in\cQ\times\RR:\|q\|_\infty+|j|=1\}$ is compact as well.

The map $(w,q,j)\mapsto\|\cL_{D_w}(q,j)\|_\infty$ is continuous on $\cW\times\cU$. By Assumption~\ref{ass:uniform-signed-aggregation-identification}, $\cL_{D_w}(q,j)=0$ with $q\in\cQ$ implies $(q,j)=(0,0)$, so this continuous map is strictly positive on $\cW\times\cU$. It therefore has a positive minimum $c_*>0$. For $(q,j)\ne(0,0)$, applying the bound to $(q,j)/(\|q\|_\infty+|j|)$ and using homogeneity gives the claim with $C=c_*^{-1}$. The case $(q,j)=(0,0)$ is trivial.
\end{proof}

\subsubsection{Approximation to the Optimal Bellman Pair}

Write $V^*(s)\coloneqq\max_{a\in\cA}Q^*(s,a)$ and $\Delta_a^*(s)\coloneqq V^*(s)-Q^*(s,a)$. Thus $\Delta_a^*(s)\ge0$, and strictly suboptimal actions have gap at least $\Gamma$.

\begin{proposition}[Approximation to the optimal Bellman pair]
\label{prop:large-beta-approx-tie}
Suppose Assumptions~\ref{ass:irreducible} and~\ref{ass:uniform-signed-aggregation-identification} hold, and let $(Q_\beta,J_\beta)$ be the anchored self-induced solution of Proposition~\ref{prop:self-induced-uniqueness}. Then there exists $C<\infty$, independent of $\beta$, such that, for every $\beta>0$,
\begin{equation}
\|Q_\beta-Q^*\|_\infty+|J_\beta-J^*|
\le
Ce^{-\beta\Gamma}.
\label{eq:large-beta-target-bound}
\end{equation}
If $\Gamma=+\infty$, then $Q_\beta=Q^*$ and $J_\beta=J^*$ exactly, for every $\beta>0$.
\end{proposition}

\begin{proof}
First evaluate the smoother at $Q^*$. Dividing the numerator and denominator of $\pi_{\beta,Q^*}(a\mid s)$ by $\exp(\beta V^*(s))$ gives
\[
\pi_{\beta,Q^*}(a\mid s)
=
\frac{\exp(-\beta\Delta_a^*(s))}{\sum_{b\in\cA}\exp(-\beta\Delta_b^*(s))}.
\]
Let $\cA^*(s)\coloneqq\{a\in\cA:\Delta_a^*(s)=0\}$. Since the actions in $\cA^*(s)$ contribute one each to the normalizing sum,
\begin{equation*}
V^*(s)-V_\beta(Q^*)(s)
=
\frac{
\sum_{a\notin\cA^*(s)}\Delta_a^*(s)\exp(-\beta\Delta_a^*(s))
}{
|\cA^*(s)|+
\sum_{b\notin\cA^*(s)}\exp(-\beta\Delta_b^*(s))
},
\qquad s\in\cS.
\end{equation*}
If $\Gamma=+\infty$, the numerator is empty and $V_\beta(Q^*)=V^*$. Otherwise, every $a\notin\cA^*(s)$ satisfies $\Delta_a^*(s)\ge\Gamma$, the denominator is at least one, and with $B_\Delta\coloneqq\max_{s,a}\Delta_a^*(s)<\infty$,
\begin{equation}
0
\le
\|V_\beta(Q^*)-V^*\|_\infty
\le
(|\cA|-1)B_\Delta e^{-\beta\Gamma}.
\label{eq:large-beta-residual-bound}
\end{equation}

Set $q_\beta\coloneqq Q_\beta-Q^*$ and $j_\beta\coloneqq J_\beta-J^*$. Subtracting \eqref{eq:opt-bellman} from \eqref{eq:self} gives
\[
j_\beta+q_\beta(s,a)
=
(P(V_\beta(Q_\beta)-V^*))(s,a),
\qquad (s,a)\in\cS\times\cA.
\]
Apply Lemma~\ref{lem:softmax-calculus} with $(Q_1,Q_2)=(Q_\beta,Q^*)$. The lemma gives $V_\beta(Q_\beta)-V_\beta(Q^*)=D_{\beta;Q_\beta,Q^*}q_\beta$ and verifies that $D_{\beta;Q_\beta,Q^*}$ is allowable. Rearranging gives the exact secant equation
\begin{equation}
-j_\beta\one-q_\beta+PD_{\beta;Q_\beta,Q^*}q_\beta
=
P(V^*-V_\beta(Q^*)).
\label{eq:large-beta-difference}
\end{equation}
The left-hand side is $\cL_{D_{\beta;Q_\beta,Q^*}}(q_\beta,j_\beta)$. Lemma~\ref{lem:uniform-sup-inverse} and $\|Pg\|_\infty\le\|g\|_\infty$ therefore yield $\|Q_\beta-Q^*\|_\infty+|J_\beta-J^*|\le C\|V^*-V_\beta(Q^*)\|_\infty$.
Combining this bound with \eqref{eq:large-beta-residual-bound} proves \eqref{eq:large-beta-target-bound}. If $\Gamma=+\infty$, the right-hand side of \eqref{eq:large-beta-difference} is zero, so the uniform inverse gives $Q_\beta=Q^*$ and $J_\beta=J^*$.
\end{proof}

\begin{proof}[Proof of Theorem~\ref{thm:exp}]
Proposition~\ref{prop:large-beta-approx-tie} gives $|J_\beta-J^*|\le Ce^{-\beta\Gamma}$. By Remark~\ref{rem:self-induced-is-policy-evaluation}, $J_\beta=J^{\pi_\beta}$ is the average reward of the stationary policy $\pi_\beta$, while $J^*$ is the optimal average reward over $\Pi$. Hence $J_\beta\le J^*$, and therefore $0\le J^*-J_\beta\le Ce^{-\beta\Gamma}$. If $\Gamma=+\infty$, Proposition~\ref{prop:large-beta-approx-tie} gives $J_\beta=J^*$ for every $\beta>0$.
\end{proof}

\subsubsection{Convergence of the Induced Objects}

Recall that $
\cA^*(s)
\coloneqq
\{a\in\cA:\Delta_a^*(s)=0\},$ and $
\pi^{\mathrm{unif}}(a\mid s)
\coloneqq
\frac{\mathbf 1(a\in\cA^*(s))}{|\cA^*(s)|}.$
The policy $\pi^{\mathrm{unif}}$ is supported on Bellman-optimal actions, and hence
\begin{equation}
\sum_{a\in\cA}\pi^{\mathrm{unif}}(a\mid s)Q^*(s,a)
=
V^*(s),
\qquad s\in\cS.
\label{eq:unif-attains-max}
\end{equation}
Substituting \eqref{eq:unif-attains-max} into \eqref{eq:opt-bellman} shows that $(Q^*,J^*)$ is the anchored policy-evaluation pair for $\pi^{\mathrm{unif}}$, so $Q^{\pi^{\mathrm{unif}}}=Q^*$ and $J^{\pi^{\mathrm{unif}}}=J^*$.

\begin{proposition}[Convergence of the continuation value and induced policy]
\label{prop:large-beta-policy}
Suppose Assumptions~\ref{ass:irreducible} and~\ref{ass:uniform-signed-aggregation-identification} hold. Then there exists $C<\infty$, independent of $\beta$, such that, for every $\beta>0$,
\begin{equation}
\|V_\beta(Q_\beta)-V^*\|_\infty
\le
Ce^{-\beta\Gamma}.
\label{eq:large-beta-target}
\end{equation}
The corresponding policy bound is
\begin{equation}
\max_{s\in\cS}
\sum_{a\in\cA}
|\pi_\beta(a\mid s)-\pi^{\mathrm{unif}}(a\mid s)|
\le
C(1+\beta)e^{-\beta\Gamma}.
\label{eq:large-beta-policy}
\end{equation}
If $\Gamma=+\infty$, then $V_\beta(Q_\beta)=V^*$ and $\pi_\beta=\pi^{\mathrm{unif}}$ exactly, for every $\beta>0$.
\end{proposition}

\begin{proof}
The value bound follows by writing $Q_\beta=Q^*+q_\beta$ and using \eqref{eq:large-beta-target-bound}:
\[
|V_\beta(Q_\beta)(s)-V^*(s)|
\le
|V_\beta(Q_\beta)(s)-V_\beta(Q^*)(s)|+|V_\beta(Q^*)(s)-V^*(s)|
\le
C\|q_\beta\|_\infty+Ce^{-\beta\Gamma}.
\]
Here the Lipschitz constant for $V_\beta$ is uniform by Lemma~\ref{lem:if-omega-tv-bound}, and \eqref{eq:large-beta-residual-bound} controls the second term.

Fix $s\in\cS$, write $m\coloneqq|\cA^*(s)|\ge1$, set $u_a\coloneqq\beta q_\beta(s,a)$, and put $u_\beta^{\max}\coloneqq\beta\|q_\beta\|_\infty$. Then $u_\beta^{\max}\le C\beta e^{-\beta\Gamma}$. The normalizing constant and induced probabilities satisfy
\[
Z_s
=
\sum_{b\in\cA^*(s)}e^{u_b}
+
\sum_{b\notin\cA^*(s)}\exp(-\beta\Delta_b^*(s)+u_b),
\qquad
\pi_\beta(a\mid s)
=
\frac{\exp(-\beta\Delta_a^*(s)+u_a)}{Z_s}.
\]
For $a\notin\cA^*(s)$, the gap bound gives
\begin{equation}
\pi_\beta(a\mid s)
\le
e^{2u_\beta^{\max}}e^{-\beta\Gamma}.
\label{eq:large-beta-suboptimal-mass}
\end{equation}
For $a\in\cA^*(s)$, the identity $\pi_\beta(a\mid s)-1/m=(me^{u_a}-Z_s)/(mZ_s)$ and the bound $|e^x-e^y|\le e^{u_\beta^{\max}}|x-y|$ for $|x|\vee|y|\le u_\beta^{\max}$ give
\begin{equation}
\left|\pi_\beta(a\mid s)-\frac1m\right|
\le
e^{2u_\beta^{\max}}
\left(
2u_\beta^{\max}+|\cA|e^{-\beta\Gamma}
\right).
\label{eq:large-beta-tied-mass}
\end{equation}
When $\Gamma<\infty$, $\beta e^{-\beta\Gamma}$ is bounded over $\beta>0$, so summing \eqref{eq:large-beta-suboptimal-mass} and \eqref{eq:large-beta-tied-mass} over actions yields \eqref{eq:large-beta-policy}. Thus, when several actions maximize $Q^*(s,\cdot)$, the induced policy converges to the uniform distribution over those maximizing actions. If $\Gamma=+\infty$, Proposition~\ref{prop:large-beta-approx-tie} gives $Q_\beta=Q^*$, and $Q^*(s,\cdot)$ is constant on $\cA$, so the softmax policy is exactly uniform.
\end{proof}

Define $(D_{\pi^{\mathrm{unif}}}q)(s)\coloneqq\sum_{a\in\cA}\pi^{\mathrm{unif}}(a\mid s)q(s,a)$. For state-wise aggregation operators $D$ and $D'$ with weights $w_s$ and $w_s'$, write
\[
\|D-D'\|
\coloneqq
\max_{s\in\cS}
\sum_{a\in\cA}|w_s(a)-w_s'(a)|.
\]

\begin{proposition}[Convergence of the derivative operator]
\label{prop:large-beta-linearization}
Suppose Assumptions~\ref{ass:irreducible} and~\ref{ass:uniform-signed-aggregation-identification} hold. Then $D_{\pi^{\mathrm{unif}}}$ is allowable, and there exists $C<\infty$, independent of $\beta$, such that
\begin{equation}
\|D_\beta-D_{\pi^{\mathrm{unif}}}\|
\le
C(1+\beta)e^{-\beta\Gamma},
\qquad \beta>0.
\label{eq:large-beta-linearization}
\end{equation}
If $\Gamma=+\infty$, then $D_\beta=D_{\pi^{\mathrm{unif}}}$ exactly, for every $\beta>0$.
\end{proposition}

\begin{proof}
The weights of $D_{\pi^{\mathrm{unif}}}$ are nonnegative and sum to one over $\cA$, so their signed $\ell_1$ norm equals one, which is at most $C_{\cA}$. Hence $D_{\pi^{\mathrm{unif}}}$ is allowable.

By \eqref{eq:omega}, the weights of $D_\beta$ are $\omega_{\beta,Q_\beta}(a\mid s)=\pi_\beta(a\mid s)(1+\beta(Q_\beta(s,a)-V_\beta(Q_\beta)(s)))$. For $a\in\cA^*(s)$, \eqref{eq:large-beta-target-bound} and \eqref{eq:large-beta-target} give $\beta|Q_\beta(s,a)-V_\beta(Q_\beta)(s)|\le C\beta e^{-\beta\Gamma}$, so the optimal-action weights satisfy
\[
|\omega_{\beta,Q_\beta}(a\mid s)-\pi_\beta(a\mid s)|
\le
C\beta e^{-\beta\Gamma},
\qquad a\in\cA^*(s).
\]
Combining this with \eqref{eq:large-beta-policy} gives
\[
|\omega_{\beta,Q_\beta}(a\mid s)-\pi^{\mathrm{unif}}(a\mid s)|
\le
C(1+\beta)e^{-\beta\Gamma},
\qquad a\in\cA^*(s).
\]
For $a\notin\cA^*(s)$, Corollary~\ref{cor:app-smoothed-target-bounded} gives $|Q_\beta(s,a)-V_\beta(Q_\beta)(s)|\le C$, while \eqref{eq:large-beta-suboptimal-mass} gives $\pi_\beta(a\mid s)\le Ce^{-\beta\Gamma}$. Hence
\[
|\omega_{\beta,Q_\beta}(a\mid s)|
\le
\pi_\beta(a\mid s)(1+C\beta)
\le
C(1+\beta)e^{-\beta\Gamma},
\qquad a\notin\cA^*(s).
\]
Summing over $a\in\cA$ and maximizing over $s$ proves \eqref{eq:large-beta-linearization}. If $\Gamma=+\infty$, then $Q_\beta=Q^*$, $V_\beta(Q_\beta)=V^*$, and $\pi_\beta=\pi^{\mathrm{unif}}$; because $Q^*(s,\cdot)$ is constant on $\cA$, the derivative correction vanishes and $D_\beta=D_{\pi^{\mathrm{unif}}}$.
\end{proof}

\section{Proofs for Section~\ref{sec:if}}\label{app:inference-proofs}
This section proves the inference results in Section~\ref{sec:if}. Section~\ref{sec:app-if-prelim} develops the derivative, adjoint weight, and orthogonality tools. Section~\ref{sec:app-de} establishes the common cross-fitted decomposition and controls nuisance-dependent remainder terms. Section~\ref{sec:app-growing-oracle} develops the additional dependence arguments used when the horizon grows, and Section~\ref{sec:app-martingale-qv} establishes the martingale and quadratic variation results. Section~\ref{sec:app-unified-inference} proves Theorem~\ref{thm:clt-N} in the fixed and diverging horizon regimes. Section~\ref{sec:app-efficiency-pe-limit} establishes the efficiency results and the large-$\beta$ score limit, and Section~\ref{sec:app-feasible-variance} proves Corollary~\ref{cor:studentized-optimal-value}.
We write $\widehat\eta_{-k}:=(\widehat J_{-k},\widehat Q_{-k},\widehat\lambda_{-k})$, and let $k(i)$ denote the fold containing trajectory $i$.
We restore the suppressed dependence on $H$ and write $\bar P_H,\bar\mu_H,\bar\nu_H$ for the pooled one-step law and its $(S,A)$  and $S'$ marginals at horizon $H$.

\subsection{Derivative and Adjoint Weight Preliminaries}\label{sec:app-if-prelim}

\subsubsection{Derivative and Taylor Remainder of the Smoothed Value Map}\label{app:if-orth-decomp}
We derive a uniform second-order remainder bound for the smoothed value map $Q\mapsto V_\beta(Q)$.
These bounds are used to control the estimation error in $Q$ in the cross-fitted score.

\begin{lemma}\label{lem:if-argsoftmax-bounds}
Under Assumption~\ref{ass:pooled-support}, there exists a constant $C<\infty$ such that, for every fold $k$, the remainder $ R_{\beta,-k}(S') \coloneqq V_\beta(\widehat Q_{-k})(S') -V_\beta(Q_\beta)(S') -D_QV_\beta(Q_\beta)(S')[\widehat Q_{-k}-Q_\beta] $ satisfies
\[
\|D_QV_\beta(Q_\beta)(S')[\widehat Q_{-k}-Q_\beta]\|_{\bar P_H,2}
\le
C\|\widehat Q_{-k}-Q_\beta\|_{\bar\nu_H,2},
\]
\[
\|R_{\beta,-k}(S')\|_{\bar P_H,2}
\le
C\beta\|\widehat Q_{-k}-Q_\beta\|_{\bar\nu_H,2}^2.
\]
\end{lemma}

\begin{proof}
Apply the derivative bounds in Lemma~\ref{lem:softmax-calculus} with $Q=Q_\beta$ and perturbation direction $\widehat Q_{-k}-Q_\beta$. To make the resulting norm conversion explicit, fix $s'\in\mathcal S$. The directional derivative is
\[
\begin{aligned}
D_QV_\beta(Q_\beta)(s')[\widehat Q_{-k}-Q_\beta]
&= \sum_{a\in\mathcal A}\omega_{\beta,Q_\beta}(a\mid s')(\widehat Q_{-k}(s',a)-Q_\beta(s',a)).
\end{aligned}
\]
Cauchy--Schwarz and the uniform first-derivative bound in Lemma~\ref{lem:softmax-calculus} give
\[
\left|
D_QV_\beta(Q_\beta)(s')[\widehat Q_{-k}-Q_\beta]
\right|
\le
C
\left(
\sum_{a\in\mathcal A}
(\widehat Q_{-k}(s',a)-Q_\beta(s',a))^2
\right)^{1/2}.
\]
Since this bound holds for every $s'$, we may evaluate it at the random next state $S'$. Hence
\[
\left|
D_QV_\beta(Q_\beta)(S')[\widehat Q_{-k}-Q_\beta]
\right|
\le
C
\left(
\sum_{a\in\mathcal A}
(\widehat Q_{-k}(S',a)-Q_\beta(S',a))^2
\right)^{1/2}.
\]
Squaring both sides and taking expectation under the pooled one-step law converts the pointwise bound into the $\bar\nu_H$-norm.
\[
\EE_{\bar P_H}\left[
\left(
D_QV_\beta(Q_\beta)(S')[\widehat Q_{-k}-Q_\beta]
\right)^2
\right]
\le
C^2
\EE_{\bar P_H}\sum_{a\in\mathcal A}
(\widehat Q_{-k}(S',a)-Q_\beta(S',a))^2.
\]
The corresponding norm is $
\|\widehat Q_{-k}-Q_\beta\|_{\bar\nu_H,2}^2
=
\EE_{\bar P_H}\sum_{a\in\mathcal A}
(\widehat Q_{-k}(S',a)-Q_\beta(S',a))^2 . $
Taking square roots, we obtain
\[
\left\|
D_QV_\beta(Q_\beta)(S')[\widehat Q_{-k}-Q_\beta]
\right\|_{\bar P_H,2}
\le
C
\|\widehat Q_{-k}-Q_\beta\|_{\bar\nu_H,2}.
\]
This proves the first claimed inequality.

The Taylor bound in Lemma~\ref{lem:softmax-calculus} gives the pointwise inequality
\[
|R_{\beta,-k}(S')|
\le
C\beta
\sum_{a\in\cA}
(\widehat Q_{-k}(S',a)-Q_\beta(S',a))^2.
\]
The resulting $L_2(\bar P_H)$ bound is
\[
\|R_{\beta,-k}(S')\|_{\bar P_H,2}
\le
C\beta
\left\|
\sum_{a\in\cA}
(\widehat Q_{-k}(S',a)-Q_\beta(S',a))^2
\right\|_{\bar P_H,2}.
\]
Recall that $\bar\nu_H$ denotes the pooled marginal law of $S'$.
By the uniform pooled support condition in Assumption~\ref{ass:pooled-support}, the finite support gives a constant $C<\infty$ such that, for every nonnegative function $g:\cS\to\RR_+$,
\[
\|g(S')\|_{\bar P_H,2}
=
\left(
\sum_{s'\in\cS}
\bar\nu_H(s')g(s')^2
\right)^{1/2}
\le
C
\sum_{s'\in\cS}
\bar\nu_H(s')g(s')
=
C\EE_{\bar P_H}[g(S')].
\]
Applying this inequality with $g(s')=\sum_{a\in\cA}(\widehat Q_{-k}(s',a)-Q_\beta(s',a))^2$ yields
\[
\left\|
\sum_{a\in\cA}
(\widehat Q_{-k}(S',a)-Q_\beta(S',a))^2
\right\|_{\bar P_H,2}
\le
C
\EE_{\bar P_H}\sum_{a\in\cA}
(\widehat Q_{-k}(S',a)-Q_\beta(S',a))^2 .
\]
The expectation on the right equals $\|\widehat Q_{-k}-Q_\beta\|_{\bar\nu_H,2}^2$ by the definition of the $\bar\nu_H$-norm. Combining the preceding displays gives
\[
\|R_{\beta,-k}(S')\|_{\bar P_H,2}
\le
C\beta
\|\widehat Q_{-k}-Q_\beta\|_{\bar\nu_H,2}^2.
\]
The constant $C$ absorbs the finite support constants. This proves the stated remainder bound.

\end{proof}

\subsubsection{Adjoint Weight Existence and Neyman Orthogonality}
\begin{proof}[Proof of Proposition~\ref{prop:adjoint-existence}.]
Recall the Section~\ref{sec:if} definition of $D_\beta$. Lemma~\ref{lem:softmax-calculus}, applied at $Q=Q_\beta$, gives $\sum_{a\in\cA}\omega_{\beta,Q_\beta}(a\mid s)=1$. The signed $\ell_1$ envelope $\sum_{a\in\cA}|\omega_{\beta,Q_\beta}(a\mid s)|\le C_{\cA}$ follows from Lemma~\ref{lem:if-omega-tv-bound}, with $Q=Q_\beta$. The row sum identity gives $D_\beta\one=\one$, hence $PD_\beta\one=\one$ and $(I-PD_\beta)\one=0$.

We first show that $\ker(I-PD_\beta)=\operatorname{span}\{\one\}$. Let $f$ satisfy $(I-PD_\beta)f=0$, and write $f=q+c\one$, where $q\in\cQ$. Since $(I-PD_\beta)\one=0$, $(I-PD_\beta)q=0$, equivalently $-0\cdot\one-q+PD_\beta q=0$. The injectivity condition in Assumption~\ref{ass:uniform-signed-aggregation-identification} gives $q=0$. Hence $f$ is constant. The reverse inclusion follows from $(I-PD_\beta)\one=0$.

Because the state-action space is finite, rank-nullity gives a one-dimensional left nullspace $\ker((I-PD_\beta)^\top)$. We next show that some vector in this left nullspace has nonzero total mass. Suppose, toward a contradiction, that every $\rho\in\ker((I-PD_\beta)^\top)$ satisfies $\rho^\top\one=0$. Then $\one$ is orthogonal to $\ker((I-PD_\beta)^\top)$, so $\one\in\operatorname{Range}(I-PD_\beta)$. Hence $(I-PD_\beta)f=\one$ for some $f$. Write $f=q+c\one$ with $q\in\cQ$. Since $(I-PD_\beta)\one=0$, $(I-PD_\beta)q=\one$, or equivalently $1-q+PD_\beta q=0$. This contradicts Assumption~\ref{ass:uniform-signed-aggregation-identification}. Hence there exists $\rho_\beta\in\ker((I-PD_\beta)^\top)$ with $\rho_\beta^\top\one\ne0$. Rescale it so that $\rho_\beta^\top\one=1$.

Define $\lambda_\beta(s,a)\coloneqq\rho_\beta(s,a)/\bar\mu_H(s,a)$. Full support of $\bar\mu_H$ and finite support imply that $\lambda_\beta$ is square-integrable under the pooled one-step law, and its normalization is inherited from $\rho_\beta^\top\one=1$, since
\[
\EE_{\bar P_H}[\lambda_\beta(S,A)]
=
\sum_{s,a}\bar\mu_H(s,a)\lambda_\beta(s,a)
=
\rho_\beta^\top\one
=1.
\]
For any $q\in\cQ$,
\[
\begin{aligned}
&\EE_{\bar P_H}\!\left[
\lambda_\beta(S,A)
\left(
q(S,A)-\sum_{a\in\cA}\omega_{\beta,Q_\beta}(a\mid S')q(S',a)
\right)
\right]  =
\sum_{s,a}\rho_\beta(s,a)((I-PD_\beta)q)(s,a)
=
\rho_\beta^\top (I-PD_\beta)q
=0,
\end{aligned}
\]
because $\rho_\beta\in\ker((I-PD_\beta)^\top)$. This verifies the normalization and adjoint equation under the pooled one-step behavior law.

It remains to prove uniqueness. Let $\widetilde\lambda_\beta$ be any other adjoint weight satisfying the same normalization and adjoint equations, and define $ \widetilde\rho_\beta(s,a) \coloneqq \bar\mu_H(s,a)\widetilde\lambda_\beta(s,a). $
The normalization gives $ \widetilde\rho_\beta^\top\one=1 $, and the adjoint equation gives $ \widetilde\rho_\beta^\top (I-PD_\beta)q=0 $ for every  $q\in\cQ.$
For an arbitrary $f$, write $f=q+c\one$ with $q\in\cQ$. Since $(I-PD_\beta)\one=0$, we have $\widetilde\rho_\beta^\top (I-PD_\beta)f=\widetilde\rho_\beta^\top (I-PD_\beta)q=0$, and hence $\widetilde\rho_\beta\in\ker((I-PD_\beta)^\top)$. The left nullspace is one-dimensional, and both $\widetilde\rho_\beta$ and $\rho_\beta$ satisfy the normalization $\rho^\top\one=1$. Therefore $ \widetilde\rho_\beta=\rho_\beta. $
Since $\bar\mu_H(s,a)>0$ for every $(s,a)$, $ \widetilde\lambda_\beta = \lambda_\beta. $
This proves uniqueness of the adjoint weight.
\end{proof}

\begin{proof}[Proof of Proposition~\ref{prop:neyman-orthogonality}.]
For this proof, write $\Phi_\beta(J,Q,\lambda):=\EE_{\bar P_H}[\varphi_\beta(Z;J,Q,\lambda)]$.
The $J$-directional derivative is $D_J\Phi_\beta(J,Q,\lambda)[j]=j-j\EE_{\bar P_H}[\lambda(S,A)]$.
The $Q$-directional derivative is
\[
D_Q\Phi_\beta(J,Q,\lambda)[q]
=
\EE_{\bar P_H}\!\left[
\lambda(S,A)(D_QV_\beta(Q)(S')[q]-q(S,A))
\right].
\]
The $\lambda$-directional derivative is
\[
D_\lambda\Phi_\beta(J,Q,\lambda)[\ell]
=
\EE_{\bar P_H}\!\left[
\ell(S,A)(R-J+V_\beta(Q)(S')-Q(S,A))
\right].
\]
At the truth, the pooled normalization in Proposition~\ref{prop:adjoint-existence} makes the $J$-derivative zero, the adjoint equation makes the $Q$-derivative zero, and \eqref{eq:self} makes the conditional expectation of the residual given $(S,A)$ zero for the $\lambda$-derivative.
\end{proof}

\subsection{Common Cross-Fitted Decomposition and Nuisance Bounds}\label{sec:app-de}
\subsubsection{Cross-Fitted Decomposition}
\begin{lemma}\label{lem:if-decomposition}
Define the three components
\[
L_{N,H}\coloneqq
\frac1N\sum_{i=1}^N\frac1H\sum_{t=0}^{H-1}
(\varphi_\beta(Z_{i,t})-J_\beta),
\]
\[
B_{N,H}\coloneqq
\sum_{k=1}^K\frac{|I_k|}{N}
\frac1H\sum_{t=0}^{H-1}
\EE\!\left[
\varphi_\beta(Z_t;\widehat\eta_{-k})
-
\varphi_\beta(Z_t)
\,\middle|\,
\mathcal D_{-k}
\right], \text{ and}
\]
\[
E_{N,H}\coloneqq
\sum_{k=1}^K\frac1N\sum_{i\in I_k}
\frac1H\sum_{t=0}^{H-1}
\left(
\varphi_\beta(Z_{i,t};\widehat\eta_{-k})
-\varphi_\beta(Z_{i,t})
-\EE_{\bar P_H}[
\varphi_\beta(Z;\widehat\eta_{-k})
-\varphi_\beta(Z)
\mid\mathcal D_{-k}]
\right).
\]
Then the cross-fitted estimator satisfies
\[
\widehat J_\beta-J_\beta
=
L_{N,H}+B_{N,H}+E_{N,H}.
\]
\end{lemma}

\begin{proof} 
By the definition of the cross-fitted estimator,
\[
\widehat J_\beta-J_\beta
=
\frac1N\sum_{k=1}^K\sum_{i\in I_k}
\frac1H\sum_{t=0}^{H-1}
(\varphi_\beta(Z_{i,t};\widehat\eta_{-k})-J_\beta).
\]
Adding and subtracting $\varphi_\beta(Z_{i,t})$ separates the leading term from the nuisance perturbation.
\[
\begin{aligned}
\widehat J_\beta-J_\beta
&= \frac1N\sum_{i=1}^N\frac1H\sum_{t=0}^{H-1}(\varphi_\beta(Z_{i,t})-J_\beta) + \sum_{k=1}^K\frac1N\sum_{i\in I_k}\frac1H\sum_{t=0}^{H-1}( \varphi_\beta(Z_{i,t};\widehat\eta_{-k})-\varphi_\beta(Z_{i,t}) ).
\end{aligned}
\]
The first term is $L_{N,H}$. For the second term, add and subtract the time-$t$ foldwise conditional mean $\EE[ \varphi_\beta(Z_t;\widehat\eta_{-k})-\varphi_\beta(Z_t) \mid \mathcal D_{-k}]$. Since this conditional mean does not depend on the evaluation trajectory index $i\in I_k$, its contribution after averaging over $i$ is $B_{N,H}$. The remaining centered term is exactly $E_{N,H}$. This proves the decomposition.
\end{proof}

\subsubsection{Score Perturbation Expansion and Population Bias}\label{sec:app-population-bias-bound}

\begin{lemma}[Algebraic expansion of the score perturbation]\label{lem:if-remainder-bound}
Let $\eta_\beta\coloneqq(J_\beta,Q_\beta,\lambda_\beta)$. Define the foldwise score perturbation
\[
\Delta_{\beta,-k}(Z)
\coloneqq
\varphi_\beta(Z;\widehat\eta_{-k})
-
\varphi_\beta(Z;\eta_\beta)
=
\varphi_\beta(Z;\widehat\eta_{-k})
-\varphi_\beta(Z).
\]
Let $\delta J\coloneqq\widehat J_{-k}-J_\beta$, $\delta Q\coloneqq\widehat Q_{-k}-Q_\beta$, $\delta\lambda\coloneqq\widehat\lambda_{-k}-\lambda_\beta$, and $\delta V_{\beta,-k}(S')\coloneqq V_\beta(\widehat Q_{-k})(S')-V_\beta(Q_\beta)(S')$. Then
\begin{align*}
\Delta_{\beta,-k}(Z)
&=(1-\lambda_\beta(S,A))\delta J
+\lambda_\beta(S,A)(\delta V_{\beta,-k}(S')-\delta Q(S,A)) \\
&\quad
+\delta\lambda(S,A)
(R-J_\beta-\delta J+V_\beta(\widehat Q_{-k})(S')-\widehat Q_{-k}(S,A)).
\end{align*}
\end{lemma}

\begin{proof}
By definition, we have
\[
\varphi_\beta(Z;\widehat\eta_{-k})
=
(J_\beta+\delta J)
+(\lambda_\beta+\delta\lambda)(S,A)
(R-J_\beta-\delta J
+V_\beta(Q_\beta+\delta Q)(S')
-Q_\beta(S,A)-\delta Q(S,A)).
\]
The score is
\[
\varphi_\beta(Z)
=
J_\beta
+\lambda_\beta(S,A)
(R-J_\beta+V_\beta(Q_\beta)(S')-Q_\beta(S,A)).
\]
The residual factor in the first line equals $R-J_\beta-\delta J+V_\beta(\widehat Q_{-k})(S')-\widehat Q_{-k}(S,A)$. Subtracting the oracle score from the plug-in score gives
\begin{align*}
\Delta_{\beta,-k}(Z)
&=\delta J
+\lambda_\beta(S,A)
(-\delta J+\delta V_{\beta,-k}(S')-\delta Q(S,A)) \\
&\quad
+\delta\lambda(S,A)
(R-J_\beta-\delta J
+V_\beta(\widehat Q_{-k})(S')
-\widehat Q_{-k}(S,A)).
\end{align*}
Combining the first two terms gives the displayed identity.
\end{proof}

We then prove the population plug-in bias bound.
When the argument depends on $\widehat\eta_{-k}$, the norm is interpreted conditionally on the training data $\mathcal D_{-k}$.
\begin{proposition}\label{prop:population-bias-bound}
Suppose Assumptions~\ref{ass:irreducible}, \ref{ass:uniform-signed-aggregation-identification}, \ref{ass:pooled-support}, and~\ref{ass:if-envelope} hold. Under Assumption~\ref{ass:if-fixed-rate}, $
B_{N,H}=o_p(N^{-1/2}). $
If Assumption~\ref{ass:if-growing-rate} holds instead, then $
B_{N,H}=o_p((NH)^{-1/2}). $
\end{proposition}

\begin{proof} 
Recall the score perturbation $\Delta_{\beta,-k}$ from Lemma~\ref{lem:if-remainder-bound}, and define $\delta_\beta(Z)\coloneqq R-J_\beta+V_\beta(Q_\beta)(S')-Q_\beta(S,A)$.
By Lemma~\ref{lem:if-remainder-bound},
\begin{align*}
\Delta_{\beta,-k}(Z)
=&
(1-\lambda_\beta(S,A))\delta J_{-k}
+
\lambda_\beta(S,A)
(\delta V_{\beta,-k}(S')-\delta Q_{-k}(S,A))\\
&
+
\delta\lambda_{-k}(S,A)
(\delta_\beta(Z)-\delta J_{-k}
+\delta V_{\beta,-k}(S')-\delta Q_{-k}(S,A)).
\end{align*}
Conditional on $\mathcal D_{-k}$, the quantities $\delta J_{-k}$, $\delta Q_{-k}$, and $\delta\lambda_{-k}$ are fixed. The first two cancellations come from the adjoint normalization and the Bellman equation.
\[
\EE_{\bar P_H}
((1-\lambda_\beta(S,A))\delta J_{-k}\mid\mathcal D_{-k})
=
\delta J_{-k}
(1-\EE_{\bar P_H}[\lambda_\beta(S,A)])
=0.
\]

The Bellman residual also has zero conditional mean, $\EE_{\bar P_H}(\delta_\beta(Z)\mid S,A)=0$. The residual component multiplied by $\delta\lambda_{-k}$ therefore has zero conditional expectation:
\[
\EE_{\bar P_H}
(\delta\lambda_{-k}(S,A)\delta_\beta(Z)\mid\mathcal D_{-k})
=
\EE_{\bar P_H}
\big[
\delta\lambda_{-k}(S,A)
\EE_{\bar P_H}(\delta_\beta(Z)\mid S,A)
\mid\mathcal D_{-k}
\big]
=0.
\]
Combining the preceding cancellations gives
\begin{align*}
\EE_{\bar P_H}(\Delta_{\beta,-k}(Z)\mid\mathcal D_{-k})
=&
\EE_{\bar P_H}
\big[
\lambda_\beta(S,A)
(\delta V_{\beta,-k}(S')-\delta Q_{-k}(S,A))
\mid\mathcal D_{-k}
\big]\\
&
+
\EE_{\bar P_H}
\big[
\delta\lambda_{-k}(S,A)
(-\delta J_{-k}
+\delta V_{\beta,-k}(S')-\delta Q_{-k}(S,A))
\mid\mathcal D_{-k}
\big].
\end{align*}
Next use the Taylor decomposition $\delta V_{\beta,-k}(S')=D_QV_\beta(Q_\beta)(S')[\delta Q_{-k}]+R_{\beta,-k}(S')$. By the deterministic anchoring convention in the Section~\ref{sec:if} setup, $\widehat Q_{-k}\in\cQ$ for each fold. Since $Q_\beta\in\cQ$, we have $\delta Q_{-k}\in\cQ$, so the adjoint identity in Proposition~\ref{prop:adjoint-existence} applies with $q=\delta Q_{-k}$.
\[
\EE_{\bar P_H}
\big[
\lambda_\beta(S,A)
(D_QV_\beta(Q_\beta)(S')[\delta Q_{-k}]
-\delta Q_{-k}(S,A))
\mid\mathcal D_{-k}
\big]
=0.
\]
Only the Taylor remainder and the bilinear nuisance term remain, giving
\begin{align*}
\EE_{\bar P_H}(\Delta_{\beta,-k}(Z)\mid\mathcal D_{-k})
=&
\EE_{\bar P_H}
(\lambda_\beta(S,A)R_{\beta,-k}(S')\mid\mathcal D_{-k})\\
&
+
\EE_{\bar P_H}
\big[
\delta\lambda_{-k}(S,A)
(-\delta J_{-k}
+\delta V_{\beta,-k}(S')-\delta Q_{-k}(S,A))
\mid\mathcal D_{-k}
\big].
\end{align*}

We now bound the two remaining terms. Assumption~\ref{ass:if-envelope} together with Corollaries~\ref{cor:app-smoothed-target-bounded} and~\ref{cor:target-adjoint-bound}, and Lemma~\ref{lem:if-argsoftmax-bounds}, give
\[
\left|
\EE_{\bar P_H}
(\lambda_\beta(S,A)R_{\beta,-k}(S')\mid\mathcal D_{-k})
\right|
\le
C\|R_{\beta,-k}(S')\|_{\bar P_H,2}
\le
C\,\beta\|\delta Q_{-k}\|_{\bar\nu_H,2}^2.
\]
For the bilinear term, Cauchy--Schwarz gives
\begin{align*}
&\left|
\EE_{\bar P_H}
\big[
\delta\lambda_{-k}(S,A)
(-\delta J_{-k}
+\delta V_{\beta,-k}(S')-\delta Q_{-k}(S,A))
\mid\mathcal D_{-k}
\big]
\right|\\
&\quad
\le
\|\delta\lambda_{-k}\|_{\bar P_H,2}
\,
\|-\delta J_{-k}
+\delta V_{\beta,-k}(S')-\delta Q_{-k}(S,A)\|_{\bar P_H,2}.
\end{align*}
Using the derivative remainder decomposition again and Lemma~\ref{lem:if-argsoftmax-bounds},
\[
\|-\delta J_{-k}
+\delta V_{\beta,-k}(S')-\delta Q_{-k}(S,A)\|_{\bar P_H,2}
\le
|\delta J_{-k}|
+C\|\delta Q_{-k}\|_{\bar\nu_H,2}
+C\,\beta\|\delta Q_{-k}\|_{\bar\nu_H,2}^2.
\]
Combining the preceding displays yields
\begin{align*}
\left|
\EE_{\bar P_H}
(\Delta_{\beta,-k}(Z)\mid\mathcal D_{-k})
\right|
\le&
C\,\beta\|\delta Q_{-k}\|_{\bar\nu_H,2}^2
+\|\delta\lambda_{-k}\|_{\bar P_H,2}|\delta J_{-k}|\\
&\quad
+C\|\delta\lambda_{-k}\|_{\bar P_H,2}\|\delta Q_{-k}\|_{\bar\nu_H,2}
+C\,\beta\|\delta\lambda_{-k}\|_{\bar P_H,2}
\|\delta Q_{-k}\|_{\bar\nu_H,2}^2 .
\end{align*}
The right-hand side is bounded by $Cb_k$, so that
\[
\left|
\EE_{\bar P_H}
(\Delta_{\beta,-k}(Z)\mid\mathcal D_{-k})
\right|
\le C b_k.
\]
Under Assumption~\ref{ass:if-fixed-rate}, $\max_k b_k=o_p(N^{-1/2})$, and hence
\[
\max_{1\le k\le K}
\left|
\EE_{\bar P_H}
(\Delta_{\beta,-k}(Z)\mid\mathcal D_{-k})
\right|
=
o_p(N^{-1/2}).
\]
By Lemma~\ref{lem:if-decomposition},
\[
B_{N,H}
=
\sum_{k=1}^K
\frac{|I_k|}{N}
\EE_{\bar P_H}
(\Delta_{\beta,-k}(Z)\mid\mathcal D_{-k}).
\]
The foldwise representation therefore gives
\[
|B_{N,H}|
\le
\sum_{k=1}^K\frac{|I_k|}{N}
\left|
\EE_{\bar P_H}
(\Delta_{\beta,-k}(Z)\mid\mathcal D_{-k})
\right|
=
o_p(N^{-1/2}).
\]
Hence $B_{N,H}=o_p(N^{-1/2})$.

Under Assumption~\ref{ass:if-growing-rate}, $\max_k b_k=o_p((NH)^{-1/2})$, so
\[
\max_{1\le k\le K}
\left|
\EE_{\bar P_H}
(\Delta_{\beta,-k}(Z)\mid\mathcal D_{-k})
\right|
=
o_p((NH)^{-1/2}).
\]
Substituting this uniform rate into the same foldwise average gives
\[
|B_{N,H}|
\le
\sum_{k=1}^K\frac{|I_k|}{N}
\left|
\EE_{\bar P_H}
(\Delta_{\beta,-k}(Z)\mid\mathcal D_{-k})
\right|
=
o_p((NH)^{-1/2}).
\]
Thus $B_{N,H}=o_p((NH)^{-1/2})$.
\end{proof}

\subsubsection{Perturbation Bounds and Empirical Score Replacement}
Define the estimator boundedness event
\[
\mathcal E_{\mathrm{bd}}
:=
\left\{
\max_{1\le k\le K}
\left(
|\widehat J_{-k}|
+
\|\widehat Q_{-k}\|_\infty
+
\|\widehat\lambda_{-k}\|_\infty
\right)
\le C
\right\},
\]
for a sufficiently large fixed constant $C<\infty$.
\begin{lemma}[Score perturbation bound]\label{lem:score-perturbation-bound}
Suppose Assumptions~\ref{ass:irreducible}, \ref{ass:uniform-signed-aggregation-identification}, \ref{ass:pooled-support}, and~\ref{ass:if-envelope} hold and $\mathcal E_{\mathrm{bd}}$ occurs. Then there exists a constant $C<\infty$, depending only on the population envelope constants and the constant in $\mathcal E_{\mathrm{bd}}$, such that, for every fold $k$,
\[
\|\Delta_{\beta,-k}\|_{\bar P_H,2}
\le
C(|\delta J_{-k}|
+\|\delta Q_{-k}\|_{\bar\nu_H,2}
+\beta\|\delta Q_{-k}\|_{\bar\nu_H,2}^2
+\|\delta\lambda_{-k}\|_{\bar P_H,2}),
\]
Moreover, if Assumption~\ref{ass:if-consistency} also holds, then, for a sufficiently large fixed $C$, $\PP(\mathcal E_{\mathrm{bd}})\to1$.
\end{lemma}

\begin{proof}
Lemma~\ref{lem:if-remainder-bound} gives the exact expansion
\[
\begin{aligned}
\Delta_{\beta,-k}(Z)
&=(1-\lambda_\beta(S,A))\delta J_{-k}
+\lambda_\beta(S,A)
(\delta V_{\beta,-k}(S')-\delta Q_{-k}(S,A)) \\
&\quad
+\delta\lambda_{-k}(S,A)
(R-J_\beta-\delta J_{-k}
+V_\beta(\widehat Q_{-k})(S')
-\widehat Q_{-k}(S,A)).
\end{aligned}
\]
Taking the $L_2(\bar P_H)$ norm and applying the triangle inequality gives three contributions.

The first contribution is the $J$-estimation term, and it is bounded by $\|(1-\lambda_\beta(S,A))\delta J_{-k}\|_{\bar P_H,2}\le C|\delta J_{-k}|$. The second contribution is the estimation term for $Q$. Write the perturbation of the value map as $\delta V_{\beta,-k}(S')=D_QV_\beta(Q_\beta)(S')[\delta Q_{-k}]+R_{\beta,-k}(S')$. Lemma~\ref{lem:if-argsoftmax-bounds} gives $\|D_QV_\beta(Q_\beta)(S')[\delta Q_{-k}]\|_{\bar P_H,2}\le C\|\delta Q_{-k}\|_{\bar\nu_H,2}$. The Taylor remainder satisfies $\|R_{\beta,-k}(S')\|_{\bar P_H,2}\le C\beta\|\delta Q_{-k}\|_{\bar\nu_H,2}^2$.
Finiteness and the common lower bound on $\bar\nu_H$ in Assumption~\ref{ass:pooled-support} give $\|\delta Q_{-k}(S,A)\|_{\bar P_H,2}\le C\|\delta Q_{-k}\|_{\bar\nu_H,2}$ uniformly in $H$. Combining this comparison with the preceding bounds gives
\[
\|\delta V_{\beta,-k}(S')-\delta Q_{-k}(S,A)\|_{\bar P_H,2}
\le
C\|\delta Q_{-k}\|_{\bar\nu_H,2}
+C\beta\|\delta Q_{-k}\|_{\bar\nu_H,2}^2.
\]
Multiplying by the population envelope for $\lambda_\beta$ gives
\[
\|\lambda_\beta(S,A)(\delta V_{\beta,-k}(S')-\delta Q_{-k}(S,A))\|_{\bar P_H,2}
\le
C(\|\delta Q_{-k}\|_{\bar\nu_H,2}
+\beta\|\delta Q_{-k}\|_{\bar\nu_H,2}^2).
\]
The third contribution is the estimation term for the adjoint weight. By Assumption~\ref{ass:if-envelope}, Corollary~\ref{cor:app-smoothed-target-bounded}, on $\mathcal E_{\mathrm{bd}}$, and using $|V_\beta(Q)(s)|\le \|Q\|_\infty$, we have $|R-J_\beta-\delta J_{-k}+V_\beta(\widehat Q_{-k})(S')-\widehat Q_{-k}(S,A)|\le C$ almost surely.
It follows that the adjoint-weight contribution satisfies
\[
\|\delta\lambda_{-k}(S,A)
(R-J_\beta-\delta J_{-k}
+V_\beta(\widehat Q_{-k})(S')
-\widehat Q_{-k}(S,A))\|_{\bar P_H,2}
\le
C\|\delta\lambda_{-k}\|_{\bar P_H,2}.
\]
Combining these bounds proves the stated result.

It remains to establish the final assertion.
The common support lower bound and the definitions of the weighted norms give $\|\widehat Q_{-k}-Q_\beta\|_\infty\le C\|\widehat Q_{-k}-Q_\beta\|_{\bar\nu_H,2}$.
Similarly, the common support lower bound gives $\|\widehat\lambda_{-k}-\lambda_\beta\|_\infty\le C\|\widehat\lambda_{-k}-\lambda_\beta\|_{\bar P_H,2}$.
The same support constants are uniform in $N,H,\beta$ by Assumption~\ref{ass:pooled-support}. Together with Assumptions~\ref{ass:if-envelope} and~\ref{ass:if-consistency}, this implies
\[
\max_{1\le k\le K}
\left(
|\widehat J_{-k}|
+\|\widehat Q_{-k}\|_\infty
+\|\widehat\lambda_{-k}\|_\infty
\right)
\le C+o_p(1),
\]
for a constant $C<\infty$ uniform in $N,H,\beta$. Therefore, a sufficiently large fixed $C$ satisfies $\PP(\mathcal E_{\mathrm{bd}})\to1$.
\end{proof}

\begin{lemma}[Cross-fitted empirical score replacement]\label{lem:empirical-score-replacement}
Suppose Assumptions~\ref{ass:irreducible}, \ref{ass:uniform-signed-aggregation-identification}, \ref{ass:pooled-support}, \ref{ass:if-envelope}, and~\ref{ass:if-consistency} hold.
Suppose also that $
\max_{1\le k\le K}
\beta\|\widehat Q_{-k}-Q_\beta\|_{\bar\nu_H,2}^2
=o_p(1). $
The score replacement bound is
\[
\max_{1\le k\le K}
\left\|
\varphi_\beta(\cdot;\widehat\eta_{-k})
-
\varphi_\beta(\cdot)
\right\|_{\bar P_H,2}
=o_p(1).
\]
This convergence gives
\[
\frac{1}{NH}\sum_{i=1}^N\sum_{t=0}^{H-1}
\left(
\varphi_\beta(Z_{i,t};\widehat\eta_{-k(i)})
-
\varphi_\beta(Z_{i,t})
\right)^2
=o_p(1).
\]
\end{lemma}

\begin{proof}
Lemma~\ref{lem:score-perturbation-bound} gives, on $\mathcal E_{\mathrm{bd}}$,
\[
\|\Delta_{\beta,-k}\|_{\bar P_H,2}
\le
C(|\delta J_{-k}|
+\|\delta Q_{-k}\|_{\bar\nu_H,2}
+\beta\|\delta Q_{-k}\|_{\bar\nu_H,2}^2
+\|\delta\lambda_{-k}\|_{\bar P_H,2}).
\]
The additional nonlinear consistency condition and Assumption~\ref{ass:if-consistency} make the right-hand side $o_p(1)$ uniformly over $k$. Lemma~\ref{lem:score-perturbation-bound} gives $\PP(\mathcal E_{\mathrm{bd}})\to1$ under the stated assumptions. Therefore $\max_k\|\Delta_{\beta,-k}\|_{\bar P_H,2}=o_p(1)$.

For fold $k$, write $\Delta_{-k}(Z)\coloneqq\varphi_\beta(Z;\widehat\eta_{-k})-\varphi_\beta(Z)$. Conditional on the training data $\mathcal D_{-k}$, the function $\Delta_{-k}$ is fixed and each evaluation trajectory in $I_k$ is independent of $\mathcal D_{-k}$, so
\[
\EE\!\left[
\frac{1}{|I_k|H}
\sum_{i\in I_k}\sum_{t=0}^{H-1}
\Delta_{-k}(Z_{i,t})^2
\,\middle|\,
\mathcal D_{-k}
\right]
=
\|\Delta_{-k}\|_{\bar P_H,2}^2.
\]
The preceding foldwise $L_2(\bar P_H)$ rate makes the right-hand side $o_p(1)$ uniformly over $k$, so conditional Markov's inequality and fixed $K$ give $\max_{1\le k\le K}\frac{1}{|I_k|H}\sum_{i\in I_k}\sum_{t=0}^{H-1}\Delta_{-k}(Z_{i,t})^2=o_p(1)$. Taking the convex combination of these foldwise averages with weights $|I_k|/N$ proves the claim.
\end{proof}

\subsection{Diverging Horizon Nuisance Fluctuations}\label{sec:app-growing-oracle}

For $i\in I_k$, define the fold-specific filtration
\[
\mathcal F_{i,t}^{(-k)}
\coloneqq
\sigma\bigl(\mathcal D_{-k},
Z_{i,0},\ldots,Z_{i,t-1},S_{i,t},A_{i,t}\bigr).
\]
This filtration includes the training data outside fold $k$, so the fold-specific nuisance estimators are measurable at each evaluation step.

For the score perturbation $\Delta_{\beta,-k}$ from Lemma~\ref{lem:if-remainder-bound}, define
\[
\bar\Delta_{\beta,-k}(s,a)
\coloneqq
\EE_{\bar P_H}\!\left[
\Delta_{\beta,-k}(Z)
\mid
S=s,A=a,\mathcal D_{-k}
\right].
\]
The generic transition $Z$ is evaluated under the pooled one-step law and independently of the training sample used to construct the random integrand. Conditional on $\mathcal D_{-k}$, this is a fixed function of $(s,a)$. The common controlled reward and transition law makes the same conditional transition expectation valid at each evaluation time: time-varying behavior changes the distribution of $X_t=(S_t,A_t)$, not the conditional law of $(R,S')$ given $S,A$. Thus, for $i\in I_k$,
\[
\EE[\Delta_{\beta,-k}(Z_{i,t})\mid\mathcal F_{i,t}^{(-k)}]
=
\bar\Delta_{\beta,-k}(S_{i,t},A_{i,t}).
\]
For each time $t$, retain the time-specific centering $\mu_{t,-k}\coloneqq\EE[\bar\Delta_{\beta,-k}(S_{i,t},A_{i,t})\mid\mathcal D_{-k}]$.

\subsubsection{Discussion of Assumption~\ref{ass:behavior-process-mixing}}\label{sec:app:dis}

\begin{remark}
\label{rem:dobrushin-predictable-averaging}

Assumption~\ref{ass:behavior-process-mixing} is a property of the behavior process itself. It does not depend on $\beta$ or on the nuisance estimators. Conditional on $\mathcal D_{-k}$, the predictable drift $\bar\Delta_{\beta,-k}$ is a fixed bounded function of $X_t=(S_t,A_t)$. Lemma~\ref{lem:doeblin-predictable-residual-covariance} gives the resulting conditional variance bound directly.

\end{remark}

\begin{lemma}[Strong-mixing variance bound]
\label{lem:doeblin-predictable-residual-covariance}
Suppose Assumption~\ref{ass:behavior-process-mixing} holds. Let $f_H$ be a bounded function of $X_t=(S_t,A_t)$, possibly depending on $H$, and set $\mu_{t,H}=\EE[f_H(X_t)]$. Then
\[
\EE\!\left[
\left(
\frac1H\sum_{t=0}^{H-1}(f_H(X_t)-\mu_{t,H})
\right)^2
\right]
\le
\frac{C}{H}\|f_H\|_\infty^2.
\]
If Assumption~\ref{ass:pooled-support} also holds, the right-hand side is bounded by $CH^{-1}\|f_H\|_{\bar\mu_H,2}^2$. Both conclusions also hold conditionally when $f_H$ is measurable with respect to training data $\mathcal D_{-k}$ that are independent of the held-out trajectory, with $\mu_{t,H}$ replaced by $\EE[f_H(X_t)\mid\mathcal D_{-k}]$.
\end{lemma}

\begin{proof}
For $0\le t<u\le H-1$, the bounded-variable strong-mixing covariance inequality, see \citet{bradley2005basic}, gives
\[
\left|
\Cov\left(
f_H(X_t),f_H(X_u)
\right)
\right|
\le
C\alpha_H^b(u-t)
\|f_H\|_\infty^2.
\]
Expanding the variance, bounding the diagonal terms by $\|f_H\|_\infty^2$, and re-indexing the off-diagonal terms by the lag $m=u-t$ give
\[
\begin{aligned}
\Var\left(
\frac1H\sum_{t=0}^{H-1}f_H(X_t)
\right)
&\le
\frac{C\|f_H\|_\infty^2}{H^2}
\left[
H+
\sum_{m=1}^{H-1}(H-m)\alpha_H^b(m)
\right]\\
&\le
\frac{C}{H}\|f_H\|_\infty^2,
\end{aligned}
\]
where the final inequality uses the uniform summability in Assumption~\ref{ass:behavior-process-mixing}. Centering each summand by its possibly time-varying mean does not change this variance.

Under Assumption~\ref{ass:pooled-support}, $\bar\mu_{H,\min}\coloneqq\min_x\bar\mu_H(x)$ is bounded away from zero uniformly in $H$. Finite-state norm equivalence gives $\|f_H\|_\infty^2\le\bar\mu_{H,\min}^{-1}\|f_H\|_{\bar\mu_H,2}^2$, which proves the scaled form. Finally, condition on $\mathcal D_{-k}$. The function $f_H$ is then deterministic, and independence of the held-out trajectory from $\mathcal D_{-k}$ leaves its law and strong-mixing coefficients unchanged. Applying the preceding calculation conditionally proves the last assertion.
\end{proof}

\subsubsection{Cross-Fitted Diverging-Horizon Empirical Fluctuation}

\begin{lemma}
\label{lem:growing-empirical-decomposition}
Define the martingale component
\[
E_{N,H}^{\mathrm{mart}}
\coloneqq
\sum_{k=1}^K\frac1N\sum_{i\in I_k}\frac1H\sum_{t=0}^{H-1}
\Big[
\varphi_\beta(Z_{i,t};\widehat\eta_{-k})
-\varphi_\beta(Z_{i,t})
-
\EE(\varphi_\beta(Z_{i,t};\widehat\eta_{-k})
-\varphi_\beta(Z_{i,t})
\mid\mathcal F_{i,t}^{(-k)})
\Big],
\]
and the predictable component is defined by
\[
E_{N,H}^{\mathrm{pred}}
\coloneqq
\sum_{k=1}^K\frac1N\sum_{i\in I_k}\frac1H\sum_{t=0}^{H-1}
\Big[
\EE(\varphi_\beta(Z_{i,t};\widehat\eta_{-k})
-\varphi_\beta(Z_{i,t})
\mid\mathcal F_{i,t}^{(-k)})
-
\EE(\varphi_\beta(Z_{i,t};\widehat\eta_{-k})
-\varphi_\beta(Z_{i,t})
\mid\mathcal D_{-k})
\Big].
\]
Then the empirical nuisance fluctuation admits the decomposition
\[
E_{N,H}=E_{N,H}^{\mathrm{mart}}+E_{N,H}^{\mathrm{pred}}.
\]
\end{lemma}

\begin{proof} 
Let $\bar m_{-k}\coloneqq H^{-1}\sum_{t=0}^{H-1}\mu_{t,-k}$. By Lemma~\ref{lem:if-decomposition},
\[
E_{N,H}
=
\sum_{k=1}^K\frac1N\sum_{i\in I_k}\frac1H\sum_{t=0}^{H-1}
\left[
\Delta_{\beta,-k}(Z_{i,t})
-
\EE_{\bar P_H}(\Delta_{\beta,-k}(Z)\mid\mathcal D_{-k})
\right].
\]
Because $\Delta_{\beta,-k}$ is fixed conditional on $\mathcal D_{-k}$, and because for each $t$ the evaluation trajectories in fold $k$ have the same time-specific behavior law, the pooled one-step law gives
\[
\EE_{\bar P_H}(\Delta_{\beta,-k}(Z)\mid\mathcal D_{-k})
=
\frac1H\sum_{t=0}^{H-1}
\EE(\Delta_{\beta,-k}(Z_{i,t})\mid\mathcal D_{-k})
=
\bar m_{-k}.
\]
Substituting this pooled mean, and then using $H^{-1}\sum_{t=0}^{H-1}(\mu_{t,-k}-\bar m_{-k})=0$, we obtain the two equivalent forms
\[
\begin{aligned}
E_{N,H}
&= \sum_{k=1}^K\frac1N\sum_{i\in I_k}\frac1H\sum_{t=0}^{H-1}(\Delta_{\beta,-k}(Z_{i,t})-\bar m_{-k})  = \sum_{k=1}^K\frac1N\sum_{i\in I_k}\frac1H\sum_{t=0}^{H-1}(\Delta_{\beta,-k}(Z_{i,t})-\mu_{t,-k}).
\end{aligned}
\]
The predictable-drift identity above and the tower property give $\EE[\Delta_{\beta,-k}(Z_{i,t})\mid\mathcal D_{-k}]=\mu_{t,-k}$. This centered term therefore decomposes as
\[
\Delta_{\beta,-k}(Z_{i,t})-\mu_{t,-k}
=
\Big[
\Delta_{\beta,-k}(Z_{i,t})
-
\bar\Delta_{\beta,-k}(S_{i,t},A_{i,t})
\Big]
+
\Big[
\bar\Delta_{\beta,-k}(S_{i,t},A_{i,t})
-
\mu_{t,-k}
\Big].
\]
The first bracket has conditional mean zero given the current state-action pair.
The second bracket is the centered predictable drift along the behavior trajectory. Averaging this decomposition yields $E_{N,H}=E_{N,H}^{\mathrm{mart}}+E_{N,H}^{\mathrm{pred}}$, with $E_{N,H}^{\mathrm{mart}}$ and $E_{N,H}^{\mathrm{pred}}$ as stated in the lemma.
\end{proof}

\begin{proposition}
\label{prop:growing-empirical-fluctuation-refined}
Suppose Assumptions~\ref{ass:uniform-signed-aggregation-identification}, \ref{ass:pooled-support}, \ref{ass:if-envelope}, \ref{ass:if-consistency}, and~\ref{ass:if-growing-rate} hold.
The martingale fluctuation satisfies
\[
E_{N,H}^{\mathrm{mart}}=o_p((NH)^{-1/2}).
\]
If, in addition, Assumption~\ref{ass:behavior-process-mixing} holds, then
\[
E_{N,H}^{\mathrm{pred}}=o_p((NH)^{-1/2}),
\qquad
E_{N,H}=o_p((NH)^{-1/2}).
\]
\end{proposition}

\begin{proof} 
Because the first term in $b_k$ is nonnegative, Assumption~\ref{ass:if-growing-rate} gives $\max_k\beta\|\delta Q_{-k}\|_{\bar\nu_H,2}^2\le\max_k b_k=o_p((NH)^{-1/2})=o_p(1)$. Thus Assumptions~\ref{ass:pooled-support}, \ref{ass:if-envelope}, \ref{ass:if-consistency}, and~\ref{ass:if-growing-rate} verify the conditions of Lemma~\ref{lem:empirical-score-replacement}, which gives $\max_k\|\Delta_{\beta,-k}\|_{\bar P_H,2}=o_p(1)$.
Conditional Jensen's inequality and the definition of the predictable drift then give
\[
\max_{1\le k\le K}\|\bar\Delta_{\beta,-k}\|_{\bar P_H,2}
\le
\max_{1\le k\le K}\|\Delta_{\beta,-k}\|_{\bar P_H,2}
=o_p(1).
\]
The final equality follows from Lemma~\ref{lem:score-perturbation-bound}, Assumption~\ref{ass:if-consistency}, and the nonlinear rate supplied by Assumption~\ref{ass:if-growing-rate}.

We first control the martingale part. Write $U_{i,t,k}\coloneqq\Delta_{\beta,-k}(Z_{i,t})-\EE(\Delta_{\beta,-k}(Z_{i,t})\mid\mathcal F_{i,t}^{(-k)})$. Conditional on $\mathcal D_{-k}$, the evaluation trajectories are independent over $i\in I_k$. Moreover, for $0\le t<u\le H-1$, $U_{i,t,k}$ is $\mathcal F_{i,u}^{(-k)}$-measurable and $\EE\!\left[U_{i,u,k}\mid\mathcal F_{i,u}^{(-k)}\right]=0$, so by iterated expectation $\EE\!\left[U_{i,t,k}U_{i,u,k}\mid\mathcal D_{-k}\right]=0$ for $0\le t<u\le H-1$. This orthogonality yields
\[
\EE\!\left[
\left(
\frac1H\sum_{t=0}^{H-1} U_{i,t,k}
\right)^2
\Bigm|\mathcal D_{-k}
\right]
\le
\frac1{H^2}
\sum_{t=0}^{H-1}
\EE\!\left[
U_{i,t,k}^2
\mid
\mathcal D_{-k}
\right].
\]
By conditional Jensen's inequality, $\EE[U_{i,t,k}^2\mid\mathcal D_{-k}]\le\EE[\Delta_{\beta,-k}(Z_{i,t})^2\mid\mathcal D_{-k}]$, so
\[
\EE\!\left[
\left(
\frac1H\sum_{t=0}^{H-1} U_{i,t,k}
\right)^2
\Bigm|\mathcal D_{-k}
\right]
\le
\frac{C}{H}\|\Delta_{\beta,-k}\|_{\bar P_H,2}^2 .
\]
Write $A^{\mathrm{mart}}_{k,N,H}\coloneqq\frac1{|I_k|}\sum_{i\in I_k}\frac1H\sum_{t=0}^{H-1} U_{i,t,k}$. Since the evaluation trajectories are independent over $i\in I_k$,
\[
\EE\!\left[
(A^{\mathrm{mart}}_{k,N,H})^2
\Bigm|\mathcal D_{-k}
\right]
\le
\frac{C}{|I_k|H}\|\Delta_{\beta,-k}\|_{\bar P_H,2}^2 .
\]
Because $|I_k|\asymp N$ and $\|\Delta_{\beta,-k}\|_{\bar P_H,2}=o_p(1)$ for every fixed $k$, conditional Chebyshev's inequality gives $A^{\mathrm{mart}}_{k,N,H}=o_p((NH)^{-1/2})$. Since $K$ is fixed, a union bound yields $\max_{1\le k\le K}|A^{\mathrm{mart}}_{k,N,H}|=o_p((NH)^{-1/2})$. Moreover, $E_{N,H}^{\mathrm{mart}}=\sum_{k=1}^K\frac{|I_k|}{N}A^{\mathrm{mart}}_{k,N,H}$, and since the weights $|I_k|/N$ sum to one, it follows that $E_{N,H}^{\mathrm{mart}}=o_p((NH)^{-1/2})$.

It remains to control the predictable part. Lemma~\ref{lem:growing-empirical-decomposition} supplies the conditional mean identity
\[
\EE(\Delta_{\beta,-k}(Z_{i,t})\mid\mathcal F_{i,t}^{(-k)})
=
\bar\Delta_{\beta,-k}(S_{i,t},A_{i,t}).
\]
Writing $\mu_{t,-k}\coloneqq\EE(\bar\Delta_{\beta,-k}(S_{i,t},A_{i,t})\mid\mathcal D_{-k})$, that lemma then gives
\[
E_{N,H}^{\mathrm{pred}}
=
\sum_{k=1}^K\frac1N\sum_{i\in I_k}\frac1H\sum_{t=0}^{H-1}
(\bar\Delta_{\beta,-k}(S_{i,t},A_{i,t})-\mu_{t,-k}).
\]
Write $Y_{i,k}\coloneqq\frac1H\sum_{t=0}^{H-1}(\bar\Delta_{\beta,-k}(S_{i,t},A_{i,t})-\mu_{t,-k})$. Given $\mathcal D_{-k}$, the variables $\{Y_{i,k}:i\in I_k\}$ are independent over $i\in I_k$ and centered. Applying the conditional, $L_2(\bar\mu_H)$-scaled form of Lemma~\ref{lem:doeblin-predictable-residual-covariance} with $f_H=\bar\Delta_{\beta,-k}$ gives $\EE[Y_{i,k}^2\mid\mathcal D_{-k}]\le C\|\bar\Delta_{\beta,-k}\|_{\bar P_H,2}^2/H$. Write $A^{\mathrm{pred}}_{k,N,H}\coloneqq\frac1{|I_k|}\sum_{i\in I_k}Y_{i,k}$. Since the evaluation trajectories are independent over $i\in I_k$,
\[
\EE\!\left[
(A^{\mathrm{pred}}_{k,N,H})^2
\Bigm|\mathcal D_{-k}
\right]
\le
\frac{C}{|I_k|H}\|\bar\Delta_{\beta,-k}\|_{\bar P_H,2}^2 .
\]
Because $|I_k|\asymp N$, conditional Chebyshev's inequality gives $A^{\mathrm{pred}}_{k,N,H}=o_p((NH)^{-1/2})$ for every fixed $k$. Since $K$ is fixed, a union bound yields $\max_{1\le k\le K}|A^{\mathrm{pred}}_{k,N,H}|=o_p((NH)^{-1/2})$. Moreover, $E_{N,H}^{\mathrm{pred}}=\sum_{k=1}^K\frac{|I_k|}{N}A^{\mathrm{pred}}_{k,N,H}$, and since the weights $|I_k|/N$ sum to one, it follows that $E_{N,H}^{\mathrm{pred}}=o_p((NH)^{-1/2})$. Finally, Lemma~\ref{lem:growing-empirical-decomposition} gives $E_{N,H}=E_{N,H}^{\mathrm{mart}}+E_{N,H}^{\mathrm{pred}}$, and combining the two bounds yields $E_{N,H}=o_p((NH)^{-1/2})$.
\end{proof}

\subsection{Oracle Score and Variance}\label{sec:app-martingale-qv}

Order the transition indices lexicographically: $(i,t)\prec(j,u)$ if $i<j$, or if $i=j$ and $t<u$. For a fixed trajectory $i$, define $
\mathcal F_{i,-1}
\coloneqq
\sigma\!\left(
\{Z_{j,u}:j<i,\ 0\le u\le H-1\},\,
S_{i,0},A_{i,0}
\right), $
and, for $0\le t\le H-1$, $
\mathcal F_{i,t}
\coloneqq
\sigma\!\left(
\{Z_{j,u}:j<i,\ 0\le u\le H-1\},\,
Z_{i,0},\ldots,Z_{i,t},\,
S_{i,t+1},A_{i,t+1}
\right),$
with the extra look-ahead action $A_{i,H}$ omitted when $t=H-1$, since the terminal next state $S_{i,H}$ is already included in $Z_{i,H-1}$. Thus $\mathcal F_{i,t-1}$ contains the current state-action pair $(S_{i,t},A_{i,t})$ but not the completed transition $Z_{i,t}$. Across the boundary between trajectories, the next initial pair is included in $\mathcal F_{i+1,-1}$, and independence of trajectories preserves the same conditional transition law.

To view these fields as a standard single-index martingale array, insert a zero increment when each initial state-action pair is revealed. The ordered array first reveals $(S_{1,0},A_{1,0})$ with increment zero, then reveals transition $(1,0)$ with increment $\xi_{1,0}$, continues through $\xi_{1,H-1}$, reveals $(S_{2,0},A_{2,0})$ with increment zero, and so on. The preceding field for every nonzero increment $\xi_{i,t}$ is $\mathcal F_{i,t-1}$, and the current field is $\mathcal F_{i,t}$. The inserted zero increments make the boundary nesting explicit and leave unchanged the sum of the original $\xi_{i,t}$'s, the predictable quadratic variation, the realized quadratic variation, and the conditional Lindeberg sum. The normalization remains $\sqrt{NH}$, since there are still $NH$ nonzero transition contributions.

\begin{lemma}\label{lem:mds-true-score-active}
With respect to the filtration just defined,
\[
\EE[\xi_{i,t}\mid\mathcal F_{i,t-1}]=0,
\qquad
\xi_{i,t}\in\mathcal F_{i,t},
\qquad
1\le i\le N,\quad 0\le t\le H-1.
\]
Consequently, after adding the zero boundary increments just described, the lexicographically ordered nonzero increments form a martingale-difference array with preceding field $\mathcal F_{i,t-1}$ and current field $\mathcal F_{i,t}$.
\end{lemma}

\begin{proof}
By construction, $(S_{i,t},A_{i,t})$ is $\mathcal F_{i,t-1}$-measurable, while the reward and next state in $Z_{i,t}$ are not yet revealed. The Markov transition and reward laws, together with independence across trajectories at the boundary $t=-1$, give
\[
\EE\!\left[
R_{i,t+1}-J_\beta+V_\beta(Q_\beta)(S_{i,t+1})
-Q_\beta(S_{i,t},A_{i,t})
\mid \mathcal F_{i,t-1}
\right]
=0
\]
by the self-induced Bellman equation \eqref{eq:self}. Multiplying by the $\mathcal F_{i,t-1}$-measurable factor $\lambda_\beta(S_{i,t},A_{i,t})$ gives $\EE[\xi_{i,t}\mid\mathcal F_{i,t-1}]=0$. The measurability statement follows because $\mathcal F_{i,t}$ contains the completed transition $Z_{i,t}$.
\end{proof}

We next prove convergence of the realized average quadratic variation. The following two lemmas establish stabilization of the predictable quadratic variation and control the difference between the realized and predictable quadratic variations. Proposition~\ref{prop:rqv-sufficient-route} then combines these results.

\begin{lemma}[Predictable quadratic variation]
\label{lem:behavior-contraction-pqv}
Suppose Assumptions~\ref{ass:if-envelope}, \ref{ass:irreducible}, \ref{ass:uniform-signed-aggregation-identification}, \ref{ass:pooled-support}, and~\ref{ass:behavior-process-mixing} hold.
Then we have
\[
\frac1{NH}\sum_{i=1}^N\sum_{t=0}^{H-1}
\EE[\xi_{i,t}^2\mid\mathcal F_{i,t-1}]
-
\sigma^2
=
O_p((NH)^{-1/2}).
\]
Thus the predictable quadratic variation stabilizes at $\sigma^2$.
\end{lemma}

\begin{proof}
Define $g_\beta(s,a)\coloneqq\EE[\xi_{i,t}^2\mid S_{i,t}=s,A_{i,t}=a]$. This function is time-homogeneous because the transition and reward laws, and $(J_\beta,Q_\beta,\lambda_\beta)$, are time-homogeneous. Assumption~\ref{ass:if-envelope} together with Corollaries~\ref{cor:app-smoothed-target-bounded} and~\ref{cor:target-adjoint-bound} gives $\|\xi_{i,t}\|_\infty\le C$ uniformly in $i,t,N,H,\beta$, and hence $\|g_\beta\|_\infty\le C$. Moreover, $\EE[\xi_{i,t}^2\mid\mathcal F_{i,t-1}]=g_\beta(S_{i,t},A_{i,t})$.

Applying Lemma~\ref{lem:doeblin-predictable-residual-covariance} with $f_H=g_\beta$ gives
\[
\Var\left(
\frac1H
\sum_{t=0}^{H-1}
g_\beta(X_{i,t})
\right)
\le \frac{C}{H}.
\]
The lemma applies uniformly because $g_\beta$ is a bounded function of the behavior-process state-action pair. Independence across trajectories then gives
\[
\Var\left(
\frac1{NH}
\sum_{i=1}^N
\sum_{t=0}^{H-1}
g_\beta(X_{i,t})
\right)
\le
\frac{C}{NH}.
\]
The expectation of this average is exactly $H^{-1}\sum_{t=0}^{H-1}\EE[\xi_{i,t}^2]=\sigma^2$.
Chebyshev's inequality proves the claimed $O_p((NH)^{-1/2})$ bound.

\end{proof}

\begin{lemma}[Realized quadratic variation]
\label{lem:pqv-to-rqv-comparison}
Under Assumption~\ref{ass:if-envelope} together with Assumptions~\ref{ass:irreducible}, \ref{ass:uniform-signed-aggregation-identification}, and~\ref{ass:pooled-support},
\[
\frac1{NH}\sum_{i=1}^N\sum_{t=0}^{H-1}
\left(
\xi_{i,t}^2
-
\EE[\xi_{i,t}^2\mid\mathcal F_{i,t-1}]
\right)
=
O_p\!\left((NH)^{-1/2}\right).
\]
\end{lemma}

\begin{proof}
Define $D_{i,t}\coloneqq\xi_{i,t}^2-\EE[\xi_{i,t}^2\mid\mathcal F_{i,t-1}]$. By Lemma~\ref{lem:mds-true-score-active}, $D_{i,t}\in\mathcal F_{i,t}$. By construction, $\EE[D_{i,t}\mid\mathcal F_{i,t-1}]=0$. Thus, with zero increments inserted at trajectory boundaries as above, the lexicographically ordered $D_{i,t}$'s form a bounded martingale-difference array with respect to the same shifted filtration.
The boundedness follows from Assumption~\ref{ass:if-envelope} together with Corollaries~\ref{cor:app-smoothed-target-bounded} and~\ref{cor:target-adjoint-bound}, which gives $\sup_{i,t,N,H,\beta}|\xi_{i,t}|\le C$. Martingale orthogonality then gives
\[
\EE\left[
\left(
\frac1{NH}\sum_{i=1}^N\sum_{t=0}^{H-1}D_{i,t}
\right)^{2}
\right]
\le
\frac{C}{NH}.
\]
The displayed $O_p((NH)^{-1/2})$ bound follows by Chebyshev's inequality.
\end{proof}

\begin{proposition}[Realized quadratic variation]
\label{prop:rqv-sufficient-route}
Under Assumptions~\ref{ass:irreducible}, \ref{ass:uniform-signed-aggregation-identification}, \ref{ass:pooled-support}, \ref{ass:if-envelope}, and~\ref{ass:behavior-process-mixing},
\[
(NH)^{-1}\sum_{i=1}^N\sum_{t=0}^{H-1}\xi_{i,t}^2
-\sigma^2
\xrightarrow{p}0
\]
holds as $NH\to\infty$.

\end{proposition}

\begin{proof}
Decompose
\begin{align*}
\frac1{NH}\sum_{i=1}^N\sum_{t=0}^{H-1}\xi_{i,t}^2
-
\sigma^2
&= \frac1{NH}\sum_{i=1}^N\sum_{t=0}^{H-1}\left(\xi_{i,t}^2-\EE[\xi_{i,t}^2\mid\mathcal F_{i,t-1}]\right)  +\left(\frac1{NH}\sum_{i=1}^N\sum_{t=0}^{H-1}\EE[\xi_{i,t}^2\mid\mathcal F_{i,t-1}]-\sigma^2\right).
\end{align*}
Lemma~\ref{lem:pqv-to-rqv-comparison} gives the first bracket as $o_p(1)$, and Lemma~\ref{lem:behavior-contraction-pqv} gives the second bracket as $o_p(1)$. Hence the realized average quadratic variation converges to $\sigma^2$.
\end{proof}

\begin{proposition}[Bounds for $\sigma^2$]\label{prop:bellman-innovation-variance-lower-bounds}
Suppose Assumptions~\ref{ass:irreducible}, \ref{ass:uniform-signed-aggregation-identification}, \ref{ass:pooled-support}, \ref{ass:if-envelope}, and~\ref{ass:optimal-bellman-innovation-nondegenerate} hold. Then there exist constants $0<c<C<\infty$, independent of $N$, $H$, and $\beta$, such that $c\le \sigma^2\le C$ whenever $\beta\ge C/v_{\min}^*$, uniformly over the relevant horizons.
\end{proposition}

\begin{proof}
Set $
v_\beta(s,a)
\coloneqq
\Var(R+V_\beta(Q_\beta)(S')\mid S=s,A=a), 
v^*(s,a)
\coloneqq
\Var(R+V^*(S')\mid S=s,A=a).$
By \eqref{eq:large-beta-target}, $\|V_\beta(Q_\beta)-V^*\|_\infty\le Ce^{-\beta\Gamma}$. Fix $(s,a)$, and write $X\coloneqq R+V^*(S')$ and $Y\coloneqq V_\beta(Q_\beta)(S')-V^*(S')$, conditionally on $S=s,A=a$. Assumption~\ref{ass:if-envelope}, finite support, and Corollary~\ref{cor:app-smoothed-target-bounded} give $|X|\le C$, while $|Y|\le Ce^{-\beta\Gamma}$. The identity $\Var(X+Y\mid s,a)-\Var(X\mid s,a)=2\Cov(X,Y\mid s,a)+\Var(Y\mid s,a)$ and Cauchy--Schwarz yield
\[
\|v_\beta-v^*\|_\infty
\le
C\|V_\beta(Q_\beta)-V^*\|_\infty
\le
Ce^{-\beta\Gamma}.
\]
When $\Gamma<\infty$, the inequality $e^{-x}\le x^{-1}$ for $x>0$ gives $\|v_\beta-v^*\|_\infty\le C/\beta$, after absorbing the fixed MDP quantity $\Gamma^{-1}$ into $C$. When $\Gamma=+\infty$, \eqref{eq:large-beta-target} gives $V_\beta(Q_\beta)=V^*$, so $v_\beta=v^*$ exactly. The nondegeneracy assumption therefore implies
\[
\min_{(s,a)}v_\beta(s,a)
\ge
\frac12
\min_{(s,a)}v^*(s,a)
>0
\]
under the stated threshold on $\beta$.

By \eqref{eq:self}, the Bellman residual inside $\xi_{i,t}$ has conditional mean zero given $(S_{i,t},A_{i,t})$. Hence
\[
\sigma^2
=
\EE_{\bar P_H}\!\left[
\lambda_\beta(S,A)^2v_\beta(S,A)
\right].
\]
The adjoint normalization gives $\EE_{\bar P_H}[\lambda_\beta(S,A)]=1$, so Jensen's inequality yields $\EE_{\bar P_H}[\lambda_\beta(S,A)^2]\ge1$, and the preceding lower bound on $v_\beta$ gives $\sigma^2\ge c$. For the upper bound, Assumption~\ref{ass:if-envelope}, Corollary~\ref{cor:app-smoothed-target-bounded}, and Corollary~\ref{cor:target-adjoint-bound} give $v_\beta\le C$ and $|\lambda_\beta|\le C$ uniformly, so the same display gives $\sigma^2\le C$.
\end{proof}

\subsection{Proof of the Inference Theorem}\label{sec:app-unified-inference}

\begin{proof}[Proof of Theorem~\ref{thm:clt-N}]
By Lemma~\ref{lem:if-decomposition}, $\widehat J_\beta-J_\beta=L_{N,H}+B_{N,H}+E_{N,H}$.
Under the fixed-$H$ regime, Proposition~\ref{prop:population-bias-bound} gives $B_{N,H}=o_p(N^{-1/2})$. To control $E_{N,H}$, write
\[
\Delta_{\beta,-k}(Z)
\coloneqq
\varphi_\beta(Z;\widehat\eta_{-k})-\varphi_\beta(Z).
\]
By Lemma~\ref{lem:empirical-score-replacement} and the fixed-$H$ rate and consistency assumptions, $\max_{1\le k\le K}\|\Delta_{\beta,-k}\|_{\bar P_H,2}=o_p(1)$. For each fold $k$, define
\[
A_{k,N}
\coloneqq
\frac1{|I_k|}
\sum_{i\in I_k}
\left[
\frac1H
\sum_{t=0}^{H-1}
\Delta_{\beta,-k}(Z_{i,t})
-
\EE_{\bar P_H}(\Delta_{\beta,-k}(Z)\mid\mathcal D_{-k})
\right].
\]
Conditionally on $\mathcal D_{-k}$, the summands in $A_{k,N}$ are independent and centered across $i\in I_k$. Jensen's inequality and independence of the held-out trajectories give
\[
\begin{aligned}
\EE[A_{k,N}^2\mid\mathcal D_{-k}]
&\le \frac{1}{|I_k|^2}
\sum_{i\in I_k}
\EE\!\left[
\left(
\frac1H\sum_{t=0}^{H-1}\Delta_{\beta,-k}(Z_{i,t})
\right)^2
\middle|\mathcal D_{-k}
\right] \\
&\le \frac{1}{|I_k|^2H}
\sum_{i\in I_k}\sum_{t=0}^{H-1}
\EE[\Delta_{\beta,-k}(Z_{i,t})^2\mid\mathcal D_{-k}]
=
\frac{\|\Delta_{\beta,-k}\|_{\bar P_H,2}^2}{|I_k|}
=o_p(N^{-1}).
\end{aligned}
\]
Chebyshev's inequality gives $A_{k,N}=o_p(N^{-1/2})$ for each fixed $k$, uniformly over the fixed number of folds. The decomposition in Lemma~\ref{lem:if-decomposition} then gives $E_{N,H}=\sum_{k=1}^K(|I_k|/N)A_{k,N}=o_p(N^{-1/2})$. Since $H$ is fixed and Proposition~\ref{prop:bellman-innovation-variance-lower-bounds} gives $\sigma\ge c$ eventually, these remainders are also $o_p(\sigma/\sqrt{NH})$.
Define $U_i\coloneqq H^{-1}\sum_{t=0}^{H-1}\xi_{i,t}$ only inside this proof.
Lemma~\ref{lem:mds-true-score-active} gives $\EE[U_i]=0$, and martingale orthogonality gives $\EE[\xi_{i,t}\xi_{i,u}]=0$ for $t\ne u$. Hence the variance is
\[
\Var(U_i)
=
\frac1{H^2}\sum_{t=0}^{H-1}\EE[\xi_{i,t}^2]
=
\frac{\sigma^2}{H}.
\]
Assumption~\ref{ass:if-envelope} together with Corollaries~\ref{cor:app-smoothed-target-bounded} and~\ref{cor:target-adjoint-bound} gives $|\xi_{i,t}|\le C$. Hence $|U_i|\le C$.
The $U_i$'s are independent across trajectories and uniformly bounded, so the Lindeberg triangular-array CLT gives
\[
\frac{\sqrt{NH}\,L_{N,H}}{\sigma}
=
\frac1{\sqrt N}\sum_{i=1}^N \frac{U_i}{\sigma/\sqrt H}
\xrightarrow{d}N(0,1).
\]
Slutsky's theorem then gives the smoothed value conclusion in the fixed-$H$ regime.

Under the diverging-$H$ regime, Propositions~\ref{prop:population-bias-bound} and~\ref{prop:growing-empirical-fluctuation-refined} give $B_{N,H}=o_p((NH)^{-1/2})$ and $E_{N,H}=o_p((NH)^{-1/2})$. Proposition~\ref{prop:bellman-innovation-variance-lower-bounds} again gives $\sigma\ge c$, so these remainders are $o_p(\sigma/\sqrt{NH})$.
Set $M_{i,t}\coloneqq\xi_{i,t}/(\sqrt{NH}\,\sigma)$.
Lemma~\ref{lem:mds-true-score-active} gives the martingale-difference property for the augmented lexicographic array with zero boundary increments. Assumption~\ref{ass:if-envelope} together with Corollaries~\ref{cor:app-smoothed-target-bounded} and~\ref{cor:target-adjoint-bound} gives $|\xi_{i,t}|\le C$. Combining this bound with $\sigma\ge c$ gives $\max_{i,t}|M_{i,t}|\le C/\sqrt{NH}\to0$ and $\EE[\max_{i,t}M_{i,t}^2]\le C/(NH)\to0$. The inserted zero increments do not affect these maxima or any of the following sums. Proposition~\ref{prop:rqv-sufficient-route} gives
\[
\sum_{i=1}^N\sum_{t=0}^{H-1}
M_{i,t}^2
=
\frac{(NH)^{-1}\sum_{i=1}^N\sum_{t=0}^{H-1}
\xi_{i,t}^2}{\sigma^2}
\xrightarrow{p}1.
\]
The division by $\sigma^2$ is legitimate uniformly because Proposition~\ref{prop:bellman-innovation-variance-lower-bounds} gives $\sigma^2\ge c$ eventually. McLeish's martingale central limit theorem \citep[Theorem~2.3]{mcleish1974dependent} therefore yields
\[
\frac{\sqrt{NH}\,L_{N,H}}{\sigma}
=
\sum_{i=1}^N\sum_{t=0}^{H-1}M_{i,t}
\xrightarrow{d}N(0,1),
\]
and Slutsky's theorem gives the smoothed value conclusion in the diverging-$H$ regime.

For both regimes, Theorem~\ref{thm:exp} gives $|J_\beta-J^*|\le Ce^{-\beta\Gamma}$. If $\sqrt{NH}e^{-\beta\Gamma}\to0$, then $\sqrt{NH}(J_\beta-J^*)/\sigma\to0$, because $\sigma\ge c$. Another application of Slutsky's theorem proves the optimal value conclusion.
\end{proof}

\subsection{Efficiency and Convergence to Uniform Optimal-Policy Evaluation}
\label{sec:app-efficiency-pe-limit}

\subsubsection{Convergence to Uniform Optimal-Policy Evaluation}

Proposition~\ref{prop:bellman-innovation-variance-lower-bounds} is the variance input for Theorem~\ref{thm:clt-N}. We also use the derivative operator convergence in Proposition~\ref{prop:large-beta-linearization} to connect the adjoint weight and variance to ordinary fixed-policy evaluation under the uniform policy over maximizing actions.
We make the horizon dependence explicit in this section and write $\lambda_H^{\mathrm{unif}}$ and $\sigma_{\mathrm{PE},H}^2$.

\begin{proposition}
\label{prop:large-beta-pe-limit}
Suppose Assumptions~\ref{ass:irreducible}, \ref{ass:uniform-signed-aggregation-identification}, \ref{ass:pooled-support}, and~\ref{ass:if-envelope} hold. Then, for each horizon $H$, there is a unique $\lambda_H^{\mathrm{unif}}\in L_2(\bar\mu_H)$ satisfying
\[
\EE_{\bar P_H}[\lambda_H^{\mathrm{unif}}(S,A)]=1
\quad\text{and}\quad
\EE_{\bar P_H}\!\left[
\lambda_H^{\mathrm{unif}}(S,A)
\left(
q(S,A)-(D_{\pi^{\mathrm{unif}}}q)(S')
\right)
\right]
=0,
\qquad q\in\cQ.
\]
Define the ordinary fixed-policy-evaluation variance under $\pi^{\mathrm{unif}}$, in the same $\sqrt{NH}$ normalization, by
\[
\sigma_{\mathrm{PE},H}^2
\coloneqq
\EE_{\bar P_H}\!\left[
(\lambda_H^{\mathrm{unif}}(S,A))^2
\Var(R+V^*(S')\mid S,A)
\right].
\]
Define the smoothed one-step score by
\[
\xi_{\beta,H}(Z)
\coloneqq
\lambda_\beta(S,A)
(R-J_\beta+V_\beta(Q_\beta)(S')-Q_\beta(S,A)).
\]
Define the corresponding fixed-policy score by
\[
\xi_{\mathrm{unif},H}(Z)
\coloneqq
\lambda_H^{\mathrm{unif}}(S,A)
(R-J^*+V^*(S')-Q^*(S,A)).
\]
There exists $C<\infty$, independent of $\beta$, $N$, and $H$, such that, for every $\beta>0$, $
\|\lambda_\beta-\lambda_H^{\mathrm{unif}}\|_{\bar\mu_H,2}
\le
C(1+\beta)e^{-\beta\Gamma},$ $
\|\xi_{\beta,H}-\xi_{\mathrm{unif},H}\|_{\bar P_H,2}
\le
C(1+\beta)e^{-\beta\Gamma},$ 
and $\left|\sigma^2-\sigma_{\mathrm{PE},H}^2\right|
\le
C(1+\beta)e^{-\beta\Gamma},$
uniformly over the relevant horizons. If $\Gamma=+\infty$, then $\lambda_\beta=\lambda_H^{\mathrm{unif}}$, $\xi_{\beta,H}=\xi_{\mathrm{unif},H}$, and $\sigma^2=\sigma_{\mathrm{PE},H}^2$ exactly, for every $\beta>0$.
\end{proposition}

\begin{proof}
First, Proposition~\ref{prop:large-beta-linearization} shows that $D_{\pi^{\mathrm{unif}}}$ is allowable. The proof of Proposition~\ref{prop:adjoint-existence} uses only Assumption~\ref{ass:uniform-signed-aggregation-identification}, the full support supplied by Assumption~\ref{ass:pooled-support}, and allowability of the aggregation operator. Running that finite-dimensional argument with $D_\beta$ replaced by $D_{\pi^{\mathrm{unif}}}$ gives existence and uniqueness of $\lambda_H^{\mathrm{unif}}$. Applying Lemma~\ref{lem:uniform-dual-identification} with $D=D_{\pi^{\mathrm{unif}}}$ and $r=\lambda_H^{\mathrm{unif}}$ gives a uniform $L_2(\bar\mu_H)$ bound for $\lambda_H^{\mathrm{unif}}$. Together with Corollary~\ref{cor:target-adjoint-bound},
\begin{equation}
\|\lambda_\beta\|_{\bar\mu_H,2}
+
\|\lambda_H^{\mathrm{unif}}\|_{\bar\mu_H,2}
\le
C,
\label{eq:large-beta-adjoint-bounds}
\end{equation}
with $C$ uniform over the relevant horizons.

Next set $r\coloneqq\lambda_\beta-\lambda_H^{\mathrm{unif}}$. The two normalizations give $\EE_{\bar P_H}[r(S,A)]=0$. Subtracting the two adjoint equations gives, for every $q\in\cQ$,
\[
\EE_{\bar P_H}\!\left[
r(S,A)
\left(
q(S,A)-(D_{\pi^{\mathrm{unif}}}q)(S')
\right)
\right]
=
\EE_{\bar P_H}\!\left[
\lambda_\beta(S,A)
\bigl((D_\beta-D_{\pi^{\mathrm{unif}}})q\bigr)(S')
\right].
\]
For $\|q\|_{\bar\mu_H,2}\le1$, Assumption~\ref{ass:pooled-support}, Cauchy--Schwarz, \eqref{eq:large-beta-adjoint-bounds}, and Proposition~\ref{prop:large-beta-linearization} bound the absolute value of the right-hand side by $C(1+\beta)e^{-\beta\Gamma}$. Applying Lemma~\ref{lem:uniform-dual-identification} with $D=D_{\pi^{\mathrm{unif}}}$ and this $r$, and using the zero normalization of $r$, yields $\|\lambda_\beta-\lambda_H^{\mathrm{unif}}\|_{\bar\mu_H,2}\le C(1+\beta)e^{-\beta\Gamma}$.

We next compare the score increments. Adding and subtracting the residual with the limiting Bellman pair gives
\[
\begin{aligned}
\|\xi_{\beta,H}-\xi_{\mathrm{unif},H}\|_{\bar P_H,2}
&\le
C\|\lambda_\beta-\lambda_H^{\mathrm{unif}}\|_{\bar\mu_H,2} +
C\|\lambda_H^{\mathrm{unif}}\|_{\bar\mu_H,2}
(|J_\beta-J^*|+\|V_\beta(Q_\beta)-V^*\|_\infty+\|Q_\beta-Q^*\|_\infty)\\
&\le
C(1+\beta)e^{-\beta\Gamma}.
\end{aligned}
\]
The first inequality uses the uniform boundedness of both Bellman residuals and the adjoint weights. The second uses Propositions~\ref{prop:large-beta-approx-tie} and~\ref{prop:large-beta-policy}, together with the adjoint bound proved above.

Finally write $v_\beta(s,a)\coloneqq\Var(R+V_\beta(Q_\beta)(S')\mid S=s,A=a)$ and $v^*(s,a)\coloneqq\Var(R+V^*(S')\mid S=s,A=a)$. The variance-transfer calculation in Proposition~\ref{prop:bellman-innovation-variance-lower-bounds} gives
\begin{equation}
\|v_\beta-v^*\|_\infty
\le
Ce^{-\beta\Gamma}.
\label{eq:large-beta-variance-map}
\end{equation}
Using $\sigma^2=\EE_{\bar P_H}[\lambda_\beta(S,A)^2v_\beta(S,A)]$, add and subtract $\EE_{\bar P_H}[\lambda_\beta(S,A)^2v^*(S,A)]$. Then
\[
\left|\sigma^2-\sigma_{\mathrm{PE},H}^2\right|
\le
\|v_\beta-v^*\|_\infty
\EE_{\bar P_H}[\lambda_\beta(S,A)^2]
+
\|v^*\|_\infty
\EE_{\bar P_H}\!\left[
|\lambda_\beta-\lambda_H^{\mathrm{unif}}|
|\lambda_\beta+\lambda_H^{\mathrm{unif}}|
\right].
\]
Assumption~\ref{ass:if-envelope} gives $\|v^*\|_\infty\le C$. Cauchy--Schwarz, \eqref{eq:large-beta-adjoint-bounds}, the adjoint stability bound, and \eqref{eq:large-beta-variance-map} give $|\sigma^2-\sigma_{\mathrm{PE},H}^2|\le C(1+\beta)e^{-\beta\Gamma}$.

If $\Gamma=+\infty$, Proposition~\ref{prop:large-beta-linearization} gives $D_\beta=D_{\pi^{\mathrm{unif}}}$, and Propositions~\ref{prop:large-beta-approx-tie} and~\ref{prop:large-beta-policy} give $(Q_\beta,J_\beta,V_\beta(Q_\beta))=(Q^*,J^*,V^*)$. Uniqueness of the adjoint weight gives $\lambda_\beta=\lambda_H^{\mathrm{unif}}$, hence $\xi_{\beta,H}=\xi_{\mathrm{unif},H}$, and $v_\beta=v^*$ gives $\sigma^2=\sigma_{\mathrm{PE},H}^2$.
\end{proof}

Both $\xi_{\beta,H}(Z_{i,t})$ and $\xi_{\mathrm{unif},H}(Z_{i,t})$ have conditional mean zero given $(S_{i,t},A_{i,t})$. Their difference is therefore a martingale-difference sequence within each trajectory, and trajectories are independent. Consequently,
\[
\EE\!\left[
\left\{
\frac1{\sqrt{NH}}
\sum_{i=1}^N\sum_{t=0}^{H-1}
(\xi_{\beta,H}(Z_{i,t})-\xi_{\mathrm{unif},H}(Z_{i,t}))
\right\}^2
\right]
=
\|\xi_{\beta,H}-\xi_{\mathrm{unif},H}\|_{\bar P_H,2}^2.
\]
The score bound in Proposition~\ref{prop:large-beta-pe-limit} and Markov's inequality therefore give the corresponding normalized empirical-score replacement as $o_p(1)$ whenever $(1+\beta)e^{-\beta\Gamma}\to0$, without requiring uniqueness of the optimal policy.

\subsubsection{Efficient Influence Function for $J_\beta$}
In this section, write $\sigma_H^2$ for the horizon-specific variance denoted by $\sigma^2$ in Section~\ref{sec:if}.
\begin{proposition}[Efficiency for the smoothed target at fixed horizon]
Under the standing assumptions, for fixed $H$ and fixed $\beta$,
$$
\Psi_{\beta,H}(O)
=
\frac1H\sum_{t=0}^{H-1}\xi_{\beta,H}(Z_t)
$$
is the efficient influence function for $J_\beta$. Its variance is $\sigma_H^2/H$, and the proposed estimator attains the corresponding efficiency bound $\sigma_H^2/(NH)$ under the fixed-horizon conditions of Theorem~\ref{thm:clt-N}.
\end{proposition}

\begin{proof}
Fix $H$, and write one observed trajectory as
$$
O
=
(S_0,A_0,R_1,S_1,\ldots,A_{H-1},R_H,S_H).
$$
For this local calculation, consider the dominated nonparametric time-homogeneous Markov model whose trajectory law factors as
$$
p(O)
=
p_0(S_0)
\prod_{t=0}^{H-1}
\pi_t^b(A_t\mid S_t)
\mathcal K(R_{t+1},S_{t+1}\mid S_t,A_t).
$$
The initial-state law $p_0$ and the possibly time-varying behavior policy $\pi^b=\{\pi_t^b\}_{t=0}^{H-1}$ are nuisance components. The parameter $J_\beta$ depends only on the common environment law $\mathcal K$, and the calculation below is along regular dominated submodels of $\mathcal K$.

Let $\mathcal K_\epsilon$ be a regular conditional submodel through the true environment law, with one-step score $h(Z)$ satisfying $\EE[h(Z)\mid S,A]=0$. Let
$
\delta_\beta(Z)
\coloneqq
R-J_\beta+V_\beta(Q_\beta)(S')-Q_\beta(S,A).
$
For fixed $\beta$, the finite state and action spaces make the Bellman map smooth in the conditional environment law and in $Q$. Its linearization on $\cQ\times\RR$ is invertible under Assumption~\ref{ass:uniform-signed-aggregation-identification}, so the finite-dimensional implicit-function theorem gives differentiability of the anchored solution. If $(\dot J_\beta,\dot Q_\beta)$ denotes the derivative along the submodel, then differentiating \eqref{eq:self} gives
$$
\dot J_\beta\one
+
\dot Q_\beta
-
PD_\beta\dot Q_\beta
=
m_h,
\qquad
m_h(s,a)
=
\EE[\delta_\beta(Z)h(Z)\mid S=s,A=a].
$$
The sign follows from moving the derivative of the continuation term to the left-hand side. The anchoring condition is fixed, so $\dot Q_\beta\in\cQ$. Multiplying by $\lambda_\beta$ and averaging under $\bar P_H$ gives
$
\dot J_\beta(h)
=
\EE_{\bar P_H}[
\lambda_\beta(S,A)\delta_\beta(Z)h(Z)
],
$
because $\EE_{\bar P_H}[\lambda_\beta(S,A)]=1$ and the adjoint equation annihilates $\dot Q_\beta(S,A)-(D_\beta\dot Q_\beta)(S')$.

Consider the candidate trajectory-level influence function $\Psi_{\beta,H}(O)\coloneqq H^{-1}\sum_{t=0}^{H-1}\xi_{\beta,H}(Z_t)$.
The trajectory score for the environment submodel is $S_{\mathcal K,H}(O)=\sum_{t=0}^{H-1}h(Z_t)$. For the natural pre-transition filtration $\mathcal F_t^-=\sigma(S_0,A_0,R_1,\ldots,S_t,A_t)$, the conditional score centering gives $\EE[h(Z_t)\mid\mathcal F_t^-]=0$, and the Bellman equation gives $\EE[\xi_{\beta,H}(Z_t)\mid\mathcal F_t^-]=0$. Hence, if $t<u$, then $\xi_{\beta,H}(Z_t)$ is $\mathcal F_u^-$-measurable and $\EE[\xi_{\beta,H}(Z_t)h(Z_u)]=0$. If $u<t$, then $h(Z_u)$ is $\mathcal F_t^-$-measurable and $\EE[\xi_{\beta,H}(Z_t)h(Z_u)]=0$. Therefore,
$$
\EE[\Psi_{\beta,H}(O)S_{\mathcal K,H}(O)]
=
\frac1H
\sum_{t=0}^{H-1}
\EE[\xi_{\beta,H}(Z_t)h(Z_t)]
=
\dot J_\beta(h).
$$
The same conditional-mean-zero argument gives orthogonality to initial-law tangent scores and behavior-policy tangent scores. If $s_0(S_0)$ is an initial-law score, then it is measurable before every transition, so $\EE[s_0(S_0)\xi_{\beta,H}(Z_t)]=0$ for all $t$. If $s_u^b(S_u,A_u)$ is a behavior-policy score at time $u$, then $\EE[s_u^b(S_u,A_u)\mid S_u]=0$. For $t<u$, $\xi_{\beta,H}(Z_t)$ and $S_u$ are known before $A_u$ is drawn, giving $\EE[\xi_{\beta,H}(Z_t)s_u^b(S_u,A_u)]=0$. For $t\ge u$, the behavior score is $\mathcal F_t^-$-measurable, and $\EE[\xi_{\beta,H}(Z_t)\mid\mathcal F_t^-]=0$ gives the same conclusion.
Thus, for a generic joint tangent score
$
S_H(O)
=
s_0(S_0)
+
\sum_{t=0}^{H-1}s_t^b(S_t,A_t)
+
\sum_{t=0}^{H-1}h(Z_t),
$
we have
$
\EE[\Psi_{\beta,H}(O)S_H(O)]
=
\dot J_\beta(h),
$
which is the pathwise derivative along the joint submodel because $J_\beta$ depends only on the environment law.

It remains to check that the candidate influence function lies in the tangent space. Define $h_{\beta,H}^*(Z)\coloneqq H^{-1}\lambda_\beta(S,A)\delta_\beta(Z)$. The self-induced Bellman equation gives $\EE[h_{\beta,H}^*(Z)\mid S,A]=0$, and $h_{\beta,H}^*$ is bounded under the standing assumptions. Hence it is an admissible conditional environment score in the nonparametric model. Moreover, $\Psi_{\beta,H}(O)=\sum_{t=0}^{H-1}h_{\beta,H}^*(Z_t)$, so $\Psi_{\beta,H}$ belongs to the environment tangent space. Since it represents the pathwise derivative over the full model, it is the canonical gradient, equivalently the efficient influence function, for $J_\beta$. We have
$$
\Var(\Psi_{\beta,H}(O))
=
\frac1{H^2}
\sum_{t=0}^{H-1}
\EE[\xi_{\beta,H}(Z_t)^2]
=
\frac{\sigma_H^2}{H}.
$$
For $N$ independent trajectories, the corresponding efficiency bound is $\sigma_H^2/(NH)$. The fixed-$H$ part of the proof of Theorem~\ref{thm:clt-N} gives
$$
\widehat J_\beta-J_\beta
=
\frac1N
\sum_{i=1}^N
\Psi_{\beta,H}(O_i)
+
o_p(N^{-1/2}).
$$
Thus, under the usual local regularity of this asymptotic linear representation along the dominated submodels above, the estimator attains the semiparametric efficiency bound for $J_\beta$ in the fixed-$H$ regime.
\end{proof}

\subsubsection{Unique Optimum and Efficiency for $J^*$}

\begin{corollary}[Efficiency for $J^*$ under a unique optimum]
Suppose the Bellman-optimal action is unique at every state. For fixed $H$, $J^*$ is locally regular and has the same efficient influence function as the value of the fixed optimal policy $\pi^*$. If the fixed-horizon conditions of Theorem~\ref{thm:clt-N} hold and $\sqrt{NH}e^{-\beta\Gamma}\to0$, the proposed estimator attains the corresponding semiparametric efficiency bound for $J^*$.
\end{corollary}

\begin{proof}
Suppose the Bellman-optimal action is unique at every state, and write the unique optimal policy as $\pi^*$. Since $\cS$ and $\cA$ are finite, the optimal action gap is strictly positive at the true law. Along any sufficiently small regular dominated local submodel of the environment law, the environment kernel and the anchored fixed-policy evaluation pair for $\pi^*$ vary continuously. The perturbed action-value function therefore remains within a fixed fraction of the true-law positive gap, so $\pi^*$ remains greedy and hence optimal in a neighborhood of the true law. Thus locally $J^*=J^{\pi^*}$, and $J^*$ has the same pathwise derivative and efficient influence function as the value of the fixed policy $\pi^*$.

Let $\lambda_H^*$ be the normalized fixed-policy adjoint weight for $\pi^*$, and define $\xi_H^*(Z)\coloneqq\lambda_H^*(S,A)(R-J^*+V^*(S')-Q^*(S,A))$ and $\Psi_{*,H}(O)\coloneqq H^{-1}\sum_{t=0}^{H-1}\xi_H^*(Z_t)$. The fixed-policy version of the argument in the preceding section shows that $\Psi_{*,H}$ is the efficient influence function for $J^*$. Its variance is $\sigma_{*,H}^2/H$, where $\sigma_{*,H}^2\coloneqq H^{-1}\sum_{t=0}^{H-1}\EE[\xi_H^*(Z_t)^2]$, so the efficiency bound based on $N$ independent trajectories is $\sigma_{*,H}^2/(NH)$.

Under uniqueness, $\pi^{\mathrm{unif}}=\pi^*$. Proposition~\ref{prop:large-beta-pe-limit} therefore gives $\lambda_\beta\to\lambda_H^*$ and $\xi_{\beta,H}\to\xi_H^*$ as $\beta\to\infty$. The normalized empirical replacement following Proposition~\ref{prop:large-beta-pe-limit} then allows $\xi_{\beta,H}(Z_{i,t})$ to be replaced by $\xi_H^*(Z_{i,t})$ up to $o_p(1)$ at the inference scale. Moreover, Theorem~\ref{thm:exp} gives $|J_\beta-J^*|\le Ce^{-\beta\Gamma}$, so $\sqrt{NH}e^{-\beta\Gamma}\to0$ makes the smoothing bias negligible. Combining these facts with the fixed-horizon asymptotic linear representation yields $\widehat J_\beta-J^*=(NH)^{-1}\sum_{i=1}^N\sum_{t=0}^{H-1}\xi_H^*(Z_{i,t})+o_p(N^{-1/2})$. Hence the proposed estimator has the efficient influence function for $J^*$ and attains the semiparametric efficiency bound $\sigma_{*,H}^2/(NH)$.
\end{proof}

\subsubsection{Nonunique Optimum and the Fixed-$\pi^{\mathrm{unif}}$ Limit}

At nonunique optima, the optimized functional $J^*$ is generally nonregular and need not admit a classical efficient influence function. If different optimal policies have different pathwise derivatives, the directional derivative of $J^*$ is not linear in the score direction, as in the nonunique optimal-treatment-value problem of \citet{luedtke2016statistical}.

At the true law, however, \eqref{eq:unif-attains-max} gives $J^{\pi^{\mathrm{unif}}}=J^*$ and $Q^{\pi^{\mathrm{unif}}}=Q^*$. If the true-law policy $\pi^{\mathrm{unif}}$ is held fixed under local perturbations, its value is an ordinary fixed-policy functional with efficient influence function
$$
\Psi_{\mathrm{unif},H}(O)
=
\frac1H
\sum_{t=0}^{H-1}
\lambda_H^{\mathrm{unif}}(S_t,A_t)
(R_{t+1}-J^*+V^*(S_{t+1})-Q^*(S_t,A_t)).
$$
Proposition~\ref{prop:large-beta-pe-limit} gives $\Psi_{\beta,H}\to\Psi_{\mathrm{unif},H}$ in $L_2$ and $\sigma_H^2\to\sigma_{\mathrm{PE},H}^2$ for fixed $H$. Thus, at ties, the large-$\beta$ limit of the smoothed influence function coincides with the efficient influence function for the frozen policy $\pi^{\mathrm{unif}}$. This is not an efficiency statement for the optimized functional $J^*$ itself at a generic nonregular tie.

\subsection{Feasible Variance Estimation}\label{sec:app-feasible-variance}

For $i\in I_k$, define the foldwise estimated innovation $\widehat\xi_{i,t}\coloneqq\varphi_\beta(Z_{i,t};\widehat J_{-k},\widehat Q_{-k},\widehat\lambda_{-k})-\widehat J_{-k}$.
The unified feasible scale estimator is
\[
\widehat\sigma
\coloneqq
\left(
\frac1{NH}
\sum_{k=1}^K
\sum_{i\in I_k}
\sum_{t=0}^{H-1}
\widehat\xi_{i,t}^2
\right)^{1/2}.
\]

\begin{lemma}[Common score replacement for variance estimation]\label{lem:variance-score-replacement}
Under the conditions of either regime in Theorem~\ref{thm:clt-N},
\[
\frac1{NH}
\sum_{k=1}^K\sum_{i\in I_k}\sum_{t=0}^{H-1}
(\widehat\xi_{i,t}-\xi_{i,t})^2
=o_p(1),
\]
and the corresponding variance difference satisfies
\[
\widehat\sigma^2
-
\frac1{NH}\sum_{i=1}^N\sum_{t=0}^{H-1}\xi_{i,t}^2
=o_p(1).
\]
\end{lemma}

\begin{proof}
Because all terms in $b_k$ are nonnegative, either Assumption~\ref{ass:if-fixed-rate} or Assumption~\ref{ass:if-growing-rate} implies $\max_k\beta\|\widehat Q_{-k}-Q_\beta\|_{\bar\nu_H,2}^2=o_p(1)$. Hence Lemma~\ref{lem:empirical-score-replacement} applies under either regime. Since $\xi_{i,t}=\varphi_\beta(Z_{i,t})-J_\beta$, for $i\in I_k$ we have
\[
\widehat\xi_{i,t}-\xi_{i,t}
=
\left(
\varphi_\beta(Z_{i,t};\widehat\eta_{-k})-\varphi_\beta(Z_{i,t})
\right)
-
(\widehat J_{-k}-J_\beta).
\]
The inequality $(a-b)^2\le2a^2+2b^2$, Lemma~\ref{lem:empirical-score-replacement}, and Assumption~\ref{ass:if-consistency} therefore give
\[
\begin{aligned}
&\frac1{NH}
\sum_{k=1}^K\sum_{i\in I_k}\sum_{t=0}^{H-1}
(\widehat\xi_{i,t}-\xi_{i,t})^2 \le
\frac2{NH}
\sum_{k=1}^K\sum_{i\in I_k}\sum_{t=0}^{H-1}
\left(
\varphi_\beta(Z_{i,t};\widehat\eta_{-k})-\varphi_\beta(Z_{i,t})
\right)^2
+2\max_{1\le k\le K}|\widehat J_{-k}-J_\beta|^2
=o_p(1).
\end{aligned}
\]
Let $A_{N,H}$ denote the average on the left. By $|a^2-b^2|\le|a-b|(|a-b|+2|b|)$ and Cauchy--Schwarz,
\[
\left|
\widehat\sigma^2
-
\frac1{NH}\sum_{i=1}^N\sum_{t=0}^{H-1}\xi_{i,t}^2
\right|
\le
A_{N,H}^{1/2}
\left[
A_{N,H}^{1/2}
+2\left(
\frac1{NH}\sum_{i=1}^N\sum_{t=0}^{H-1}\xi_{i,t}^2
\right)^{1/2}
\right].
\]
The first factor is $o_p(1)$. Assumption~\ref{ass:if-envelope} and Corollaries~\ref{cor:app-smoothed-target-bounded} and~\ref{cor:target-adjoint-bound} uniformly bound $|\xi_{i,t}|$, so the bracket is $O_p(1)$. This proves the second conclusion.
\end{proof}

\subsubsection{Fixed-Horizon Consistency of the Unified Estimator}

\begin{lemma}[Fixed-$H$ variance consistency]\label{lem:fixed-H-variance-consistency}
Assume the fixed-$H$ conditions of Theorem~\ref{thm:clt-N}. Then
\[
\widehat\sigma^2-\sigma^2\xrightarrow{p}0.
\]
\end{lemma}

\begin{proof}
Lemma~\ref{lem:variance-score-replacement} gives $\widehat\sigma^2-(NH)^{-1}\sum_{i,t}\xi_{i,t}^2=o_p(1)$.
Set $W_{i,N}\coloneqq H^{-1}\sum_{t=0}^{H-1}\xi_{i,t}^2$. For each $N$, the variables $\{W_{i,N}\}_{i=1}^N$ are independent across trajectories and have common mean $\EE[W_{i,N}]=H^{-1}\sum_{t=0}^{H-1}\EE[\xi_{i,t}^2]=\sigma^2$.  Although $\beta$ may vary with $N$, this is a bounded triangular array, so $\Var(N^{-1}\sum_{i=1}^N W_{i,N})\le C/N$. Chebyshev's inequality gives
\[
\frac1N\sum_{i=1}^N W_{i,N}-\sigma^2
=
\frac1{NH}\sum_{i=1}^N\sum_{t=0}^{H-1}\xi_{i,t}^2-\sigma^2
=o_p(1).
\]
Combining the quadratic-variation and replacement bounds proves $\widehat\sigma^2-\sigma^2=o_p(1)$. Proposition~\ref{prop:bellman-innovation-variance-lower-bounds} gives $\sigma^2\ge c>0$ eventually, and therefore $\widehat\sigma/\sigma\to_p1$.
\end{proof}

\subsubsection{Diverging-Horizon Consistency of the Unified Estimator}

\begin{lemma}[Diverging-$H$ variance consistency]\label{lem:growing-H-variance-consistency}
Assume the diverging-$H$ conditions of Theorem~\ref{thm:clt-N}. Then $\widehat\sigma^2-\sigma^2\xrightarrow{p}0$.
\end{lemma}

\begin{proof}
Lemma~\ref{lem:variance-score-replacement} gives $\widehat\sigma^2-(NH)^{-1}\sum_{i,t}\xi_{i,t}^2=o_p(1)$.
Proposition~\ref{prop:rqv-sufficient-route} gives $(NH)^{-1}\sum_{i=1}^N\sum_{t=0}^{H-1}\xi_{i,t}^2-\sigma^2=o_p(1)$. Combining these two bounds proves the claim, and the lower bound on $\sigma^2$ gives $\widehat\sigma/\sigma\to_p1$.
\end{proof}

\subsubsection{Proof of Corollary~\ref{cor:studentized-optimal-value}}
\begin{proof}
In either regime, Theorem~\ref{thm:clt-N} gives
\[
\frac{\sqrt{NH}(\widehat J_\beta-J^*)}{\sigma}
\xrightarrow{d}N(0,1).
\]
In the fixed-$H$ regime, Lemma~\ref{lem:fixed-H-variance-consistency} gives $\widehat\sigma/\sigma\to_p1$. In the diverging-$H$ regime, Lemma~\ref{lem:growing-H-variance-consistency} gives the same conclusion. Slutsky's theorem therefore gives $\sqrt{NH}(\widehat J_\beta-J^*)/\widehat\sigma\xrightarrow{d}N(0,1)$ in either regime. The unified Wald interval follows by inversion of this studentized limit.

\end{proof}

\section{Proofs for Section~\ref{sec:tabular-fast-rate}}\label{app:fast-rate-proofs}
This appendix proves the finite-sample rate in Theorem~\ref{thm:fast_rate}. We first represent the nonlinear conditional Bellman residual by an exact secant identity and use it to obtain a uniform lower bound for the residual. We then establish high-probability bounds for the empirical criterion. These bounds control the empirical norm in the quadratic penalty, the criterion evaluated at $(Q_\beta,J_\beta)$, and empirical deviations of residual increments. Finally, we combine the residual lower bound and the empirical bounds for the estimator. 
For the proofs in this section, write $\delta_{Q,J}(Z)\coloneqq R-J+V_\beta(Q)(S')-Q(S,A)$ and $m_{Q,J}(s,a)\coloneqq\EE[\delta_{Q,J}(Z)\mid S=s,A=a]$.

\subsection{Exact Secant Representation and Global Residual Identification}\label{app:fast-rate-secant}

\begin{lemma}
\label{lem:secant-residual}
For every $(Q,J)\in\cQ\times\RR$, set $Q_t\coloneqq Q_\beta+t(Q-Q_\beta)$, $t\in[0,1]$, and define $(D_{\beta;Q,Q_\beta}f)(s)\coloneqq\int_0^1\sum_{a\in\cA}\omega_{\beta,Q_t}(a\mid s)f(s,a)\,dt$. Then
\begin{equation}
m_{Q,J}-m_{Q_\beta,J_\beta}
=
-(J-J_\beta)\one-(Q-Q_\beta)
+PD_{\beta;Q,Q_\beta}(Q-Q_\beta).
\label{eq:fast-active-exact-residual-operator}
\end{equation}
Since $m_{Q_\beta,J_\beta}=0$, equivalently,
\[
m_{Q,J}
=
-(J-J_\beta)\one
-(Q-Q_\beta)
+
PD_{\beta;Q,Q_\beta}(Q-Q_\beta).
\]
Moreover, $D_{\beta;Q,Q_\beta}$ is allowable.
\end{lemma}

\begin{proof} 
Fix $(Q,J)\in\cQ\times\RR$, and write $\delta Q\coloneqq Q-Q_\beta$, $\delta J\coloneqq J-J_\beta$, and $Q_t\coloneqq Q_\beta+t\delta Q$.
Since $m_{Q_\beta,J_\beta}=0$,
\[
(m_{Q,J}-m_{Q_\beta,J_\beta})(s,a)
=
-\delta J-\delta Q(s,a)
+
\sum_{s'}P(s'\mid s,a)
(V_\beta(Q)(s')-V_\beta(Q_\beta)(s')).
\]
The operator $D_{\beta;Q,Q_\beta}$ in the statement is the operator from Lemma~\ref{lem:softmax-calculus}. That lemma gives $V_\beta(Q)-V_\beta(Q_\beta)=D_{\beta;Q,Q_\beta}\delta Q$ and verifies that $D_{\beta;Q,Q_\beta}$ is allowable. Substitution gives \eqref{eq:fast-active-exact-residual-operator}; the equivalent identity follows because $m_{Q_\beta,J_\beta}=0$.
\end{proof}

The residual identity involves the operator $D_{\beta;Q,Q_\beta}$, which is obtained by averaging the derivative of $V_\beta$ along the segment from $Q_\beta$ to $Q$. It is therefore not enough to control only the derivative operator evaluated at $Q_\beta$. Lemma~\ref{lem:injectivity-equivalence} below applies because the row sum and signed envelope bounds imply that every such operator $D_{\beta;Q,Q_\beta}$ belongs to the allowable family.

\begin{lemma}[Uniform quantitative identification over allowable aggregation operators] 
\label{lem:injectivity-equivalence}
Suppose Assumption~\ref{ass:uniform-signed-aggregation-identification} holds. Then there exists $c>0$ such that, for every allowable aggregation operator $D$,
\[
\left\|-j\one-q+PDq\right\|_{\bar\mu,2}
\ge
c(\|q\|_{\bar\mu,2}+|j|),
\qquad
q\in\cQ,\quad j\in\RR .
\]
The constant $c$ is independent of $n$ and $\beta$.
\end{lemma}

\begin{proof} 
Let $r=-j\one-q+PDq$ and write $\bar\mu_{\min}\coloneqq\min_x\bar\mu(x)>0$. Lemma~\ref{lem:uniform-sup-inverse} gives $\|q\|_\infty+|j|\le C\|r\|_\infty$, uniformly over allowable $D$. Finite-state norm equivalence gives
\[
\|q\|_{\bar\mu,2}+|j|
\le
\|q\|_\infty+|j|
\le
C\|r\|_\infty
\le
C\bar\mu_{\min}^{-1/2}\|r\|_{\bar\mu,2}.
\]
Rearranging proves the stated lower bound with $c=C^{-1}\bar\mu_{\min}^{1/2}$. The fixed one-step support lower bound makes this constant independent of $n$ and $\beta$. For a horizon-dependent pooled law, Assumption~\ref{ass:pooled-support} gives $\inf_H\min_x\bar\mu_H(x)>0$, so the same norm-equivalence argument makes the constant uniform over the relevant horizons, and hence independent of $N$ and $H$ as well.
\end{proof}

\begin{proposition}[Quantitative identification through the Bellman residual]
\label{prop:bellman-residual-identification}
Suppose Assumption~\ref{ass:uniform-signed-aggregation-identification} holds.
Then there exists a constant $c>0$, independent of $n$ and $\beta$, such that, for every $(Q,J)\in\cQ\times\RR$,
\[
\|m_{Q,J}\|_{\bar\mu,2}
\ge
c
(\|Q-Q_\beta\|_{\bar\mu,2}+|J-J_\beta|).
\]
\end{proposition}

\begin{proof} 

Fix $(Q,J)\in\cQ\times\RR$, and write $\delta Q\coloneqq Q-Q_\beta$ and $\delta J\coloneqq J-J_\beta$. Because both $Q$ and $Q_\beta$ belong to $\cQ$, we have $\delta Q\in\cQ$. Lemma~\ref{lem:secant-residual} gives $m_{Q,J}-m_{Q_\beta,J_\beta}=-\delta J\one-\delta Q+PD_{\beta;Q,Q_\beta}\delta Q$, and $D_{\beta;Q,Q_\beta}$ is allowable. Applying Lemma~\ref{lem:injectivity-equivalence} with $(q,j,D)=(\delta Q,\delta J,D_{\beta;Q,Q_\beta})$ yields the displayed bound. Since $m_{Q_\beta,J_\beta}=0$, this proves the claim.
\end{proof}

\subsection{Bounds for the Empirical Rate Argument}\label{app:fast-rate-empirical-bounds}
For a measurable one step function $f$, write $
\|f\|_2
:=
(\EE[f(Z)^2])^{1/2}$
.

\begin{lemma} 
\label{lem:fast-residual-lipschitz}
There exists a constant $C<\infty$, independent of $n$ and $\beta$, such that, for every $(Q,J)\in\cQ\times\RR$,
\[
\|\delta_{Q,J}-\delta_{Q_\beta,J_\beta}\|_2
\le
C(\|Q-Q_\beta\|_{\bar\mu,2}+|J-J_\beta|).
\]
Moreover, $|\delta_{Q,J}(Z)-\delta_{Q_\beta,J_\beta}(Z)|\le C(\|Q-Q_\beta\|_{\bar\mu,2}+|J-J_\beta|)$ almost surely.
\end{lemma}

\begin{proof}
Fix $(Q,J)\in\cQ\times\RR$, and write $G\coloneqq Q-Q_\beta$ and $\Delta_J\coloneqq J-J_\beta$.
By the definition of the one step residual,
\[
\delta_{Q,J}(Z)-\delta_{Q_\beta,J_\beta}(Z)
=
-\Delta_J
+(V_\beta(Q)(S')-V_\beta(Q_\beta)(S'))
-G(S,A).
\]
Lemma~\ref{lem:softmax-calculus}, applied with $(Q_1,Q_2)=(Q,Q_\beta)$, and its signed-envelope bound give
\[
|V_\beta(Q)(s)-V_\beta(Q_\beta)(s)|
\le
C_{\cA}\max_{a\in\cA}|G(s,a)|.
\]
By the full-support condition on $\bar\mu$ from Section~\ref{sec:tabular-fast-rate} and finite support, $\|G\|_\infty\le C\|G\|_{\bar\mu,2}$. Consequently $\|V_\beta(Q)(S')-V_\beta(Q_\beta)(S')\|_2\le C\|Q-Q_\beta\|_{\bar\mu,2}$, and the same direct bound holds almost surely. Moreover, $\|G(S,A)\|_2=\|G\|_{\bar\mu,2}$ and $|G(S,A)|\le C\|G\|_{\bar\mu,2}$ almost surely. Combining the preceding bounds by the triangle inequality, and then absorbing constants into a single $C<\infty$, gives $\|\delta_{Q,J}-\delta_{Q_\beta,J_\beta}\|_2\le C(\|Q-Q_\beta\|_{\bar\mu,2}+|J-J_\beta|)$.

The almost-sure envelope follows from the same decomposition and the finite-support bound on $\|G\|_\infty$.
\end{proof}

We then present the empirical bounds. The three bounds below play different roles in the fast rate proof. The first controls the empirical quadratic test function penalty. The second controls the regularized criterion at $(Q_\beta,J_\beta)$. The third transfers the population moment generated by a residual aligned test function to its empirical counterpart.

\begin{lemma}[Empirical norm and process bounds]
\label{lem:stochastic-verified}

Use the independent one-step setup of Section~\ref{sec:tabular-fast-rate} and the bounded tabular test function class $\mathcal G$ introduced there. Suppose Assumption~\ref{ass:if-envelope} holds, and recall that Section~\ref{sec:tabular-fast-rate} assumes that the fixed one-step state-action marginal $\bar\mu$ has full support. For a sufficiently large fixed constant $C<\infty$, define
\[
\eta_n
\coloneqq
C
\sqrt{\frac{|\cS||\cA|\log(en)+\log(1/\delta)}{n}},
\]
Then, with probability at least $1-\delta$, there is an event $\mathcal E_n$ on which the following bounds hold.
\[
\left|\Pn[g(S,A)^2]-\|g\|_{\bar\mu,2}^2\right|
\le
\frac12\|g\|_{\bar\mu,2}^2+\frac12\eta_n^2,
\qquad g\in\mathcal G,
\]
\[
\sup_{g\in\mathcal G}
\left\{
2(\Pn[g\delta_{Q_\beta,J_\beta}]-\EE[g\delta_{Q_\beta,J_\beta}])
-\rho\Pn[g(S,A)^2]
\right\}
\le
C\eta_n^2,
\]
and, for every $g\in\mathcal G$ and every $Q\in\cQ$ and $J\in\RR$ with $\|Q\|_\infty\le B_Q$ and $|J|\le B_J$,
\[
\left|
\Pn\!\left[g(\delta_{Q,J}-\delta_{Q_\beta,J_\beta})\right]
-\EE\!\left[g(\delta_{Q,J}-\delta_{Q_\beta,J_\beta})\right]
\right|
\le
C\eta_n\|g\|_{\bar\mu,2}
(\|Q-Q_\beta\|_{\bar\mu,2}+|J-J_\beta|)
+
C\eta_n^2.
\]

The constants may depend on the fixed state/action cardinalities, the positive state-action support lower bound of the one-step law, the reward bound, $B_Q$, $B_J$, $B_G$, $\rho$, and the population identification constants, but are independent of $n$, $\beta$, and $\delta$. If the one-step law is a pooled law satisfying Assumption~\ref{ass:pooled-support}, the same constants are uniform over the relevant horizons by the common support lower bound.
\end{lemma}

\begin{proof} 
Let $\mathcal X\coloneqq\mathcal S\times\mathcal A$ and $d\coloneqq|\mathcal X|$. We first establish the empirical norm comparison directly from the empirical state-action frequencies. For $x=(s,a)\in\mathcal X$, write $\widehat\mu_n(x)\coloneqq\Pn\{(S,A)=x\}$. For every tabular function $g$,
\[
\Pn[g(S,A)^2]-\|g\|_{\bar\mu,2}^2
=
\sum_{x\in\mathcal X}
(\widehat\mu_n(x)-\bar\mu(x))g(x)^2.
\]
For each $x\in\mathcal X$, a binomial Bernstein bound applies to $\widehat\mu_n(x)$. Taking a union bound over the finite state-action space gives, with probability at least $1-\delta/3$, simultaneously over all $x$,
\[
\left|
\widehat\mu_n(x)-\bar\mu(x)
\right|
\le
\frac12\bar\mu(x)
+
C\frac{\log(en)+\log(d/\delta)}{n}.
\]
Here and below, the constant $C$ depends only on fixed tabular quantities.
Let $B_{\mathcal G}\coloneqq\sup_{g\in\mathcal G}\|g\|_\infty<\infty$. Then, uniformly over $g\in\mathcal G$,
\[
\begin{aligned}
\left|
\Pn[g(S,A)^2]-\|g\|_{\bar\mu,2}^2
\right|
&\le \frac12
\sum_x\bar\mu(x)g(x)^2
+
C\frac{\log(en)+\log(d/\delta)}{n}
\sum_x g(x)^2 \\
&\le \frac12\|g\|_{\bar\mu,2}^2
+
C B_{\mathcal G}^2d
\frac{\log(en)+\log(d/\delta)}{n},
\end{aligned}
\]
Choosing the fixed constant $C$ in the definition of $\eta_n$ sufficiently large makes the second term at most $\eta_n^2/2$. Hence
\[
\left|\Pn[g(S,A)^2]-\|g\|_{\bar\mu,2}^2\right|
\le
\frac12\|g\|_{\bar\mu,2}^2+\frac12\eta_n^2,
\qquad g\in\mathcal G.
\]

Then we consider the second inequality. Let $\delta_\beta\coloneqq\delta_{Q_\beta,J_\beta}$. At the truth, the conditional mean vanishes, and therefore so does the corresponding moment for every $g\in\mathcal G$:
\[
\EE[\delta_\beta(Z)\mid S,A]
=
m_{Q_\beta,J_\beta}(S,A)
=
0.
\]
The corresponding moment therefore vanishes:
\[
\EE[g\delta_\beta]
=
\EE\left[g(S,A)\EE(\delta_\beta(Z)\mid S,A)\right]
=
0.
\]
By Assumption~\ref{ass:if-envelope}, Corollary~\ref{cor:app-smoothed-target-bounded}, and $|V_\beta(Q_\beta)(s)|\le\|Q_\beta\|_\infty$, the residual $\delta_\beta$ is uniformly bounded. Define the cell moments $M_x\coloneqq \Pn[1\{(S,A)=x\}\delta_\beta(Z)]$, $x\in\mathcal X$.
Since $\EE[\delta_\beta(Z)\mid S,A]=0$, each $M_x$ is an average of bounded mean-zero variables with variance bounded by $C\bar\mu(x)/n$. Let $\ell\coloneqq\log(6d/\delta)$. Bernstein's inequality and a union bound over $x\in\mathcal X$ imply that, with probability at least $1-\delta/3$,
\[
|M_x|
\le
C\sqrt{\frac{\bar\mu(x)\ell}{n}}
+
C\frac{\ell}{n},
\qquad x\in\mathcal X .
\]
Let $\underline\mu\coloneqq\min_x\bar\mu(x)>0$. If $\ell/n\le1$, then
\[
\sum_{x\in\mathcal X}\frac{M_x^2}{\bar\mu(x)}
\le
C\sum_{x\in\mathcal X}
\left(
\frac{\ell}{n}
+
\frac{\ell^2}{n^2\bar\mu(x)}
\right)
\le
C\frac{d\ell}{n}
\le
C\eta_n^2.
\]
If $\ell/n>1$, deterministic boundedness gives $\sum_xM_x^2/\bar\mu(x)\le C$, while the fixed constant in $\eta_n$ can be chosen so that $\eta_n^2\ge C\ell/n\ge C$. Thus, in both cases, $\sum_xM_x^2/\bar\mu(x)\le C\eta_n^2$.
The preceding cell-moment bound gives, uniformly over $g\in\mathcal G$,
\[
2(\Pn[g\delta_\beta]-\EE[g\delta_\beta])
=
2\sum_{x\in\mathcal X}g(x)M_x
\le
C\eta_n\|g\|_{\bar\mu,2}.
\]
On the empirical norm event, $\Pn[g(S,A)^2]\ge\frac12\|g\|_{\bar\mu,2}^2-\frac12\eta_n^2$, and taking the supremum over $\|g\|_{\bar\mu,2}$ on the right hand side then gives
\[
2(\Pn[g\delta_\beta]-\EE[g\delta_\beta])
-\rho\Pn[g(S,A)^2]
\le
C\eta_n\|g\|_{\bar\mu,2}
-
\frac{\rho}{2}\|g\|_{\bar\mu,2}^2
+
C\eta_n^2.
\]
Taking the supremum over $g\in\mathcal G$ gives
\[
\sup_{g\in\mathcal G}
\left\{
2(\Pn[g\delta_\beta]-\EE[g\delta_\beta])
-\rho\Pn[g(S,A)^2]
\right\}
\le
C\eta_n^2.
\]
The negative quadratic test function penalty absorbs the linear localized fluctuation and leaves an $O(\eta_n^2)$ bound.

It remains to prove the product process bound. For every $Q\in\cQ$ and $J\in\RR$ with $\|Q\|_\infty\le B_Q$ and $|J|\le B_J$, set $d_{Q,J}(Z)\coloneqq\delta_{Q,J}(Z)-\delta_{Q_\beta,J_\beta}(Z)$ and $e_{Q,J}\coloneqq\|Q-Q_\beta\|_{\bar\mu,2}+|J-J_\beta|$. Lemma~\ref{lem:fast-residual-lipschitz} supplies the $\beta$-uniform residual increment bounds $\|d_{Q,J}\|_2\le Ce_{Q,J}$ and $|d_{Q,J}(Z)|\le Ce_{Q,J}$ almost surely. Since $\|g(S,A)\|_2=\|g\|_{\bar\mu,2}$, we have
\[
\|g d_{Q,J}\|_2
\le
Ce_{Q,J}
\|g(S,A)\|_2
\le
C\|g\|_{\bar\mu,2}e_{Q,J}.
\]
If $e_{Q,J}=0$, then full support of $\bar\mu$ gives $Q=Q_\beta$ and $J=J_\beta$, so $d_{Q,J}=0$. For $e_{Q,J}>0$, introduce the normalized proof-local coordinates $a=(J-J_\beta)/e_{Q,J}$, $u=(Q-Q_\beta)/e_{Q,J}$, and $v=\{V_\beta(Q)-V_\beta(Q_\beta)\}/e_{Q,J}$. Then $|a|\le1$, $\|u\|_{\bar\mu,2}\le1$, and finite support gives $\|u\|_\infty\le C$. The signed-weight envelope in Lemma~\ref{lem:if-omega-tv-bound} gives the $\beta$-uniform Lipschitz bound for $V_\beta$, hence $\|v\|_\infty\le C$. The normalized function is
\[
h(Z)
=
\frac{g(S,A)d_{Q,J}(Z)}{e_{Q,J}}
=
g(S,A)\{-a+v(S')-u(S,A)\}.
\]
The normalized class is therefore contained in an enlarged class $\mathcal H$ obtained by allowing $(g,a,u,v)$ to range over this fixed bounded coordinate set. With $d=|\mathcal S||\mathcal A|$, this enlarged class has coordinate dimension $p=2d+|\mathcal S|+1$: $d$ coordinates for $g$, $d$ for $u$, $|\mathcal S|$ for $v$, and one for $a$. On the fixed bounded coordinate box, the map $(g,a,u,v)\mapsto g(S,A)\{-a+v(S')-u(S,A)\}$ is uniformly Lipschitz into the supremum norm. This containment does not require the inverse map $(Q,J)\mapsto(a,u,v)$ to be Lipschitz at $(Q_\beta,J_\beta)$.

Construct a coordinate grid fine enough that the resulting function net $\mathcal H_n$ has supremum-norm radius at most $n^{-1}$. The Lipschitz constant and coordinate ranges are fixed, so $\log|\mathcal H_n|\le Cp\log(Cn)$.
For $f\in\mathcal H_n$, the envelope is bounded by a fixed constant and $\Var(f(Z))\le\|f\|_2^2$. Bernstein's inequality and a union bound with $t_n=\log(6|\mathcal H_n|/\delta)$ give, with probability at least $1-\delta/3$, simultaneously over $f\in\mathcal H_n$,
\[
\left|(\Pn-\EE)f\right|
\le
C\|f\|_2\sqrt{\frac{t_n}{n}}
+
C\frac{t_n}{n}.
\]
For an arbitrary $h$ in the enlarged normalized class, choose $f\in\mathcal H_n$ with $\|h-f\|_\infty\le n^{-1}$. Then $|(\Pn-\EE)(h-f)|\le2/n$ and $\|f\|_2\le\|h\|_2+n^{-1}$. Since $p\le Cd$, $t_n\le C\{d\log(en)+\log(1/\delta)\}$, and $n\ge2$, the terms $t_n/n$, $n^{-1}\sqrt{t_n/n}$, and $n^{-1}$ are absorbed by $C\eta_n^2$. Thus
\[
\left|(\Pn-\EE)h\right|
\le
C\eta_n\|h\|_2+C\eta_n^2,
\qquad h\in\mathcal H .
\]
Applying this bound to the original normalized functions, multiplying back by $e_{Q,J}$, and using the product-norm bound above gives
\[
\left|
\Pn[g d_{Q,J}]-\EE[g d_{Q,J}]
\right|
\le
C\eta_n\|g d_{Q,J}\|_2
+
C\eta_n^2e_{Q,J}.
\]
For admissible $(Q,J)$, finite support gives $\|Q-Q_\beta\|_{\bar\mu,2}\le\|Q-Q_\beta\|_\infty\le B_Q+\|Q_\beta\|_\infty$ and $|J-J_\beta|\le B_J+|J_\beta|$, while Corollary~\ref{cor:app-smoothed-target-bounded} gives $\sup_\beta\|Q_\beta\|_\infty<\infty$ and $\sup_\beta|J_\beta|<\infty$. Hence $e_{Q,J}\le C$, uniformly in $n$, $N$, $H$, and $\beta$, and the additive term $C\eta_n^2e_{Q,J}$ is absorbed into $C\eta_n^2$.
Using the preceding product-norm bound gives
\[
\left|
\Pn[g d_{Q,J}]-\EE[g d_{Q,J}]
\right|
\le
C\eta_n\|g\|_{\bar\mu,2}e_{Q,J}
+
C\eta_n^2.
\]
Taking a union bound over the empirical norm, cell-moment, and product-process events gives $\PP(\mathcal E_n)\ge 1-\delta$.
\end{proof}

\begin{lemma}[Bellman localization]\label{lem:bellman-localization}
Let $P$ be a one-step law with state-action marginal $\mu$, let $\widehat P$ be an empirical averaging operator, and fix $\eta>0$. Use the Bellman residuals $\delta_{Q,J}$ and conditional moments $m_{Q,J}$ defined above, and write $e_{Q,J}=\|Q-Q_\beta\|_{\mu,2}+|J-J_\beta|$. Suppose the true pair $(Q_\beta,J_\beta)$ belongs to the optimization domain, every pair in that domain satisfies $e_{Q,J}\le C_{\mathrm{dom}}$, and $(\widehat Q,\widehat J)$ is an exact minimizer of the corresponding regularized minimax criterion. Suppose that, on a given realization,
\begin{enumerate}[label=(\roman*)]
\item $\|m_{Q,J}\|_{\mu,2}\ge c e_{Q,J}$;
\item $\widehat L(Q_\beta,J_\beta)\le C_0\eta^2$;
\item $|\widehat P[g(S,A)^2]-\|g\|_{\mu,2}^2|\le \frac12\|g\|_{\mu,2}^2+C_1\eta^2$ for $g\in\mathcal G$;
\item $|(\widehat P-P)[g(\delta_{Q,J}-\delta_{Q_\beta,J_\beta})]|\le C_2\eta\|g\|_{\mu,2}e_{Q,J}+C_2\eta^2$ uniformly over the parameter domain and $g\in\mathcal G$.
\end{enumerate}
Suppose also that $\mathcal G$ is symmetric and star-shaped and contains the $L_2(\mu)$ ball of radius $r_{\mathcal G}>0$ around zero. Then, on the same realization,
\[
\|\widehat Q-Q_\beta\|_{\mu,2}+|\widehat J-J_\beta|
\le C_*\eta,
\]
where $C_*$ depends only on $c,C_0,C_1,C_2,\rho,r_{\mathcal G}$, and $C_{\mathrm{dom}}$, and not on $\eta$.
\end{lemma}

\begin{proof}
Let $e\coloneqq\|\widehat Q-Q_\beta\|_{\mu,2}+|\widehat J-J_\beta|$. Empirical optimality and condition (ii) give $\widehat L(\widehat Q,\widehat J)\le\widehat L(Q_\beta,J_\beta)\le C_0\eta^2$. For each $g\in\mathcal G$, symmetry and star-shapedness imply that $g/2,-g/2\in\mathcal G$. Evaluating the estimated criterion at $g/2$ and the truth criterion at $-g/2$, adding the two inequalities, and multiplying by two gives
\begin{equation}
2\widehat P\!\left[g(\delta_{\widehat Q,\widehat J}-\delta_{Q_\beta,J_\beta})\right]
-\rho\widehat P[g(S,A)^2]
\le 4C_0\eta^2,
\qquad g\in\mathcal G.
\label{eq:fast-basic-ineq}
\end{equation}

If $e=0$, there is nothing to prove. Otherwise condition (i) gives $\|m_{\widehat Q,\widehat J}\|_{\mu,2}\ge ce>0$. Choose a fixed $\eta_0>0$ satisfying $\eta_0\le\min\{r_{\mathcal G},c/[2(1+C_2)]\}$. For $0<\eta<\eta_0$, define $g_\eta\coloneqq\eta m_{\widehat Q,\widehat J}/\|m_{\widehat Q,\widehat J}\|_{\mu,2}$.
Then $g_\eta\in\mathcal G$, $\|g_\eta\|_{\mu,2}=\eta$, and, because $m_{Q_\beta,J_\beta}=0$,
\[
P[g_\eta(\delta_{\widehat Q,\widehat J}-\delta_{Q_\beta,J_\beta})]
=
\eta\|m_{\widehat Q,\widehat J}\|_{\mu,2}
\ge c\eta e.
\]
Condition (iv) therefore gives
\[
\begin{aligned}
\widehat P\!\left[g_\eta(\delta_{\widehat Q,\widehat J}-\delta_{Q_\beta,J_\beta})\right]
&\ge
P\!\left[g_\eta(\delta_{\widehat Q,\widehat J}-\delta_{Q_\beta,J_\beta})\right]
-C_2\eta^2e-C_2\eta^2\\
&\ge c\eta e-C_2\eta^2e-C_2\eta^2.
\end{aligned}
\]
Condition (iii) gives $\widehat P[g_\eta(S,A)^2]\le(3/2+C_1)\eta^2$. Substituting these bounds into \eqref{eq:fast-basic-ineq} yields
\[
2c\eta e
\le
\{4C_0+2C_2+\rho(3/2+C_1)\}\eta^2
+2C_2\eta^2e.
\]
The definition of $\eta_0$ allows the last term to be absorbed into the left-hand side, giving $e\le C\eta$. If $\eta\ge\eta_0$, the deterministic domain bound gives $e\le C_{\mathrm{dom}}\le(C_{\mathrm{dom}}/\eta_0)\eta$. Taking $C_*$ to be the larger of these two fixed constants proves the eventwise implication for every $\eta>0$.
\end{proof}

\subsection{Proof of Theorem~\ref{thm:fast_rate}}\label{app:fast-rate-theorem-proof}
\begin{proof}
Proposition~\ref{prop:bellman-residual-identification} supplies condition (i) of Lemma~\ref{lem:bellman-localization}. On precisely the event $\mathcal E_n$ from Lemma~\ref{lem:stochastic-verified}, its second inequality bounds the criterion at $(Q_\beta,J_\beta)$ by $C\eta_n^2$, its first inequality gives the required empirical norm comparison, and its third inequality gives the localized product fluctuation. The class $\mathcal G$ is symmetric, star-shaped, and contains a fixed $L_2(\bar\mu)$ ball by its construction in Section~\ref{sec:tabular-fast-rate} and the full support of $\bar\mu$. The optimization domain contains the truth and gives the required deterministic error bound. The constants in these conditions are independent of $\delta$ apart from the displayed dependence of $\eta_n$. The deterministic implication in Lemma~\ref{lem:bellman-localization}, with $(P,\widehat P,\mu,\eta)=(\EE,\Pn,\bar\mu,\eta_n)$, therefore gives $\|\widehat Q-Q_\beta\|_{\bar\mu,2}+|\widehat J-J_\beta|\le C\eta_n$ on $\mathcal E_n$. Since $\PP(\mathcal E_n)\ge1-\delta$ for every $n\ge2$, substituting the definition of $\eta_n$ proves the stated finite-sample conclusion in Theorem~\ref{thm:fast_rate}.
\end{proof}

\section{Proofs for Section~\ref{sec:plug-in-adjoint-erm}}\label{app:plug-in-adjoint-erm-proofs}
We first derive a quantitative population lower bound for the adjoint criterion from Assumption~\ref{ass:uniform-signed-aggregation-identification}. We then compare the population criterion with the reference empirical criterion and the feasible criterion. The proof of Theorem~\ref{thm:main} combines these ingredients with the empirical optimality of $\widehat\lambda_\beta$.
For $q\in\cQ$, write $A_\beta q(Z)\coloneqq q(S,A)-\sum_{a\in\cA}\omega_{\beta,Q_\beta}(a\mid S')q(S',a)$. Also write $\widehat A_\beta q(Z)\coloneqq q(S,A)-\sum_{a\in\cA}\widehat\omega_\beta(a\mid S')q(S',a)$.

\subsection{From Primal Injectivity to Dual Quantitative Identification}

The next lemma uses the map $\cL_D$ defined in Lemma~\ref{lem:uniform-sup-inverse} to convert the primal injectivity condition in Assumption~\ref{ass:uniform-signed-aggregation-identification} into a quantitative lower bound for the adjoint moment and normalization discrepancies.

\begin{lemma}[Uniform quantitative identification for the adjoint equation]  
\label{lem:uniform-dual-identification}

Under Assumption~\ref{ass:uniform-signed-aggregation-identification} in the one-step setup, and under Assumptions~\ref{ass:uniform-signed-aggregation-identification} and~\ref{ass:pooled-support} in the pooled-law setup, there exists $c>0$, independent of $n$ and $\beta$, such that, for every allowable aggregation operator $D$ and every $r\in L_2(\bar\mu)$,
\[
\sup_{\substack{q\in\cQ\\ \|q\|_{\bar\mu,2}\le1}}
\left|
\EE\!\left[
r(S,A)(q(S,A)-(Dq)(S'))
\right]
\right|^2
+
|\EE[r(S,A)]|^2
\ge
c\|r\|_{\bar\mu,2}^2.
\]
Under the pooled-law support condition, $c$ is also independent of $N$ and $H$.
\end{lemma}

\begin{proof} 
Fix an allowable $D$. Lemma~\ref{lem:injectivity-equivalence} gives the quantitative inverse bound
\[
\|q\|_{\bar\mu,2}+|j|
\le
C\|\cL_D(q,j)\|_{\bar\mu,2},
\qquad (q,j)\in\cQ\times\RR,
\]
uniformly over allowable $D$. In particular, $\cL_D$ is injective. Since its domain and codomain have the same finite dimension, it is bijective. For arbitrary $r$, let $(q_r,j_r)=\cL_D^{-1}r$. The inverse bound gives $\|q_r\|_{\bar\mu,2}+|j_r|\le C\|r\|_{\bar\mu,2}$.

Let $
A_r
\coloneqq
\sup_{\substack{q\in\cQ\\ \|q\|_{\bar\mu,2}\le1}}
\left|
\EE\!\left[
r(S,A)(q(S,A)-(Dq)(S'))
\right]
\right|,$ and 
$B_r\coloneqq|\EE[r(S,A)]|.$
Using $\cL_D(q_r,j_r)=r$ and the definition $\cL_D(q,j)=-j\one-q+PDq$, we obtain
\[
\begin{aligned}
&\|r\|_{\bar\mu,2}^2
=\langle r,\cL_D(q_r,j_r)\rangle_{\bar\mu}=-\EE\!\left[
r(S,A)(q_r(S,A)-(Dq_r)(S'))
\right]
-j_r\EE[r(S,A)]\\
&\le A_r\|q_r\|_{\bar\mu,2}+B_r|j_r|\le C(A_r+B_r)\|r\|_{\bar\mu,2}.
\end{aligned}
\]
For $r\ne0$, divide by $\|r\|_{\bar\mu,2}$ and use $(A_r+B_r)^2\le2(A_r^2+B_r^2)$ to obtain $A_r^2+B_r^2\ge c\|r\|_{\bar\mu,2}^2$. The case $r=0$ is immediate. This is the claimed dual lower bound.

For the fixed one-step law, the constant is independent of $n$ and $\beta$ by Lemma~\ref{lem:injectivity-equivalence}. For the pooled-law case, Assumption~\ref{ass:pooled-support} gives a uniform lower bound on $\min_x\bar\mu_H(x)$, so the norm-equivalence derivation of Lemma~\ref{lem:injectivity-equivalence} makes the primal inverse constant, and therefore the dual constant above, uniform in $N$ and $H$.
\end{proof}

\begin{corollary}[Uniform bound for the adjoint weight]
\label{cor:target-adjoint-bound}

Under Assumption~\ref{ass:uniform-signed-aggregation-identification} in the one-step setup, and under Assumptions~\ref{ass:uniform-signed-aggregation-identification} and~\ref{ass:pooled-support} in the pooled-law setup, there exists a constant $C$, determined by the constant in Lemma~\ref{lem:uniform-dual-identification}, such that
\[
\|\lambda_\beta\|_{\bar\mu,2}
+
\|\lambda_\beta\|_\infty
\le
C.
\]
The constant $C$ is independent of $n$ and $\beta$; under Assumption~\ref{ass:pooled-support}, it is also independent of $N$ and $H$.
\end{corollary}

\begin{proof}
Apply Lemma~\ref{lem:uniform-dual-identification} with $D=D_\beta$ and $r=\lambda_\beta$. The adjoint moment term vanishes by the adjoint equation $\EE\!\left[\lambda_\beta(S,A)(q(S,A)-(D_\beta q)(S'))\right]=0$ for $q\in\cQ$, and the normalization term equals one, so $1\ge c\|\lambda_\beta\|_{\bar\mu,2}^2$ and hence $\|\lambda_\beta\|_{\bar\mu,2}\le c^{-1/2}$. The relevant support lower bound gives $\|\lambda_\beta\|_\infty\le C\|\lambda_\beta\|_{\bar\mu,2}\le Cc^{-1/2}$, and absorbing the support factor and $c^{-1/2}$ into $C$ gives the result.
\end{proof}

\begin{proposition}[Quantitative identification through the adjoint criterion]
\label{prop:adjoint-quantitative-lower}

Under Assumption~\ref{ass:uniform-signed-aggregation-identification}, there exists $c>0$, independent of $n$ and $\beta$, such that, for every $\lambda\in\Lambda$,
\[
c\|\lambda-\lambda_\beta\|_{\bar\mu,2}^2
\le
G_\beta(\lambda).
\]  
\end{proposition} 

\begin{proof}  
By Lemma~\ref{lem:if-omega-tv-bound}, with $Q=Q_\beta$, the row sum identity and signed $\ell_1$ envelope imply that $D_\beta$ is allowable.

Fix $\lambda\in\Lambda$. Since $\cQ$ is finite dimensional and equipped with the $\bar\mu$-inner product, the linear functional $q\mapsto\EE[\lambda(S,A)A_\beta q(Z)]$ has a unique representer $a_\lambda\in\cQ$ satisfying $\EE[\lambda(S,A)A_\beta q(Z)]=\langle a_\lambda,q\rangle_{\bar\mu}$ for $q\in\cQ$. Completing the square, and then using the dual characterization of the norm of $a_\lambda$, gives
\[
\sup_{q\in\cQ}
\left\{
2\EE[\lambda(S,A)A_\beta q(Z)]-\|q\|_{\bar\mu,2}^2
\right\}
=
\|a_\lambda\|_{\bar\mu,2}^2.
\]
The population criterion therefore satisfies
\[
\begin{aligned}
G_\beta(\lambda)
&= \|a_\lambda\|_{\bar\mu,2}^2
+
(\EE[\lambda(S,A)]-1)^2 = \sup_{\substack{q\in\cQ\\ \|q\|_{\bar\mu,2}\le1}}
\left|
\EE[\lambda(S,A)A_\beta q(Z)]
\right|^2
+
(\EE[\lambda(S,A)]-1)^2.
\end{aligned}
\]

Let $r=\lambda-\lambda_\beta$. The adjoint moment equation and normalization of $\lambda_\beta$ give
\[
\EE[\lambda(S,A)A_\beta q(Z)]
=
\EE\!\left[
r(S,A)
\left(
q(S,A)-(D_\beta q)(S')
\right)
\right],
\qquad q\in\cQ.
\]
The normalization gives $\EE[\lambda(S,A)]-1=\EE[r(S,A)]$. Together with the preceding display, this gives the two discrepancies controlled by Lemma~\ref{lem:uniform-dual-identification}. Applying the lemma with $D=D_\beta$ gives
\[
G_\beta(\lambda)
\ge
c\|r\|_{\bar\mu,2}^2
=
c\|\lambda-\lambda_\beta\|_{\bar\mu,2}^2.
\]
At $\lambda=\lambda_\beta$, the same adjoint equation and normalization also give $G_\beta(\lambda_\beta)=0$.
\end{proof}

\subsection{Comparison of the Empirical Criteria}

For $\lambda\in\Lambda$, define the reference empirical criterion
\[
G_{n,\beta}(\lambda)
\coloneqq
\sup_{q\in\cQ}
\left\{
2\Pn[\lambda(S,A)A_\beta q(Z)]-\|q\|_\tau^2
\right\}
+(\Pn\lambda-1)^2 ,
\]
which replaces the population expectation in $G_\beta$ by $\Pn$ and the population norm by $\|\cdot\|_\tau$, while keeping the derivative weights at $Q_\beta$ through $A_\beta$. The feasible criterion $\widehat G_\beta$ of Section~\ref{sec:plug-in-adjoint-erm} instead uses $\widehat A_\beta$, which evaluates the derivative weights at $\widehat Q$.

Define $\Delta_n:=\sup_{\lambda\in\Lambda}|G_{n,\beta}(\lambda)-G_\beta(\lambda)|$ and $\Pi_n:=\sup_{\lambda\in\Lambda}|\widehat G_\beta(\lambda)-G_{n,\beta}(\lambda)|$. Thus, $\Delta_n$ compares the reference empirical and population criteria, whereas $\Pi_n$ isolates the additional plug-in discrepancy.

\subsubsection{Reference Empirical Criterion Deviation}
The inner optimizations in $G_{n,\beta}$ and $G_\beta$ range over the unbounded linear space $\cQ$. On a high-probability norm comparison event, the quadratic penalties dominate their linear terms outside a common deterministic ball. It therefore suffices to control the two criteria on this bounded ball, where a finite dimensional covering argument applies. After this localization, the relevant class is bounded and finite dimensional, and standard concentration applies.

\begin{lemma}[Reference empirical criterion deviation]
\label{lem:delta}

Under the bounded tabular construction of $\Lambda$,
\[
\Delta_n
=
O_P\left(\sqrt{\frac{|\cS||\cA|\log n}{n}}\right).
\]

\end{lemma}

\begin{proof}
Recall that $
G_{n,\beta}(\lambda)
=
\sup_{q\in\cQ}
\{2\Pn[\lambda(S,A)A_\beta q(Z)]-\|q\|_\tau^2\}
+(\Pn\lambda-1)^2.$
The population criterion is $
G_\beta(\lambda)
=
\sup_{q\in\cQ}
\{2\EE[\lambda(S,A)A_\beta q(Z)]-\|q\|_{\bar\mu,2}^2\}
+(\EE\lambda-1)^2. $

\paragraph{Restriction to a common test function ball.}
Let $\widehat\mu_n(s,a)\coloneqq\Pn\{(S,A)=(s,a)\}$.
Since $\cS\times\cA$ is finite, a finite union bound and Hoeffding's inequality imply that, with probability at least $1-e^{-t}$,
\[
\max_{(s,a)\in\cS\times\cA}
|\widehat\mu_n(s,a)-\bar\mu(s,a)|
\le
C
\sqrt{\frac{|\cS||\cA|+t}{n}}.
\]
Let $\mathcal E_{\mu,n}$ be the event on which
\[
\frac12\|q\|_{\bar\mu,2}^2
\le
\Pn[q(S,A)^2]
\le
\frac32\|q\|_{\bar\mu,2}^2,
\qquad q\in\cQ.
\]
The preceding finite union bound implies $\PP(\mathcal E_{\mu,n})\to1$.
In the remainder of the proof, work on this event.

Since $\lambda\in\Lambda$ implies $\|\lambda\|_\infty\le B_\Lambda$, the $\beta$-free signed-envelope bound for $A_\beta$, and finite support imply that there exists $C<\infty$ such that, for every $\lambda\in\Lambda$ and $q\in\cQ$, $|\Pn[\lambda(S,A)A_\beta q(Z)]|\le C\|q\|_{\bar\mu,2}$. The same bound holds for the population moment, $|\EE[\lambda(S,A)A_\beta q(Z)]|\le C\|q\|_{\bar\mu,2}$.
Indeed, finite support and the full-support condition on $\bar\mu$ from Section~\ref{sec:tabular-fast-rate} give $|A_\beta q(Z)|\le C\|q\|_\infty\le C\|q\|_{\bar\mu,2}$. Consequently,
\[
2\Pn[\lambda(S,A)A_\beta q(Z)]-\|q\|_\tau^2
\le
C\|q\|_{\bar\mu,2}
-
\frac{1}{2}\|q\|_{\bar\mu,2}^2.
\]
The population objective satisfies the corresponding bound,
\[
2\EE[\lambda(S,A)A_\beta q(Z)]-\|q\|_{\bar\mu,2}^2
\le
C\|q\|_{\bar\mu,2}
-
\|q\|_{\bar\mu,2}^2.
\]
The same signed $\ell_1$ envelope, evaluated at $\widehat Q$, gives $\left|\Pn[\lambda(S,A)\widehat A_\beta q(Z)]\right|\le C\|q\|_{\bar\mu,2}$, and therefore
\[
2\Pn[\lambda(S,A)\widehat A_\beta q(Z)]-\|q\|_\tau^2
\le
C\|q\|_{\bar\mu,2}
-
\frac{1}{2}\|q\|_{\bar\mu,2}^2.
\]
All three inner objectives are negative outside a sufficiently large ball and equal zero at $q=0$. Hence, after choosing a sufficiently large fixed $C$, the inner suprema in $G_{n,\beta}$, $G_\beta$, and $\widehat G_\beta$ may all be restricted to the common deterministic ball $\mathcal B_Q\coloneqq\left\{q\in\cQ:\|q\|_{\bar\mu,2}\le C\right\}$, where $C<\infty$ is fixed and independent of $n$, $N$, $H$, and $\beta$.

\paragraph{Reduction to empirical averages.} We have
\[
\begin{aligned}
\Delta_n
&\le \sup_{\substack{\lambda\in\Lambda\\ q\in\mathcal B_Q}}
\left|
2(\Pn-\EE)\!\left[
\lambda(S,A)A_\beta q(Z)
\right]
-
(\Pn-\EE)\!\left[q(S,A)^2\right]
-
\frac{1}{n|\cS||\cA|}\sum_{s\in\cS}\sum_{a\in\cA}q(s,a)^2
\right|
\\
&\quad+ \sup_{\lambda\in\Lambda}
\left|
(\Pn\lambda-1)^2-(\EE\lambda-1)^2
\right|.
\end{aligned}
\]
Indeed, for each $\lambda\in\Lambda$,
\[
\begin{aligned}
&
\left|
\sup_{q\in\mathcal B_Q}
\left\{
2\Pn[\lambda(S,A)A_\beta q(Z)]
-
\|q\|_\tau^2
\right\}
-
\sup_{q\in\mathcal B_Q}
\left\{
2\EE[\lambda(S,A)A_\beta q(Z)]
-
\|q\|_{\bar\mu,2}^2
\right\}
\right|
\\
&\le \sup_{q\in\mathcal B_Q}
\left|
2(\Pn-\EE)\!\left[
\lambda(S,A)A_\beta q(Z)
\right]
-
(\Pn-\EE)\!\left[q(S,A)^2\right]
-
\frac{1}{n|\cS||\cA|}\sum_{s\in\cS}\sum_{a\in\cA}q(s,a)^2
\right|.
\end{aligned}
\]
The added finite state-action stabilizer is deterministic and uniformly of order $n^{-1}$ on $\mathcal B_Q$.
For the normalization term,
\[
\left|
(\Pn\lambda-1)^2-(\EE\lambda-1)^2
\right|
=
|(\Pn-\EE)\lambda|
\,|\Pn\lambda+\EE\lambda-2|.
\]
Since $\lambda\in\Lambda$, we have $\|\lambda\|_\infty\le B_\Lambda$, which gives a uniform bound on the second factor.
It therefore remains to control the empirical averages uniformly over $\Lambda\times\mathcal B_Q$.

\paragraph{Finite dimensional concentration.}
Consider
\[
\mathcal H
=
\left\{
\lambda(S,A)A_\beta q(Z),\ q(S,A)^2,\ \lambda(S,A):
\lambda\in\Lambda,\ q\in\mathcal B_Q
\right\}.
\]
Every function in $\mathcal H$ is uniformly bounded by a constant depending only on the finite tabular boundedness constants and the support lower bound.
Moreover, the elements of $\mathcal H$ are Lipschitz in the finite dimensional tabular coordinates of $(\lambda,q)$ on the bounded set $\Lambda\times\mathcal B_Q$. Since $\Lambda=\{\lambda:\mathcal S\times\mathcal A\to\mathbb R:\|\lambda\|_\infty\le B_\Lambda\}$ is a bounded finite-dimensional set, $\log N(\epsilon,\Lambda,L_\infty)\le C|\cS||\cA|\log(C/\epsilon)$ for $0<\epsilon<1$, and $\dim(\cQ)\le|\cS||\cA|$. Hence a finite $n^{-1}$-net in the tabular coordinates of $(\lambda,q)$ has cardinality at most $\exp(C|\cS||\cA|\log n)$. Hoeffding's inequality on the net and a union bound give, with probability at least $1-e^{-t}$,
\[
\sup_{f\in\mathcal H}|(\Pn-\EE)f|
\le
C
\sqrt{\frac{|\cS||\cA|\log n+t}{n}}.
\]
The discretization error is absorbed into the same bound because the net radius is $n^{-1}$ and the class is uniformly Lipschitz and bounded.

Since the tabular space is finite, $\sup_{q\in\mathcal B_Q}(n|\cS||\cA|)^{-1}\sum_{s,a}q(s,a)^2\le Cn^{-1}$, which is absorbed into the same rate.
Finally, since $\lambda$ is uniformly bounded, the normalization term satisfies, uniformly over $\lambda\in\Lambda$, $|(\Pn\lambda-1)^2-(\EE\lambda-1)^2|=|(\Pn-\EE)\lambda|\,|\Pn\lambda+\EE\lambda-2|\le C|(\Pn-\EE)\lambda|$, and this is controlled by the same empirical process bound applied to the subclass $\{\lambda(S,A):\lambda\in\Lambda\}\subset\mathcal H$.
Combining the Fenchel and normalization bounds yields
\[
\Delta_n
\le
C
\sqrt{\frac{|\cS||\cA|\log n+t}{n}}
\]
with probability at least $1-e^{-t}$ on $\mathcal E_{\mu,n}$, so the displayed bound holds on an event of probability at least $1-e^{-t}-\PP(\mathcal E_{\mu,n}^c)$. Fix $\epsilon>0$, choose $t$ with $e^{-t}\le\epsilon/2$, and use $\PP(\mathcal E_{\mu,n}^c)\to0$ to make the second term at most $\epsilon/2$ for all large $n$. Since $\epsilon>0$ was arbitrary, this is the stated $O_P$ rate.
\end{proof}

\subsubsection{Plug-in Criterion Perturbation}
On the event $\mathcal E_{\mu,n}$, replacing $Q_\beta$ by $\widehat Q$ perturbs the derivative weights by an amount of order $\beta\|\widehat Q-Q_\beta\|_{\bar\mu,2}$.  

\begin{lemma}[Plug-in criterion perturbation]
\label{lem:plugin}

Under the bounded tabular construction of $\Lambda$, if the generic first-stage estimator satisfies $\|\widehat Q-Q_\beta\|_{\bar\mu,2}=O_P(r_Q)$,
\[
\Pi_n
=
\sup_{\lambda\in\Lambda}
\left|\widehat G_\beta(\lambda)-G_{n,\beta}(\lambda)\right|
=
O_P(\beta r_Q).
\]
\end{lemma}

\begin{proof}
Work on $\mathcal E_{\mu,n}$. By the preceding localization argument, the inner suprema in $G_{n,\beta}(\lambda)$ and $\widehat G_\beta(\lambda)$ can both be restricted to $\mathcal B_Q$, uniformly over $\lambda\in\Lambda$.

We first bound the perturbation of the derivative weights. Since the derivative weights are the gradient of the softmax-smoothed value map, the Hessian bound in Lemma~\ref{lem:softmax-calculus}, together with the mean value theorem and norm equivalence on the fixed action space, gives a constant $C<\infty$, independent of $n$ and $\beta$, such that
\[
\sum_{a\in\mathcal A}
\left|
\widehat\omega_\beta(a\mid s)
-
\omega_{\beta,Q_\beta}(a\mid s)
\right|
\le
C\beta
\|\widehat Q(s,\cdot)-Q_\beta(s,\cdot)\|_\infty,
\qquad
s\in\mathcal S.
\]
The preceding bound gives
\begin{equation}
\begin{aligned}
\sup_{s\in\mathcal S}
\sum_{a\in\mathcal A}
\left|
\widehat\omega_\beta(a\mid s)
-
\omega_{\beta,Q_\beta}(a\mid s)
\right|
&\le C\beta
\|\widehat Q-Q_\beta\|_\infty
\le
C\beta
\|\widehat Q-Q_\beta\|_{\bar\mu,2}.
\end{aligned}
\label{eq:F4}
\end{equation}
The normalization term $(\Pn\lambda-1)^2$ is identical in $\widehat G_\beta$ and $G_{n,\beta}$. Hence, for every $\lambda\in\Lambda$,
\[
\begin{aligned}
\left|
\widehat G_\beta(\lambda)-G_{n,\beta}(\lambda)
\right|
&\le 2\sup_{q\in\mathcal B_Q}
\left|
\Pn\!\left[
\lambda(S,A)
\left(
\widehat A_\beta q(Z)-A_\beta q(Z)
\right)
\right]
\right|.
\end{aligned}
\]
For every $q\in\mathcal B_Q$,
\[
\begin{aligned}
\left|
\widehat A_\beta q(Z)-A_\beta q(Z)
\right|
&= \left|
\sum_{a\in\mathcal A}
\left(
\omega_{\beta,Q_\beta}(a\mid S')
-
\widehat\omega_\beta(a\mid S')
\right)
q(S',a)
\right|
\\
&\le \left[
\sup_{s\in\mathcal S}
\sum_{a\in\mathcal A}
\left|
\widehat\omega_\beta(a\mid s)
-
\omega_{\beta,Q_\beta}(a\mid s)
\right|
\right]
\|q\|_\infty.
\end{aligned}
\]
Moreover, finite support gives $\sup_{q\in\mathcal B_Q}\|q\|_\infty\le C<\infty$, and $\sup_{\lambda\in\Lambda}\|\lambda\|_\infty<\infty$, where the second bound follows from $\|\lambda\|_\infty\le B_\Lambda$ by construction. It follows that
\[
\Pi_n
\le
C
\sup_{s\in\mathcal S}
\sum_{a\in\mathcal A}
\left|
\widehat\omega_\beta(a\mid s)
-
\omega_{\beta,Q_\beta}(a\mid s)
\right|,
\]
so combining with \eqref{eq:F4} yields $\Pi_n\le C\beta\|\widehat Q-Q_\beta\|_{\bar\mu,2}$ on $\mathcal E_{\mu,n}$. Since $\PP(\mathcal E_{\mu,n})\to1$ and $\|\widehat Q-Q_\beta\|_{\bar\mu,2}=O_P(r_Q)$, we conclude that $\Pi_n=O_P(\beta r_Q)$.
\end{proof}

\begin{lemma}[Generic adjoint criterion comparison]\label{lem:adjoint-criterion-comparison}
Let $G$ be a population criterion with $G(\lambda_\beta)=0$ and $c\|\lambda-\lambda_\beta\|_{\bar\mu,2}^2\le G(\lambda)$ on $\Lambda$. Let $G^0$ be a reference empirical criterion, let $\widehat G$ be its plug-in version, and define
\[
\Delta\coloneqq\sup_{\lambda\in\Lambda}|G^0(\lambda)-G(\lambda)|,
\qquad
\Pi\coloneqq\sup_{\lambda\in\Lambda}|\widehat G(\lambda)-G^0(\lambda)|.
\]
If $\widehat\lambda$ minimizes $\widehat G$ over $\Lambda$ and $\lambda_\beta\in\Lambda$, then
\[
G(\widehat\lambda)
\le2\Delta+2\Pi,
\qquad
\|\widehat\lambda-\lambda_\beta\|_{\bar\mu,2}^2
\le C(\Delta+\Pi).
\]
\end{lemma}

\begin{proof}
The definitions of $\Delta$ and $\Pi$, followed by empirical optimality, give
\[
\begin{aligned}
&G(\widehat\lambda)
\le G^0(\widehat\lambda)+\Delta
\le\widehat G(\widehat\lambda)+\Delta+\Pi
\le\widehat G(\lambda_\beta)+\Delta+\Pi \le G^0(\lambda_\beta)+\Delta+2\Pi
\le G(\lambda_\beta)+2\Delta+2\Pi
=2\Delta+2\Pi.
\end{aligned}
\]
The population lower bound applied at $\widehat\lambda$ therefore gives $c\|\widehat\lambda-\lambda_\beta\|_{\bar\mu,2}^2\le2\Delta+2\Pi$, which proves the second conclusion.
\end{proof}

\subsection{Proof of Theorem~\ref{thm:main} and Corollary~\ref{cor:density-first-stage}}

We now combine the population identification bound with the two criterion comparisons.

\begin{proof}[Proof of Theorem~\ref{thm:main}]
Proposition~\ref{prop:adjoint-quantitative-lower} supplies the population quadratic identification required by Lemma~\ref{lem:adjoint-criterion-comparison}. Apply that lemma with $(G,G^0,\widehat G,\widehat\lambda,\Delta,\Pi)=(G_\beta,G_{n,\beta},\widehat G_\beta,\widehat\lambda_\beta,\Delta_n,\Pi_n)$. Lemmas~\ref{lem:delta} and~\ref{lem:plugin} then give the stated stochastic rate.
\end{proof}

\begin{proof}[Proof of Corollary~\ref{cor:density-first-stage}]
Theorem~\ref{thm:fast_rate} gives $\|\widehat Q-Q_\beta\|_{\bar\mu,2}=O_P\left(\sqrt{\frac{|\cS||\cA|\log n}{n}}\right)$. Substituting this rate into Theorem~\ref{thm:main} gives
\[
\|\widehat\lambda_\beta-\lambda_\beta\|_{\bar\mu,2}^2
=
O_P\left( 
\left[
\sqrt{\frac{|\cS||\cA|\log n}{n}}
+
\beta
\sqrt{\frac{|\cS||\cA|\log n}{n}}
\right]
\right).
\]
\end{proof}

\section{Proof for Section \ref{sec:dependent-trajectory-nuisance-rates}}
\label{app:dependent-trajectory-nuisance-estimation}

\subsection{Pooled Empirical Averages under Dependent Trajectories}
\label{app:dependent-pooled-averages}
For a one-step function $f$, write
$
P_{N,H}f
=
\frac1{NH}
\sum_{i=1}^N
\sum_{t=0}^{H-1}
f(Z_{i,t}),$ and $
\bar P_Hf
=
\frac1H
\sum_{t=0}^{H-1}
\EE[f(Z_{i,t})].
$
For trajectory $i$, let $\mathcal F_{i,t}^- \coloneqq \sigma(S_{i,0},A_{i,0},R_{i,1},\ldots,S_{i,t},A_{i,t})$ be the pre-transition filtration at time $t$. For a bounded one-step function $f$, define
$
m_f(s,a)
\coloneqq
\EE[f(Z_{i,t})\mid S_{i,t}=s,A_{i,t}=a].
$
The Markov condition in Section~\ref{sec:formu}, together with the time-homogeneous reward and transition law, implies that this conditional map does not depend on $t$ and that
$$
\EE[f(Z_{i,t})-m_f(S_{i,t},A_{i,t})\mid \mathcal F_{i,t}^-]=0.
$$

\begin{lemma}[Pooled root-$NH$ averaging]
\label{lem:dependent-pooled-root-nh}
Suppose the Markov condition of Section~\ref{sec:formu} and Assumption~\ref{ass:behavior-process-mixing} hold. Let $f_H$ be a possibly $H$-dependent one-step function with no explicit time index and $\sup_H\|f_H\|_\infty<\infty$. Then $P_{N,H}f_H-\bar P_Hf_H=O_P((NH)^{-1/2})$.
The same conclusion holds uniformly over any class admitting a representation
$
f_{\theta,H}(Z)
=
\sum_{r=1}^{D}
a_{r,H}(\theta)h_{r,H}(Z),
$
where $D<\infty$ is fixed, $\sup_{\theta,H,r}|a_{r,H}(\theta)|<\infty$, and $\sup_{H,r}\|h_{r,H}\|_\infty<\infty$.
\end{lemma}

\begin{proof}
Write $X_{i,t}=(S_{i,t},A_{i,t})$, $m_{f_H}(x)=\EE[f_H(Z_{i,t})\mid X_{i,t}=x]$, $\epsilon_{i,t}=f_H(Z_{i,t})-m_{f_H}(X_{i,t})$, and $r_{i,t}=m_{f_H}(X_{i,t})-\EE[m_{f_H}(X_{i,t})]$. The first display above gives $\EE[\epsilon_{i,t}\mid\mathcal F_{i,t}^-]=0$. For $t<u$, $\epsilon_{i,t}$ is $\mathcal F_{i,u}^-$-measurable, so iterated expectation gives $\EE[\epsilon_{i,t}\epsilon_{i,u}]=0$. Uniform boundedness therefore gives
$$
\Var\left(
\frac1H\sum_{t=0}^{H-1}\epsilon_{i,t}
\right)
\le
\frac{C}{H}.
$$
The predictable term $r_{i,t}$ is the centered version of the bounded function $m_{f_H}(X_{i,t})$. Applying Lemma~\ref{lem:doeblin-predictable-residual-covariance} with $f_H=m_{f_H}$ gives
$$
\Var\left(
\frac1H\sum_{t=0}^{H-1}r_{i,t}
\right)
\le
\frac{C}{H}.
$$
Combining the two components with $\Var(U+V)\le2\Var(U)+2\Var(V)$ shows that one trajectory average has variance $O(H^{-1})$. The $N$ trajectories are independent, so Chebyshev's inequality gives the first claim. For the finite-basis class,
$
\sup_\theta |(P_{N,H}-\bar P_H)f_{\theta,H}|
\le
C\sum_{r=1}^D |(P_{N,H}-\bar P_H)h_{r,H}|,
$
and the same rate follows because $D$ is fixed.
\end{proof}

The stochastic terms in the Bellman and adjoint arguments below reduce, after normalization, to fixed finite linear spans generated by bounded local transition coordinates, such as $1\{(S,A)=x\}$, $1\{(S,A)=x,S'=s'\}$, and $1\{(S,A)=x\}R$. The reward coordinate is bounded by Assumption~\ref{ass:if-envelope} when the Bellman residual is involved. The coefficients are uniformly bounded by finite state and action spaces, the fixed bounded parameter domains, Corollary~\ref{cor:app-smoothed-target-bounded} and Corollary~\ref{cor:target-adjoint-bound} when adjoint weights enter, and the $\beta$-uniform signed-derivative envelope in Lemma~\ref{lem:if-omega-tv-bound}. Lemma~\ref{lem:dependent-pooled-root-nh} therefore applies without assuming that the transition sequence $Z_t$ is itself strong mixing.

\subsection{Proof of Theorem~\ref{thm:dependent-trajectory-nuisance-rates}}
\label{app:dependent-nuisance-rates-proof}

\begin{proof}
We first establish the Bellman pair rate. Let $\widehat\mu_H(x)=P_{N,H}1\{X=x\}$. Applying Lemma~\ref{lem:dependent-pooled-root-nh} to the finitely many state-action indicators gives
$
\max_x|\widehat\mu_H(x)-\bar\mu_H(x)|
=
O_P((NH)^{-1/2}).
$
By Assumption~\ref{ass:pooled-support}, with probability tending to one,
$
\frac12\|g\|_{\bar\mu_H,2}^2
\le
P_{N,H}[g(S,A)^2]
\le
\frac32\|g\|_{\bar\mu_H,2}^2
$
uniformly over tabular $g$.
Proposition~\ref{prop:bellman-residual-identification}, applied with the pooled state-action marginal $\bar\mu_H$, gives
$
\|m_{Q,J}\|_{\bar\mu_H,2}
\ge
c
(\|Q-Q_\beta\|_{\bar\mu_H,2}+|J-J_\beta|)
$
for every admissible $(Q,J)$. Its constant is uniform in $H$ because the norm-equivalence proof of Lemma~\ref{lem:injectivity-equivalence} uses the common lower bound $\inf_H\min_x\bar\mu_H(x)>0$ from Assumption~\ref{ass:pooled-support}.

Let $\eta_{N,H}=(NH)^{-1/2}$ and $\delta_\beta(Z)=R-J_\beta+V_\beta(Q_\beta)(S')-Q_\beta(S,A)$. For each state-action cell $x$, the process $1\{X_{i,t}=x\}\delta_\beta(Z_{i,t})$ is a bounded martingale difference by the self-induced Bellman equation, and hence
$
\bar P_H[1\{X=x\}\delta_\beta(Z)]=0.
$
Lemma~\ref{lem:dependent-pooled-root-nh} therefore gives
$
P_{N,H}[1\{X=x\}\delta_\beta(Z)]=O_P(\eta_{N,H})
$
uniformly over the finitely many cells. Let $\widehat L_{N,H}$ be the criterion of Section~\ref{sec:tabular-fast-rate} with $\Pn$ replaced by $P_{N,H}$. Completing the square on the norm-comparison event yields
$
\widehat L_{N,H}(Q_\beta,J_\beta)
=
O_P(\eta_{N,H}^2).
$

For admissible $(Q,J)$, set $d_{Q,J}(Z)=\delta_{Q,J}(Z)-\delta_\beta(Z)$ and $e_{Q,J}=\|Q-Q_\beta\|_{\bar\mu_H,2}+|J-J_\beta|$. The $\beta$-uniform Lipschitz bound in Lemma~\ref{lem:fast-residual-lipschitz}, with the horizon-uniform support from Assumption~\ref{ass:pooled-support}, gives $\|d_{Q,J}\|_\infty\le Ce_{Q,J}$. After normalizing by $e_{Q,J}$ and $\|g\|_{\bar\mu_H,2}$ when these quantities are nonzero, the products $g(S,A)d_{Q,J}(Z)$ belong to a bounded fixed-dimensional tabular class of functions of $(S,A,S')$. Hence Lemma~\ref{lem:dependent-pooled-root-nh} gives
$$
\sup_{g,Q,J}
\frac{|(P_{N,H}-\bar P_H)[g d_{Q,J}]|}
{\eta_{N,H}\|g\|_{\bar\mu_H,2}e_{Q,J}}
=O_P(1),
$$
where the supremum is over nonzero denominators. If $e_{Q,J}=0$, then $d_{Q,J}=0$, while $\|g\|_{\bar\mu_H,2}=0$ implies $g=0$ by pooled support, so the corresponding product is zero; equivalently, set such ratios to zero.

We now apply the eventwise form of Lemma~\ref{lem:bellman-localization}. Fix $\epsilon>0$. The tightness of $\widehat L_{N,H}(Q_\beta,J_\beta)/\eta_{N,H}^2$ and of the preceding supremum provides a deterministic $A_\epsilon<\infty$ such that, for all sufficiently large $NH$, the truth-criterion and uniform product bounds hold with constant $A_\epsilon$, each outside an event of probability at most $\epsilon/3$. Intersect these two events with the empirical norm-comparison event above, whose failure probability is at most $\epsilon/3$ for all sufficiently large $NH$. The resulting event $\mathcal E_{\epsilon,N,H}$ has probability at least $1-\epsilon$.

On $\mathcal E_{\epsilon,N,H}$, all hypotheses of Lemma~\ref{lem:bellman-localization} hold at the original scale $\eta_{N,H}$. The identification constant, deterministic domain bound, and fixed test-ball radius are uniform in $H$ by Assumption~\ref{ass:pooled-support}; in particular, the fixed sup-norm test class contains a common $L_2(\bar\mu_H)$ ball. The remaining constants may depend on $\epsilon$ through $A_\epsilon$, but not on $N,H$, or $\beta$ along the sequence. Applying the lemma with $(P,\widehat P,\mu,\eta)=(\bar P_H,P_{N,H},\bar\mu_H,\eta_{N,H})$ gives, on $\mathcal E_{\epsilon,N,H}$,
$$
\|\widehat Q-Q_\beta\|_{\bar\mu_H,2}+|\widehat J-J_\beta|
\le C_\epsilon(NH)^{-1/2}.
$$
Since $\epsilon$ was arbitrary, this yields
$
\|\widehat Q-Q_\beta\|_{\bar\mu_H,2}+|\widehat J-J_\beta|
=
O_P((NH)^{-1/2}).
$

We next establish the adjoint-weight rate, using the Bellman pair rate above to control the plug-in derivative weights. The population identification argument in Appendix~\ref{app:plug-in-adjoint-erm-proofs} is deterministic. Proposition~\ref{prop:adjoint-quantitative-lower}, applied under the pooled law, therefore gives
$
c\|\lambda-\lambda_\beta\|_{\bar\mu_H,2}^2\le G_\beta(\lambda)
$
uniformly over the relevant horizons and $\lambda\in\Lambda$.
Let $G_{N,H,\beta}^0$ be the reference empirical criterion obtained by using $P_{N,H}$ and the all-transition stabilized quadratic norm, with $n$ replaced by $NH$, while keeping the derivative weights evaluated at $Q_\beta$. On the empirical norm-comparison event established above, the inner supremum over $q\in\cQ$ is restricted to a fixed deterministic ball, as in Lemma~\ref{lem:delta}. On this ball, the criterion is generated by bounded fixed-dimensional tabular classes, uniformly in $\beta$ by Lemma~\ref{lem:if-omega-tv-bound}. Hence Lemma~\ref{lem:dependent-pooled-root-nh} gives
$
\Delta^{\mathrm{adj}}_{N,H}
:=
\sup_{\lambda\in\Lambda}
|G_{N,H,\beta}^0(\lambda)-G_\beta(\lambda)|
=
O_P((NH)^{-1/2}).
$

It remains to control the additional error from replacing $Q_\beta$ by $\widehat Q$ in the derivative weights. Lemma~\ref{lem:plugin}, whose derivative-weight step uses the master Hessian bound in Lemma~\ref{lem:softmax-calculus}, gives the required plug-in perturbation. Assumption~\ref{ass:pooled-support} makes its norm-equivalence constant uniform in $H$, so
$
\sup_s\sum_a|\widehat\omega_\beta(a\mid s)-\omega_{\beta,Q_\beta}(a\mid s)|
\le C\beta\|\widehat Q-Q_\beta\|_{\bar\mu_H,2}.
$
Thus, if $\|\widehat Q-Q_\beta\|_{\bar\mu_H,2}=O_P(r_Q)$, then
$
\Pi^{\mathrm{adj}}_{N,H}
:=
\sup_{\lambda\in\Lambda}
|\widehat G_{N,H,\beta}(\lambda)-G_{N,H,\beta}^0(\lambda)|
=
O_P(\beta r_Q).
$

Applying Lemma~\ref{lem:adjoint-criterion-comparison} with
$
(G,G^0,\widehat G,\widehat\lambda,\Delta,\Pi)
=
(G_\beta,G_{N,H,\beta}^0,\widehat G_{N,H,\beta},\widehat\lambda_\beta,\Delta^{\mathrm{adj}}_{N,H},\Pi^{\mathrm{adj}}_{N,H})
$
therefore gives
$
\|\widehat\lambda_\beta-\lambda_\beta\|_{\bar\mu_H,2}^2
=
O_P((NH)^{-1/2}+\beta r_Q).
$
The Bellman-pair rate proved above gives $r_Q=(NH)^{-1/2}$, and hence
$$
\|\widehat\lambda_\beta-\lambda_\beta\|_{\bar\mu_H,2}
=
O_P\!\left(\sqrt{1+\beta}\,(NH)^{-1/4}\right).
$$
Together with the Bellman pair rate, this proves Theorem~\ref{thm:dependent-trajectory-nuisance-rates}.
\end{proof}
\subsection{Verification of the Nuisance Conditions}
\label{app:dependent-inference-verification}

For trajectory-level cross-fitting with fixed $K$, each training sample contains a fixed fraction of the $N$ independent trajectories, so Theorem~\ref{thm:dependent-trajectory-nuisance-rates} applies foldwise with the same $NH$ order. Thus,
$q_k=\|\widehat Q_{-k}-Q_\beta\|_{\bar\nu_H,2}=O_P((NH)^{-1/2})$,
$j_k=|\widehat J_{-k}-J_\beta|=O_P((NH)^{-1/2})$, and
$\ell_k=\|\widehat\lambda_{-k}-\lambda_\beta\|_{\bar\mu_H,2}
=O_P(\sqrt{1+\beta}\,(NH)^{-1/4})$,
where the rate for $q_k$ follows from finite-state norm equivalence under Assumption~\ref{ass:pooled-support}.

Substituting these rates into the definition of $b_k$ gives
\[
\beta q_k^2=O_P(\beta(NH)^{-1}),\qquad
\ell_kj_k+\ell_kq_k=O_P(\sqrt{1+\beta}\,(NH)^{-3/4}),
\qquad
\beta\ell_kq_k^2=O_P(\beta\sqrt{1+\beta}\,(NH)^{-5/4}).
\]
After multiplication by $(NH)^{1/2}$, the corresponding orders are $
\frac{\beta}{\sqrt{NH}},$ $
\frac{\sqrt{1+\beta}}{(NH)^{1/4}}$,and $
\frac{\beta\sqrt{1+\beta}}{(NH)^{3/4}},$
all of which converge to zero when $\beta=o(\sqrt{NH})$. The same condition gives $\ell_k=o_P(1)$, while $q_k,j_k=o_P(1)$. Since the number of folds is fixed,
$\max_{1\le k\le K}b_k=o_P((NH)^{-1/2})$,
verifying Assumption~\ref{ass:if-growing-rate}, and Assumption~\ref{ass:if-consistency} follows as well.

\end{document}